\documentclass[aos]{imsart}

\RequirePackage[numbers,sort&compress]{natbib}

\usepackage{placeins}
\usepackage[utf8]{inputenc} \usepackage[T1]{fontenc}    \usepackage{url}            \usepackage{booktabs}       \usepackage{amsfonts}       \usepackage{nicefrac}       \usepackage{microtype}      \usepackage{mathrsfs}

\usepackage{amsmath,amssymb,amsthm,xcolor}
\usepackage{array}
\usepackage{float}
\theoremstyle{plain}
\newtheorem{theorem}{Theorem}[section]
\newtheorem{lemma}[theorem]{Lemma}
\newtheorem{proposition}[theorem]{Proposition}
\newtheorem{corollary}[theorem]{Corollary}

\theoremstyle{definition}
\newtheorem{definition}[theorem]{Definition}

\theoremstyle{definition}

\usepackage{graphicx}
\usepackage{mdframed}
\usepackage{wrapfig}
\usepackage[hidelinks,hypertexnames=false]{hyperref}
\usepackage{tikz}
\usetikzlibrary{arrows.meta,positioning}
\usepackage{enumitem}
\allowdisplaybreaks
\usepackage{epstopdf}
\usepackage{changepage}
\usepackage{indentfirst}
\usepackage{stmaryrd} \usepackage{soul}

\newcommand{\DeclareMainXRef}[2]{\expandafter\def\csname crossdoc@main@#1\endcsname{#2}}

\newcommand{\MainXRefNumber}[1]{\ifcsname crossdoc@main@#1\endcsname \csname crossdoc@main@#1\endcsname \else \PackageError{crossdoc-refs}{Undefined main-paper cross-reference `#1'}{}??\fi}

\newcommand{\MainResult}[2]{main-paper #1~\MainXRefNumber{#2}}
\newcommand{\MainEquation}[1]{main-paper Equation~(\MainXRefNumber{#1})}
\newcommand{\MainSection}[1]{Section~\MainXRefNumber{#1} of the main paper}

\newcommand{\DeclareSuppXRef}[2]{\expandafter\def\csname crossdoc@supp@#1\endcsname{#2}}

\newcommand{\SuppXRefNumber}[1]{\ifcsname crossdoc@supp@#1\endcsname \csname crossdoc@supp@#1\endcsname \else \PackageError{crossdoc-refs}{Undefined supplement cross-reference `#1'}{}??\fi}

\newcommand{\SuppResult}[2]{Supplement #1~\SuppXRefNumber{#2}}
\newcommand{\SuppEquation}[1]{Supplement Equation~(\SuppXRefNumber{#1})}
\newcommand{\SuppSection}[1]{Supplement Section~\SuppXRefNumber{#1}}

\DeclareMainXRef{eq:risk-increment}{1}
\DeclareMainXRef{def:expected-local-endpoint}{2.1}
\DeclareMainXRef{prop:mean-valid-decision}{2.2}
\DeclareMainXRef{eq:mean-valid-risk-oracle}{4}
\DeclareMainXRef{eq:mean-valid-selected-gap}{5}
\DeclareMainXRef{prop:orientation-selection-mean-cost}{2.3}
\DeclareMainXRef{eq:sample-covariance-pairwise}{9}
\DeclareMainXRef{eq:crossfit-single-split-calibration}{10}
\DeclareMainXRef{eq:centered-linear-crossfit-effective-dimension}{11}
\DeclareMainXRef{thm:centered-linear-expected-certificate}{2.4}
\DeclareMainXRef{eq:main-covariance-matched-subgaussian}{16}
\DeclareMainXRef{eq:population-ridge-effective-dimension}{17}
\DeclareMainXRef{eq:conditional-pilot-geometry}{19}
\DeclareMainXRef{eq:balanced-calibration-cross-term}{22b}
\DeclareMainXRef{eq:conditional-trace-relative-error}{24}
\DeclareMainXRef{eq:conditional-evaluation-trace-ucb}{25}
\DeclareMainXRef{eq:balanced-quadratic-confidence-inflation}{26}
\DeclareMainXRef{thm:balanced-one-split-linear-confidence}{2.5}
\DeclareMainXRef{eq:balanced-one-split-quadratic-confidence}{27}
\DeclareMainXRef{eq:gaussian-conditional-trace-ucb}{28}
\DeclareMainXRef{eq:gaussian-quadratic-confidence-inflation}{29}
\DeclareMainXRef{cor:gaussian-one-split-linear-confidence}{2.6}
\DeclareMainXRef{eq:bounded-feature-quadratic-confidence-inflation}{31}
\DeclareMainXRef{cor:bounded-feature-one-split-linear-confidence}{2.7}
\DeclareMainXRef{eq:main-pilot-stability-rate}{32}
\DeclareMainXRef{eq:balanced-finite-split-inflated-intercept}{36}
\DeclareMainXRef{eq:balanced-finite-split-confidence-ellipsoid}{37}
\DeclareMainXRef{eq:balanced-finite-split-linear-confidence}{38}
\DeclareMainXRef{thm:balanced-one-split-linear-confidence-field}{2.8}
\DeclareMainXRef{cor:balanced-one-split-oracle-width}{2.9}
\DeclareMainXRef{eq:balanced-finite-split-oracle-width}{41}
\DeclareMainXRef{sec:finite-sample-certificates}{2}
\DeclareMainXRef{sec:local}{3}
\DeclareMainXRef{eq:frozen-branch-decomposition}{44}
\DeclareMainXRef{eq:frozen-branch-taylor-envelope}{45}
\DeclareMainXRef{eq:frozen-branch-taylor-remainder}{46}
\DeclareMainXRef{prop:fixed-mask-crossing-reduction}{3.1}
\DeclareMainXRef{eq:pullback-lipschitz-increments}{47}
\DeclareMainXRef{prop:finite-selection-crossing-slope}{3.2}
\DeclareMainXRef{eq:finite-selection-crossing-envelope}{48}
\DeclareMainXRef{eq:lad-crossing-term}{53}
\DeclareMainXRef{eq:radial-remainder-smoothness-envelope}{55}
\DeclareMainXRef{eq:finite-split-crossing-penalty}{56}
\DeclareMainXRef{eq:hybrid-confidence-allocation}{57}
\DeclareMainXRef{eq:hybrid-taylor-confidence-width}{58}
\DeclareMainXRef{eq:hybrid-balanced-crossing-modulus}{59}
\DeclareMainXRef{eq:hybrid-crossing-confidence-correction}{60}
\DeclareMainXRef{eq:hybrid-crossing-confidence-width}{62}
\DeclareMainXRef{eq:hybrid-remainder-confidence-width}{63}
\DeclareMainXRef{eq:hybrid-upper-endpoint}{65}
\DeclareMainXRef{eq:hybrid-lower-endpoint}{66}
\DeclareMainXRef{eq:hybrid-taylor-subgaussian-conditions}{67}
\DeclareMainXRef{thm:componentwise-hybrid-local-confidence}{3.3}
\DeclareMainXRef{eq:hybrid-local-confidence-coverage}{68}
\DeclareMainXRef{eq:assembled-local-simple-width-rate}{69}
\DeclareMainXRef{cor:finite-sample-lad-endpoints}{3.4}

\DeclareSuppXRef{supp:sec:decomposition-validity}{S.1}
\DeclareSuppXRef{supp:lem:null-anchor-obstruction}{S.1.1}
\DeclareSuppXRef{supp:subsec:linear-mean-valid-proof}{S.2.1}
\DeclareSuppXRef{lem:crossfit-u-average-identities}{S.2.1}
\DeclareSuppXRef{eq:crossfit-average-pilot-covariance}{S.2.1}
\DeclareSuppXRef{lem:conditional-second-moment-calibration}{S.2.2}
\DeclareSuppXRef{supp:subsec:pilot-conditional-confidence}{S.2.2}
\DeclareSuppXRef{prop:conditional-linear-intercept-optimality}{S.2.3}
\DeclareSuppXRef{lem:conditional-evaluation-trace-sandwich}{S.2.6}
\DeclareSuppXRef{supp:subsec:finite-split-proofs-rates}{S.2.4}
\DeclareSuppXRef{supp:thm:balanced-finite-split-linear-confidence}{S.2.10}
\DeclareSuppXRef{supp:cor:balanced-finite-split-oracle-width}{S.2.11}
\DeclareSuppXRef{lem:regularized-covariance-concentration}{S.2.9}
\DeclareSuppXRef{app:crossfit-composite-lemmas}{S.4}
\DeclareSuppXRef{prop:balanced-pullback-conditional-concentration}{S.4.5}
\DeclareSuppXRef{supp:prop:assembled-confidence-width-rate}{S.4.6}
\DeclareSuppXRef{app:worked-examples}{S.5}
\DeclareSuppXRef{prop:supp-lad-direct-full-ball}{S.5.1}
\DeclareSuppXRef{prop:supp-lad-active-empirical-bernstein}{S.5.2}
  
\begin{document}

\begin{frontmatter}
\title{Cross-Calibrated Confidence Fields for Local Risk Updates}
\runtitle{Cross-Calibrated Fields for Local Risk Updates}

\begin{aug}
\author[A]{\fnms{Mingzhi}~\snm{Song}\ead[label=e1]{songmingzhi123@gmail.com}}
\address[A]{Department of Mathematics, The University of Hong Kong\printead[presep={,\ }]{e1}}
\end{aug}

\begin{abstract}
How can training data be used to compare local updates to the current model,
choose an update, and retain valid bounds for the selected update's
population-risk change?
We construct lower and upper confidence fields that jointly cover the
population-risk change of every update in a possibly continuous local update
space.  The fields can therefore be used both to compare the updates and to
select an update; a negative upper endpoint certifies improvement over the reference
model.  For linear risk changes in possibly infinite-dimensional feature
spaces, cross-calibration uses the discrepancy between two balanced folds to
calibrate the full-sample estimation error.  Under covariance-aligned
sub-Gaussian tails, covariance-estimation stability, and sufficient sample
size, the cross-calibrated field has finite-sample simultaneous coverage.  The confidence field's
directional widths are governed by a population ridge effective
dimension rather than the ambient feature dimension.  For losses formed
locally by continuous selection among finitely many smooth branches, separate
uniform bounds for the linear Taylor field, Taylor remainder, and
branch-interface discrepancy extend the field to every local update.
\end{abstract}
 

\end{frontmatter}

\section{Introduction}

How can training data be used to compare local updates to the current model,
choose an update, and retain a valid bound for the selected update's
population-risk change?  Let \(S:=(Z_1,\ldots,Z_n)\) be a training sample of
independent observations with law \(P\), and write
\(P_nf=n^{-1}\sum_{i=1}^nf(Z_i)\).  For a loss \(\ell_\theta\) and a fixed
reference point \(\theta_0\), define the loss increment of an update \(v\)
in a nonempty local update space \(\mathcal D\) by
\begin{equation}
\label{eq:risk-increment}
 \delta_v(z):=\ell_{\theta_0+v}(z)-\ell_{\theta_0}(z),
 \qquad v\in\mathcal D.
\end{equation}
The expectation \(P\delta_v\) is the population-risk change of update \(v\).
Fix \(\alpha\in(0,1)\), and let
\(\widehat{\mathsf L},\widehat{\mathsf U}\) be lower and upper random fields.
An update \(\widehat v(S)\in\mathcal D\) may be chosen using the training
sample \(S\).
Pointwise \(1-\alpha\) coverage for each fixed \(v\) does
not imply \(1-\alpha\) coverage for \(\widehat v(S)\).  The required
guarantee is the single simultaneous event
\begin{equation}
\label{eq:intro-simultaneous-coverage}
 \Pr\left\{
 \widehat{\mathsf L}(v)\le P\delta_v\le\widehat{\mathsf U}(v)
 \text{ for every }v\in\mathcal D
 \right\}\ge1-\alpha.
\end{equation}
On the event in \eqref{eq:intro-simultaneous-coverage}, the selected update is
covered.  In particular,
\(\widehat{\mathsf U}(\widehat v)<0\) certifies
that the selected update reduces population risk: \(P\delta_{\widehat v}<0\).

For linearly represented risk changes, the problem becomes simultaneous
inference for a Hilbert-valued mean.  Our construction splits the training
sample into a pilot fold and an evaluation fold.  The pilot fold
estimates a regularized covariance geometry, while the difference between the
pilot-fold and evaluation-fold means calibrates the error of the full-sample
mean in the regularized covariance geometry estimated from the pilot fold.  The
cross-calibrated confidence field is centered on all training observations.
We call the procedure cross-calibration.

\subsection*{Contributions}

For linearly represented risk changes in a separable Hilbert space,
Theorem~\ref{thm:balanced-one-split-linear-confidence-field} constructs a
finite-sample confidence ellipsoid centered at the mean of all training
observations.  Under sub-Gaussian tails at the covariance scale and stability
of the pilot covariance estimate, Corollary~
\ref{cor:balanced-one-split-oracle-width} shows that the ellipsoid's
directional widths are governed by a population ridge effective
dimension rather than the ambient dimension.

For losses that are locally continuous selections from finitely many
smooth branches, the risk increment separates into a linear term, a
within-branch remainder, and the discrepancy created at branch interfaces.
Theorem~\ref{thm:componentwise-hybrid-local-confidence} combines uniform
bounds for the linear term, within-branch remainder, and branch-interface
discrepancy into a finite-sample simultaneous confidence
field.

Under second moments,
Theorem~\ref{thm:centered-linear-expected-certificate} gives a
mean-valid field, and Proposition~
\ref{prop:mean-valid-decision} combines the mean-valid field's
empirical-center representation
with mean validity to bound the expected risk of an approximate minimizer.
The mean-valid field and expected-risk bound are expectation guarantees,
distinct from the realized simultaneous coverage in
\eqref{eq:intro-simultaneous-coverage}.

\subsection*{Related work}

Table~\ref{tab:intro-related-work} compares inferential targets and six
properties: protection of any selector, infinite or continuous indexing, direct
finite-sample validity, two-sided coverage, covariance-shaped geometry, and a
population ridge effective-dimension width bound.

\begin{table}[!htbp]
\caption{Guarantees for data-dependent selection over prespecified index
sets.  Entries describe each cited result under the assumptions stated for
the cited result.}
\label{tab:intro-related-work}
\centering
\footnotesize
\setlength{\tabcolsep}{1.8pt}
\renewcommand{\arraystretch}{1.08}
\begin{tabular}{@{}>{\raggedright\arraybackslash}p{.20\textwidth}
                    >{\raggedright\arraybackslash}p{.20\textwidth}
                    *{6}{>{\centering\arraybackslash}p{.071\textwidth}}@{}}
\toprule
Result & Target or guarantee
& \shortstack{Any\\selector}
& \shortstack{Infinite\\index}
& \shortstack{Finite\\sample}
& \shortstack{Two\\sided}
& \shortstack{Cov.\\shape}
& \shortstack{Eff.\\dim.} \\
\midrule
Fixed-class risk bound
\cite[Thm.~8]{bartlett2002rademacher}
& Uniform upper risk bound
& \(\checkmark\) & \(\checkmark\) & \(\checkmark\)
& \(\times\) & \(\times\) & \(\times\) \\
\addlinespace[1pt]
Finite-grid risk tests
\cite[Thm.~2.1]{angelopoulos2025learnthen}
& Risk-threshold control
& \(\checkmark^{*}\) & \(\times\) & \(\checkmark\)
& \(\times\) & \(\times\) & \(\times\) \\
\addlinespace[1pt]
Finite-dimensional Hoeffding mean box
\cite[Thm.~2]{hoeffding1963probability}
& Mean vector and directional projections
& \(\checkmark^{\ddagger}\) & \(\checkmark^{\ddagger}\)
& \(\checkmark^{\ddagger}\) & \(\checkmark^{\ddagger}\)
& \(\times\) & \(\times\) \\
\addlinespace[1pt]
Self-normalized empirical-Bernstein mean ellipsoid
\cite[Thm.~6.1]{whitehouse2026timeuniform}
& Bounded vector mean and directional projections
& \(\checkmark^{\dagger}\) & \(\checkmark^{\dagger}\)
& \(\checkmark\) & \(\checkmark^{\dagger}\)
& \(\checkmark\) & \(\times\) \\
\addlinespace[1pt]
Empirical-Bernstein Banach mean ball
\cite[Cor.~1 and Sec.~4.1]{martineztaboada2026banach}
& Bounded Banach mean and dual-direction projections
& \(\checkmark^{\dagger}\) & \(\checkmark^{\dagger}\)
& \(\checkmark\) & \(\checkmark^{\dagger}\)
& \(\times\) & \(\times\) \\
\addlinespace[1pt]
Regularized finite-dimensional mean region
\cite[Thm.~2, Thm.~3]{issouani2024regularizedhotelling}
& Mean vector and directional projections
& \(\checkmark^{\dagger}\) & \(\checkmark^{\dagger}\) & \(\checkmark\)
& \(\checkmark^{\dagger}\) & \(\checkmark\) & \(\times\) \\
\addlinespace[1pt]
\textbf{Present linear confidence field}
(Theorem~\ref{thm:balanced-one-split-linear-confidence-field};
Corollary~\ref{cor:balanced-one-split-oracle-width})
& Signed risk change for every update
& \(\checkmark\) & \(\checkmark\) & \(\checkmark\)
& \(\checkmark\) & \(\checkmark\) & \(\checkmark\) \\
\bottomrule
\end{tabular}
\parbox{\textwidth}{\vspace{2pt}\scriptsize
``Any selector'' means that one joint statement protects every data-dependent
choice over the stated index set.  ``Finite sample'' means that the cited
result is nonasymptotic.  ``Cov. shape'' requires a joint region or
support bound based on a covariance quadratic form.  ``Eff. dim.''
requires a proved width bound governed by population ridge effective
dimension rather than ambient feature dimension.
\(\checkmark\) records a property stated by the result,
\(\checkmark^{\dagger}\) an immediate projection consequence, and
\(\times\) means ``not established by the cited result,'' not ``impossible.''
\({}^{*}\)For the cited finite-grid theorem, selection is from a returned set
whose grid points jointly satisfy the specified risk threshold with the
theorem's stated confidence.
\({}^{\dagger}\)For a joint mean region, selection over a continuum of
directions and two-sided directional bounds follow by projection.
For the regularized mean-region row, the symmetric case uses finite covariance
for the mean interpretation; the
general case uses the cited eighth-moment, dimension-ratio, and
covariance-spectrum conditions.
\({}^{\ddagger}\)In fixed finite dimension with known deterministic
coordinate ranges, applying the cited scalar inequality to both signs of
every coordinate and taking a union bound gives a rectangular finite-sample
mean region; projection gives any-selector validity, continuous direction
indexing, and two-sided bounds.}
\end{table}

For the Hoeffding construction, let \(X_1,\ldots,X_n\in\mathbb R^d\)
be independent with common mean \(\mu\), and suppose
\(a_j\le X_{ij}\le b_j\) almost surely.  With
\[
 r_j=(b_j-a_j)\sqrt{\frac{\log(2d/\alpha)}{2n}},
 \qquad j=1,\ldots,d,
\]
Hoeffding's Theorem~2 \cite{hoeffding1963probability}, applied to both signs and all
coordinates, gives with probability at least \(1-\alpha\)
\[
 |h^\top(P_nX-\mu)|\le\sum_{j=1}^d|h_j|r_j
 \qquad\text{for every }h\in\mathbb R^d.
\]
Simultaneous validity of the coordinatewise Hoeffding inequality for every
\(h\in\mathbb R^d\) protects a direction chosen using the observations used to form the
Hoeffding mean region.  The Hoeffding mean region has axis-aligned geometry
determined by the known coordinate ranges rather
than a covariance quadratic form.

Bartlett and Mendelson \cite[Theorem~8]{bartlett2002rademacher} prove that,
for a sample-independent class satisfying the bounded-loss or bounded
dominating-cost conditions in Bartlett and Mendelson's Theorem~8, an upper risk bound holds for every function in
the class on one event of probability at least \(1-\alpha\); a function chosen
from the class after inspecting the data therefore retains the upper risk
bound attached to the selected function.  Angelopoulos et al.\ \cite[Theorem~2.1]{angelopoulos2025learnthen}
instead test, on a finite grid, whether each point's risk exceeds a chosen
threshold.  Given valid \(p\)-values and familywise-error control at level
\(\alpha\), every point in the returned subset has risk at most the threshold
with probability at least \(1-\alpha\), including a point selected from the returned
subset using the data used for finite-grid testing.  By contrast, our linear field addresses a
different target: Theorem~
\ref{thm:balanced-one-split-linear-confidence-field} gives lower and upper
bounds for signed population-risk change relative to the current model,
simultaneously over a possibly continuous update space; Corollary~
\ref{cor:balanced-one-split-oracle-width} gives the covariance-shaped and
population ridge effective-dimension width properties recorded in
Table~\ref{tab:intro-related-work}.  For prediction performance after choosing
among finitely many trained models, held-out selection procedures use an
approximate parametric correction \cite{westphal2019testdata} or
proposed bootstrap lower bounds adjusted for selecting among the models
\cite[Section~2, Equations~(12)--(14)]{rink2025postselection}.
Zrnic and Jordan
\cite[Definitions~4.1--4.2 and Theorem~5.2]{zrnic2023postselection} adjust
fixed-target intervals for a specified randomized selection rule satisfying
an explicit stability condition.
The parametric correction is approximate, the bootstrap construction is
proposed without a direct finite-sample theorem for the complete procedure,
and the stability correction applies to the specified randomized selection
rule rather than giving a direct joint bound over the update space.

\FloatBarrier
Lee et al.\ \cite[Theorem~6.1]{lee2016exact} give exact two-sided
Gaussian intervals for coefficients selected by the Lasso, conditional on the
selected model.  Lee et al.'s conditional interval result is a
selector-specific finite-sample guarantee over a
finite index set.  Berk et al.\
\cite[Theorem~4.1 and Corollary~4.1]{berk2013valid} give two-sided
finite-sample intervals simultaneously for every coefficient in a
prespecified finite model family, so one event protects any selector into the
prespecified model family.  Bachoc et al.\
\cite[Section~2.1 and Theorems~2.4 and~2.6]{bachoc2020uniformly} allow general
finite-dimensional model-specific targets and protect any selector from a
fixed finite model family with two-sided intervals, but the guarantee in
Bachoc et al. is asymptotic; the regression applications in Bachoc et al.
specialize the targets to parameter coordinates.  The conditional laws or
critical values in Lee et al., Berk et al., and Bachoc et al. account for
dependence, but none of the three cited results gives the covariance-shaped
update-space region or population ridge effective-dimension width bound in the
final two columns of Table~\ref{tab:intro-related-work}.  The present linear
confidence field
instead covers signed population-risk changes over a possibly continuous
update space.

Classical finite-dimensional Gaussian mean ellipsoids and simultaneous
Gaussian intervals for all contrasts allow a direction or contrast to be
chosen after inspection
\cite{hotelling1931generalization}
\cite[Equation~(8)]{scheffe1953contrasts}.  Regularized finite-dimensional
mean regions retain covariance-quadratic geometry beyond Gaussian models.
Issouani et al.\ \cite[Theorem~2 and the mean-testing acceptance region;
Theorem~3]{issouani2024regularizedhotelling} require only finite covariance
for the mean interpretation under symmetry, without an additional
higher-moment condition; for general distributions, Issouani et al. require
eighth moments together with dimension-ratio and covariance-spectrum
conditions.
For bounded adapted vectors in finite-dimensional Euclidean space, a
self-normalized empirical-Bernstein inequality gives a time-uniform mean
ellipsoid under a common conditional mean
\cite[Theorem~6.1]{whitehouse2026timeuniform}.
For bounded adapted observations with a common conditional mean in a
separable 2-smooth Banach space, including a separable Hilbert space,
\cite[Corollary~1 and Section~4.1]{martineztaboada2026banach} gives a norm ball centered at a predictably weighted mean whose
radius uses observed squared norm deviations.
Projection gives simultaneous bounds for all continuous linear directions for
the Euclidean ellipsoid and the Banach-space ball.  The finite-dimensional ellipsoid has sequential empirical
self-normalizer geometry, whereas the Banach-space ball retains the ambient
norm; neither Whitehouse et al. nor Mart\'inez-Taboada and Ramdas gives the
population ridge effective-dimension width recorded
under ``Eff. dim.'' in Table~\ref{tab:intro-related-work}.

For low-dimensional target parameters with high-dimensional or nonparametric
nuisance functions, cross-fitting fits the nuisance functions on one fold,
evaluates an estimating equation for the target parameter on another, then
swaps the fold roles and aggregates the estimates
\cite[Definitions~3.1--3.2]{chernozhukov2018double}.  Cross-calibration uses
the folds differently: one fold supplies the covariance geometry, while the
difference between the fold means calibrates a confidence field centered on
the full training sample
\eqref{eq:linear-core-cross-calibration}.

For square-loss regression with uniformly bounded functions and residuals, a
random-sign empirical-process bound with a negative quadratic term controls
the expected excess risk of a two-stage empirical risk minimizer
\cite[Equation~(2) and Theorem~4]{liang2015offset}.  The excess-risk bound is
in expectation, whereas \eqref{eq:intro-simultaneous-coverage} is a high-probability
simultaneous guarantee.  Theorem~
\ref{thm:centered-linear-expected-certificate} gives a
mean-valid quadratic field under second moments, and Proposition~
\ref{prop:mean-valid-decision} gives an expected-risk comparison for an
approximate minimizer of the mean-valid quadratic field.

For nonlinear loss increments, a sample-independent full increment class
satisfying the bounded-loss and bounded dominating-cost conditions of
\cite[Theorem~8]{bartlett2002rademacher} can be controlled directly through random-sign averages.  For iid observations and a
pointwise-measurable centered class with a measurable \(L_q(P)\) envelope for
\(q\ge3\), \(P\)-pre-Gaussianity, and the cited third-moment and local
entropy/process conditions, a Gaussian supremum can approximate the centered
empirical-process supremum
\cite[Theorem~2.1 and Corollary~2.1]{chernozhukov2014gaussian}.
Under the separate countable-approximation, VC-type entropy, \(q\ge4\)
envelope and class moment, and complexity conditions of \cite[Theorem~2.2]{chernozhukov2016empirical},
random-multiplier bootstrap samples admit a conditional Gaussian coupling;
distributional calibration additionally uses the anti-concentration step in
\cite[Remark~2.3]{chernozhukov2016empirical}.  Theorem~
\ref{thm:componentwise-hybrid-local-confidence} instead retains separate
contributions from the linear term, the discrepancy at branch interfaces,
and the Taylor remainder.
 \section{Confidence Fields for Linear Risk Changes}
\label{sec:finite-sample-certificates}
Throughout this section, \(\mathcal D\) is a nonempty deterministic Borel
subset of a separable metric update space and carries the induced Borel
\(\sigma\)-field.  Assume that \(P\delta_v\in\mathbb R\) for every
\(v\in\mathcal D\) and that \(v\mapsto P\delta_v\) is Borel measurable.
Joint measurability of a random field indexed by \(\mathcal D\) refers to the
sample \(\sigma\)-field and the induced Borel \(\sigma\)-field on
\(\mathcal D\).

Let \(X=X(Z)\) be an integrable random element in a real separable Hilbert
space \(\mathcal H\), and split
\(n=2N\) observations, \(N\ge2\), into folds \(A\) and \(B\), each of size
\(N\).  Write \(\overline X_A=N^{-1}\sum_{i\in A}X_i\), define
\(\overline X_B\) analogously, and put
\(\overline X=(\overline X_A+\overline X_B)/2\).  We call \(A\) the
\emph{pilot fold} and \(B\) the \emph{evaluation fold}, and write
\(S_A:=(Z_i)_{i\in A}\) for the observed pilot sample.

For a bounded, self-adjoint, positive operator \(Q\) on \(\mathcal H\), write
\(\|x\|_Q^2:=\langle x,Qx\rangle\).  Let \(M_A\) be an \(S_A\)-measurable,
self-adjoint, strictly positive operator on \(\mathcal H\) with a bounded
inverse, where measurability is with respect to the operator norm.  The balanced split gives
\begin{align}
 \|PX-\overline X\|_{M_A^{-1}}^2
 -\frac14\|\overline X_B-\overline X_A\|_{M_A^{-1}}^2
 &=-\left\langle
 M_A^{-1/2}(PX-\overline X_A),
 M_A^{-1/2}(\overline X_B-PX)
 \right\rangle .
\label{eq:linear-core-cross-calibration}
\end{align}
We call the subtraction of the cross-fold discrepancy in
\eqref{eq:linear-core-cross-calibration} \emph{cross-calibration}.  Conditional
on \(S_A\), \(M_A^{-1/2}(PX-\overline X_A)\) is fixed and
\(M_A^{-1/2}(\overline X_B-PX)\) is centered.  If the inner product in
\eqref{eq:linear-core-cross-calibration} is integrable, then
\[
 \mathbb E\!\left[
 \left\langle M_A^{-1/2}(PX-\overline X_A),
 M_A^{-1/2}(\overline X_B-PX)\right\rangle\middle|S_A\right]=0.
\]

Under second moments, averaging the cross-calibrated penalties over all pilot
folds of a fixed size gives the two signed expected-supremum inequalities for
\((P-P_n)\langle X,h\rangle\) over every deterministic nonempty
\(K\subset\mathcal H\) in
Theorem~\ref{thm:centered-linear-expected-certificate}.  Under the
covariance-matched sub-Gaussian tail condition, for every
\(\alpha\in(0,1)\) satisfying the sample-size requirements in
Theorem~\ref{thm:balanced-one-split-linear-confidence-field}, a balanced split
gives, with probability at least \(1-\alpha\), a finite-sample ellipsoid for
\(PX-P_nX\) and simultaneous bounds for
\((P-P_n)\langle X,h\rangle\), \(h\in\mathcal H\).  Under the
stronger requirements of
Corollary~\ref{cor:balanced-one-split-oracle-width}, one event of probability
at least \(1-\alpha\) bounds the empirical metric by the population ridge
metric and the radius by the ridge effective dimension and confidence level,
yielding a population-geometry width bound for every direction.

\subsection{Mean validity and decision consequences}
\label{subsec:expected-local-certificates}

We seek random upper and lower fields that bound
\(P\delta_{\widehat v}\) in expectation for every measurable update
\(\widehat v\in\mathcal D\) chosen using the observations used to construct
the fields.
Cavalier and
Golubev use an analogous expected-supremum condition for risk hulls to retain
an expected loss bound after data-dependent bandwidth selection
\cite[Section~2.1, Equations~(2.2)--(2.3)]{cavalier2006risk}.

\begin{definition}[Mean-valid population-risk fields]
\label{def:expected-local-endpoint}
A jointly measurable random field
\(\widehat{\mathsf U}_{\mathcal D}:\mathcal D\to
\mathbb R\cup\{+\infty\}\) is a \emph{mean-valid upper field} if
\[
 G_{\mathsf U}:=\sup_{v\in\mathcal D}
 \{P\delta_v-\widehat{\mathsf U}_{\mathcal D}(v)\}
\]
is measurable, \(\mathbb E G_{\mathsf U}\) is well defined as an extended
expectation, and \(\mathbb E G_{\mathsf U}\le0\).
A jointly measurable random field
\(\widehat{\mathsf L}_{\mathcal D}:\mathcal D\to
\mathbb R\cup\{-\infty\}\) is a \emph{mean-valid lower field} if
\[
 G_{\mathsf L}:=\sup_{v\in\mathcal D}
 \{\widehat{\mathsf L}_{\mathcal D}(v)-P\delta_v\}
\]
is measurable, \(\mathbb E G_{\mathsf L}\) is well defined as an extended
expectation, and \(\mathbb E G_{\mathsf L}\le0\).
\end{definition}

For every measurable same-sample selector
\(\widehat v=\widehat v(Z_1,\ldots,Z_n)\in\mathcal D\),
measurability of the target map and joint measurability of the fields ensure
that \(P\delta_{\widehat v}\),
\(\widehat{\mathsf U}_{\mathcal D}(\widehat v)\) and
\(\widehat{\mathsf L}_{\mathcal D}(\widehat v)\) are measurable.  The selected
gaps satisfy
\[
 P\delta_{\widehat v}
 -\widehat{\mathsf U}_{\mathcal D}(\widehat v)\le G_{\mathsf U},
 \qquad
 \widehat{\mathsf L}_{\mathcal D}(\widehat v)
 -P\delta_{\widehat v}\le G_{\mathsf L}.
\]
Definition~\ref{def:expected-local-endpoint} makes
\((G_{\mathsf U})_+\) and \((G_{\mathsf L})_+\) integrable, so the selected
gaps have well-defined extended expectations and
\[
 \mathbb E\{P\delta_{\widehat v}
 -\widehat{\mathsf U}_{\mathcal D}(\widehat v)\}\le0,
 \qquad
 \mathbb E\{\widehat{\mathsf L}_{\mathcal D}(\widehat v)
 -P\delta_{\widehat v}\}\le0.
\]
Mean-valid upper fields can guide the update choice: Proposition~
\ref{prop:mean-valid-decision} compares an approximate minimizer with fixed
updates in expected population-risk change.
The proof of Proposition~\ref{prop:mean-valid-decision} is in
\SuppSection{supp:sec:decomposition-validity}.

\begin{proposition}[Decision interpretation of mean-valid fields]
\label{prop:mean-valid-decision}
Let \(\widehat{\mathsf U}_{\mathcal D}\) be a finite-valued mean-valid upper
field of the empirical-center form
\[
 \widehat{\mathsf U}_{\mathcal D}(v)
 =P_n\delta_v+\widehat{\mathsf A}_{\mathcal D}(v),
\]
and suppose that \(\widehat{\mathsf A}_{\mathcal D}(v)\) is integrable for
every deterministic \(v\in\mathcal D\).  Let
\(\varepsilon_{\rm opt}\ge0\) be deterministic.  Suppose that a measurable
\(\widehat v\in\mathcal D\) satisfies, on a single event of probability one,
\[
 \widehat{\mathsf U}_{\mathcal D}(\widehat v)
 \le \widehat{\mathsf U}_{\mathcal D}(v)+\varepsilon_{\rm opt},
 \qquad \text{for every }v\in\mathcal D.
\]
If \(P\delta_{\widehat v}\) is integrable, then
\begin{equation}
\label{eq:mean-valid-risk-oracle}
 \mathbb E\{P\delta_{\widehat v}\}
 \le \inf_{v\in\mathcal D}
 \{P\delta_v+\mathbb E\widehat{\mathsf A}_{\mathcal D}(v)\}
 +\varepsilon_{\rm opt}.
\end{equation}
Moreover, every measurable same-sample selector
\(\widetilde v\in\mathcal D\) for which
\((P-P_n)\delta_{\widetilde v}\) and
\(\widehat{\mathsf A}_{\mathcal D}(\widetilde v)\) are integrable obeys
\begin{equation}
\label{eq:mean-valid-selected-gap}
 \mathbb E\{(P-P_n)\delta_{\widetilde v}\}
 \le\mathbb E\widehat{\mathsf A}_{\mathcal D}(\widetilde v).
\end{equation}
\end{proposition}

In particular, if \(P\delta_{\widehat v}\) and
\(\widehat{\mathsf U}_{\mathcal D}(\widehat v)\) are integrable and
\(
 \mathbb E\widehat{\mathsf U}_{\mathcal D}(\widehat v)<0,
\)
then \(\mathbb E\{P\delta_{\widehat v}\}<0\).

If \(0\in\mathcal D\), \(\delta_0\equiv0\), and
\(\widehat{\mathsf U}_{\mathcal D}(0)=0\) almost surely, mean validity forces
the upper field to be simultaneously valid almost surely; mean validity of a
lower field with \(\widehat{\mathsf L}_{\mathcal D}(0)=0\) likewise forces
simultaneous lower validity almost surely.
\SuppResult{Lemma}{supp:lem:null-anchor-obstruction} proves the upper- and
lower-field null-anchor statements.

The Gaussian model in
Proposition~\ref{prop:orientation-selection-mean-cost} is constructed here to
isolate the cost of choosing between fields from different folds.  Taking the
pointwise minimum of the two individually mean-valid fold-specific fields
incurs a positive mean-validity defect, whereas
the cross-calibrated field constructed from the two fold means remains exactly
mean-valid.
The contrast is also operational: the corrections \(c_N\) and \(d_N\) depend
on the known variance \(\sigma^2\), whereas \(U_{\rm cf}\) is constructed from
the two fold means and does not require knowledge of \(\sigma^2\).

The proof of Proposition~\ref{prop:orientation-selection-mean-cost} is in
\SuppSection{supp:sec:decomposition-validity}.

\begin{proposition}[Mean-valid Gaussian fields]
\label{prop:orientation-selection-mean-cost}
Fix \(M,\eta>0\).  Let \(\overline X_A\) and \(\overline X_B\) be independent
\(N(\mu,\sigma^2/N)\) variables with known \(\sigma^2>0\).  Define the
single-field correction \(c_N\) and lower-envelope defect \(d_N\) by
\[
 c_N:=\frac{\sigma^2}{2\eta MN},\qquad
 d_N:=\frac{\sigma^2}{\pi\eta MN}.
\]
For \(J\in\{A,B\}\), define
\[
 U_J(h)=\overline X_Jh+\frac{\eta M}{2}h^2+c_N.
\]
Each \(U_J\) is exactly mean-valid on \(\mathbb R\).  The lower envelope
\(U_\wedge:=\min(U_A,U_B)\), however, has defect
\begin{equation}
\label{eq:orientation-selection-mean-cost}
 \mathbb E\sup_{h\in\mathbb R}\{\mu h-U_\wedge(h)\}
 =d_N.
\end{equation}
Consequently, \(U_\wedge+d_N\) is exactly mean-valid.  The
cross-calibrated quadratic
\[
 U_{\rm cf}(h)
 =\frac{\overline X_A+\overline X_B}{2}h
 +\frac{\eta M}{2}h^2
 +\frac{(\overline X_A-\overline X_B)^2}{8\eta M}
\]
is also exactly mean-valid.
\end{proposition}
\subsection{Cross-fitted mean-valid fields}
\label{subsec:cross-fitted-quadratic-calibration}

Let \(n\ge4\), write \([n]:=\{1,\ldots,n\}\), and let
\(X=X(Z)\in L^2(P;\mathcal H)\).  The covariance operator of \(X\) is
\(
 \Sigma_X:=\mathbb E\{(X-PX)\otimes(X-PX)\}.
\)
Fix a pilot size
\begin{equation}
\label{eq:crossfit-pilot-size}
 2\le n_{\rm p}\le n_{\rm e}:=n-n_{\rm p},
\end{equation}
and write
\[
 \mathfrak A_{n_{\rm p}}:=\{A\subset[n]:|A|=n_{\rm p}\},
 \qquad B_A:=[n]\setminus A.
\]
For a nonempty index set \(\mathcal I\subset[n]\), write
\[
 P_{\mathcal I}f:=\frac1{|\mathcal I|}\sum_{i\in\mathcal I}f(Z_i),
 \qquad
 \overline X_{\mathcal I}:=P_{\mathcal I}X.
\]
Conditional on the pilot sample \(S_A\), \(\mathbb E_{B_A}\) denotes
expectation over the independent observations indexed by \(B_A\), and
analogously after reversing the two folds.
When \(|\mathcal I|\ge2\), define the unbiased covariance operator
\begin{align}
 \widehat\Sigma_{X,\mathcal I}
 :&=\frac1{|\mathcal I|-1}\sum_{i\in\mathcal I}
 (X_i-\overline X_{\mathcal I})\otimes(X_i-\overline X_{\mathcal I})\\
 &=\frac1{|\mathcal I|(|\mathcal I|-1)}
 \sum_{\{i,j\}\subset\mathcal I}
 (X_i-X_j)\otimes(X_i-X_j).
 \label{eq:sample-covariance-pairwise}
\end{align}
The full-sample quantities are denoted by
\(\overline X=\overline X_{[n]}\) and
\(\widehat\Sigma_X=\widehat\Sigma_{X,[n]}\).  For
\(A\in\mathfrak A_{n_{\rm p}}\) and \(\lambda>0\), let
\(
 M_{X,A,\lambda}:=\widehat\Sigma_{X,A}+\lambda I.
\)
Every inverse of \(M_{X,A,\lambda}\) is bounded because
\(M_{X,A,\lambda}\succeq\lambda I\).

The complete split average is denoted by
\[
 \operatorname{Av}_{A}g_A
 :=\binom{n}{n_{\rm p}}^{-1}\sum_{A\in\mathfrak A_{n_{\rm p}}}g_A.
\]
The identities
\[
 \operatorname{Av}_{A}\widehat\Sigma_{X,A}=\widehat\Sigma_X,
 \qquad
 \operatorname{Av}_{A}
 (\overline X_{B_A}-\overline X_A)^{\otimes2}
 =\frac{n}{n_{\rm p}n_{\rm e}}\widehat\Sigma_X
\]
are proved by inclusion counting in
\SuppResult{Lemma}{lem:crossfit-u-average-identities}.

For one split, define the held-out quadratic calibration statistic
\begin{align}\label{eq:crossfit-single-split-calibration}
 \widehat T_{X,A,\lambda}
 :={}&\frac{n_{\rm p}^2}{n^2}
 \|\overline X_{B_A}-\overline X_A\|_{M_{X,A,\lambda}^{-1}}^2+\frac{n_{\rm e}-n_{\rm p}}{n_{\rm e}n}
 \operatorname{tr}\!
 \left\{\widehat\Sigma_{X,B_A}M_{X,A,\lambda}^{-1}\right\}.
\end{align}
The trace in \eqref{eq:crossfit-single-split-calibration} is nonnegative, and
the coefficient \((n_{\rm e}-n_{\rm p})/(n_{\rm e}n)\) is nonnegative by
\eqref{eq:crossfit-pilot-size}; hence \(\widehat T_{X,A,\lambda}\ge0\).
The scalar used in the calibration intercept is the complete-split average of
\(\widehat T_{X,A,\lambda}\), scaled by \(n\):
\begin{equation}
\label{eq:centered-linear-crossfit-effective-dimension}
 \widehat r_{X,n_{\rm p},\lambda}^{\rm cf}:=n\operatorname{Av}_{A}\widehat T_{X,A,\lambda}.
\end{equation}
The empirical ridge quadratic form is
\[
 \widehat q_{X,\lambda}(h):=\langle h,(\widehat\Sigma_X+\lambda I)h\rangle.
\]

Fix a deterministic nonempty \(K\subset\mathcal H\) and \(\eta>0\).
Conditional on the pilot sample \(S_A\), the operator
\(M_{X,A,\lambda}\) is fixed.  Completing the square gives, for either
\(\sigma\in\{-1,1\}\),
\begin{equation}
\label{eq:splitwise-quadratic-conjugate-bound}
 \sup_{h\in K}\left\{\sigma\langle PX-P_nX,h\rangle
 -\frac\eta2\langle h,M_{X,A,\lambda}h\rangle\right\}
 \le \frac1{2\eta}
 \|PX-P_nX\|_{M_{X,A,\lambda}^{-1}}^2.
\end{equation}
By \SuppResult{Proposition}{prop:conditional-linear-intercept-optimality},
taking the supremum of the quadratic criterion in
\eqref{eq:splitwise-quadratic-conjugate-bound} over \(h\in K\), rather than
over \(h\in\mathcal H\), lowers the conditional expected supremum by \(\eta/2\)
times the conditional expectation of the squared \(M_{X,A,\lambda}\)-norm
distance from the maximizer over \(\mathcal H\) to \(K\).

The mean \(PX\) in the squared dual-metric error is unknown.  For the held-out
statistic in \eqref{eq:crossfit-single-split-calibration},
\SuppResult{Lemma}{lem:conditional-second-moment-calibration} proves the
conditional unbiasedness identity
\[
 \mathbb E_{B_A}\widehat T_{X,A,\lambda}
 =\mathbb E_{B_A}
 \|PX-P_nX\|_{M_{X,A,\lambda}^{-1}}^2
 \quad\text{almost surely in }S_A.
\]
Standard cross-fitting fits nuisance functions off-fold and evaluates scores
on the held-out fold
\cite[Definitions~3.1--3.2]{chernozhukov2018double}.  Here the held-out fold
instead calibrates the conditional mean of the error measured in the
pilot-fold covariance metric.  Equation~
\eqref{eq:splitwise-quadratic-conjugate-bound} and
\SuppResult{Lemma}{lem:conditional-second-moment-calibration} give, for each
split and both signs,
\begin{equation}
\label{eq:splitwise-conditional-mean-validity}
 \mathbb E_{B_A}\sup_{h\in K}
 \left\{\sigma\langle PX-P_nX,h\rangle
 -\frac\eta2\langle h,M_{X,A,\lambda}h\rangle
 -\frac{\widehat T_{X,A,\lambda}}{2\eta}\right\}\le0
 \quad\text{almost surely in }S_A.
\end{equation}
After taking full expectations, the complete split average remains
nonpositive: the supremum over \(h\in K\) of an average is at most the average
of the splitwise suprema in
\eqref{eq:splitwise-conditional-mean-validity}.
\SuppEquation{eq:crossfit-average-pilot-covariance} gives the averaged
quadratic \(\widehat q_{X,\lambda}(h)\), while
\eqref{eq:centered-linear-crossfit-effective-dimension} writes the averaged
calibration term as
\(\widehat r_{X,n_{\rm p},\lambda}^{\rm cf}/(2\eta n)\).

\begin{theorem}[Cross-fitted mean-valid field for a linear functional class]
\label{thm:centered-linear-expected-certificate}
Let \(f_h(z):=\langle X(z),h\rangle\), and let
\(K\subset\mathcal H\) be deterministic and nonempty.  For
\(\eta,\lambda>0\), define
\[
 \mathcal A_{X,n_{\rm p},\eta,\lambda}^{\rm cf}(h)
 :=\frac\eta2\widehat q_{X,\lambda}(h)
 +\frac{\widehat r_{X,n_{\rm p},\lambda}^{\rm cf}}{2\eta n}.
\]
Then, for both \(\sigma\in\{-1,1\}\),
\begin{equation}
\label{eq:centered-linear-expected-validity}
 \mathbb E\sup_{h\in K}
 \left\{\sigma(P-P_n)f_h
 -\mathcal A_{X,n_{\rm p},\eta,\lambda}^{\rm cf}(h)\right\}
 \le0.
\end{equation}
\end{theorem}

The proof is in \SuppSection{supp:subsec:linear-mean-valid-proof}.

\medskip
\noindent\textit{Expected-risk consequence.}
Let \(\widehat h\) be an \(\varepsilon_{\rm opt}\)-approximate minimizer over
\(K\) of
\[
 P_nf_h+\frac\eta2\widehat q_{X,\lambda}(h)
 +\frac{\widehat r_{X,n_{\rm p},\lambda}^{\rm cf}}{2\eta n}.
\]
Write \(q_{X,\lambda}\) for the population ridge quadratic form and
\(\bar r_{X,n_{\rm p},\lambda}^{\rm cf}\) for the expected cross-fitted
calibration:
\[
 q_{X,\lambda}(h)
 :=\langle h,(\Sigma_X+\lambda I)h\rangle,
 \qquad
 \bar r_{X,n_{\rm p},\lambda}^{\rm cf}
 :=\mathbb E\widehat r_{X,n_{\rm p},\lambda}^{\rm cf}.
\]
For deterministic \(h\), unbiasedness of \(\widehat\Sigma_X\) gives
\[
 \mathbb E\widehat q_{X,\lambda}(h)
 =\langle h,(\Sigma_X+\lambda I)h\rangle
 =q_{X,\lambda}(h).
\]
Hence Proposition~\ref{prop:mean-valid-decision}, applied with expected
penalty
\(\eta q_{X,\lambda}(h)/2+
\bar r_{X,n_{\rm p},\lambda}^{\rm cf}/(2\eta n)\), gives
\begin{equation}
\label{eq:linear-mean-valid-oracle-bound}
 \mathbb EPf_{\widehat h}
 \le\inf_{h\in K}\left\{Pf_h+\frac\eta2
 q_{X,\lambda}(h)\right\}
 +\frac{\bar r_{X,n_{\rm p},\lambda}^{\rm cf}}{2\eta n}
 +\varepsilon_{\rm opt}.
\end{equation}

Empirical Bernstein and variance-localized Rademacher bounds retain variance
information through different class-dependent complexity terms.  Empirical
Bernstein inequalities can retain direction-specific empirical variance with
an empirical \(L_\infty\) covering factor for a bounded function class
\cite[Theorem~6]{maurer2009empirical}.  Variance-localized Rademacher bounds
partition a bounded function class by variance, control the local Rademacher
average of each slice, and summarize the slice-wise complexities through the fixed
point of a sub-root envelope
\cite[Theorem~3.3 and Section~3.1.4]{bartlett2005local}.  Dudley's entropy
integral, expressed through covering numbers, is one way to construct a
sub-root envelope
\cite[Section~3.1.1]{bartlett2005local}.  For the
quadratically penalized linear supremum in
Theorem~\ref{thm:centered-linear-expected-certificate}, the enlargement from
\(K\) to the ambient Hilbert space in
\eqref{eq:splitwise-quadratic-conjugate-bound} evaluates the quadratic
conjugate as a single squared mean error in the dual metric, with no
Rademacher or covering-number term for \(K\).

\subsection{Balanced pilot-conditional confidence fields}
\label{subsec:balanced-pilot-conditional-confidence}

Theorem~\ref{thm:centered-linear-expected-certificate} requires only finite
second moments.  We now
strengthen the tail assumption so that
the cross term in
\eqref{eq:linear-core-cross-calibration} can be bounded with high probability
conditional on the pilot sample \(S_A\).

For a real random variable \(W\), write
\(\|W\|_{\psi_2}:=\inf\{c>0:\mathbb E\exp(W^2/c^2)\le2\}\).
Throughout this subsection, let \(n=2N\) with \(N\ge2\), let
\(A\subset[n]\) be a balanced
pilot fold with \(|A|=|B_A|=N\), and assume that a known constant
\(\bar\kappa\ge1\) satisfies
\begin{equation}
\label{eq:main-covariance-matched-subgaussian}
 \|\langle u,X-PX\rangle\|_{\psi_2}
 \le \bar\kappa\langle u,\Sigma_Xu\rangle^{1/2},
 \qquad u\in\mathcal H.
\end{equation}
Condition \eqref{eq:main-covariance-matched-subgaussian} controls every
linear functional relative to the standard deviation of the linear
functional.  In particular,
\cite[Proposition~2.6.1]{vershynin2025high} implies that, for a universal
constant \(C\),
\[
 \mathbb E\exp\{\theta\langle u,X-PX\rangle\}
 \le
 \exp\{C\bar\kappa^2\theta^2\langle u,\Sigma_Xu\rangle\},
 \qquad \theta\in\mathbb R.
\]
When \(n\) is odd, the balanced construction may be
applied after discarding one observation chosen deterministically in advance.
Define the regularized population covariance operator \(M_\lambda\) and the
ridge effective dimension \(d_{X,\lambda}\) by
\begin{equation}
\label{eq:population-ridge-effective-dimension}
 M_\lambda:=\Sigma_X+\lambda I,
 \qquad
 d_{X,\lambda}:=\operatorname{tr}(\Sigma_XM_\lambda^{-1}).
\end{equation}
If \((\sigma_j)\) are the eigenvalues of \(\Sigma_X\), then
\(d_{X,\lambda}=\sum_j\sigma_j/(\sigma_j+\lambda)\); the ridge effective
dimension \(d_{X,\lambda}\) measures the number
of covariance directions retained at ridge resolution \(\lambda\).
For the high-probability construction, take the pilot operator \(M_A\) in
\eqref{eq:linear-core-cross-calibration} to be
\begin{equation}
\label{eq:pilot-ridge-metric}
 M_A:=M_{X,A,\lambda}=\widehat\Sigma_{X,A}+\lambda I.
\end{equation}
Define the pilot-standardized covariance \(\mathcal V_A\), trace \(r_A\),
squared-spectral trace \(s_A\), and operator norm \(\rho_A\) by
\begin{equation}
\label{eq:conditional-pilot-geometry}
 \mathcal V_A:=M_A^{-1/2}\Sigma_XM_A^{-1/2},\qquad
 r_A:=\operatorname{tr}(\mathcal V_A),\quad
 s_A:=\operatorname{tr}(\mathcal V_A^2),\quad
 \rho_A:=\|\mathcal V_A\|_{\rm op}.
\end{equation}
The operator norm \(\rho_A\) is the largest covariance-to-pilot-metric ratio.

Balanced folds remove the trace term in
\eqref{eq:crossfit-single-split-calibration}, so
\begin{equation}
\label{eq:balanced-single-split-calibration}
 \widehat T_{X,A,\lambda}
 =\frac14\|\overline X_{B_A}-\overline X_A\|_{M_A^{-1}}^2.
\end{equation}
Define the \(\mathcal H\)-valued standardized pilot and evaluation errors by
\begin{equation}
\label{eq:balanced-standardized-fold-errors}
 b_A:=M_A^{-1/2}(PX-\overline X_A),
 \qquad
 \xi_A:=M_A^{-1/2}(\overline X_{B_A}-PX).
\end{equation}
For arbitrary
\(n_{\rm p}+n_{\rm e}=n\), one has
\[
 M_A^{-1/2}(PX-P_nX)
 =\frac{n_{\rm p}}n b_A-\frac{n_{\rm e}}n\xi_A,
 \qquad
 M_A^{-1/2}(\overline X_{B_A}-\overline X_A)=b_A+\xi_A.
\]
Substitution in \eqref{eq:crossfit-single-split-calibration} gives
\begin{subequations}
\begin{align}
 &\|PX-P_nX\|_{M_A^{-1}}^2-\widehat T_{X,A,\lambda}\notag\\
 &\quad=-\frac{2n_{\rm p}}n\langle b_A,\xi_A\rangle
 +\frac{n_{\rm e}-n_{\rm p}}n
 \left\{\|\xi_A\|^2-\frac1{n_{\rm e}}
 \operatorname{tr}(\widehat\Sigma_{X,B_A}M_A^{-1})\right\}.
\label{eq:unequal-calibration-cross-term}
\end{align}
For \(n_{\rm p}=n_{\rm e}=N\), the coefficient
\((n_{\rm e}-n_{\rm p})/n\) vanishes and
\eqref{eq:unequal-calibration-cross-term} reduces to
\begin{equation}
\label{eq:balanced-calibration-cross-term}
 \|PX-P_nX\|_{M_A^{-1}}^2-\widehat T_{X,A,\lambda}
 =-\langle b_A,\xi_A\rangle.
\end{equation}
\end{subequations}
The cancellation in \eqref{eq:balanced-calibration-cross-term} is the reason
for using balanced folds in the high-probability construction.  With unequal
fold sizes, the term multiplied by \((n_{\rm e}-n_{\rm p})/n\) in
\eqref{eq:unequal-calibration-cross-term} requires joint concentration of the
evaluation-fold mean and covariance; outside the Gaussian case, the
evaluation-fold mean and covariance are generally dependent.
Conditional on \(S_A\), the metric and \(b_A\) are fixed, whereas \(\xi_A\)
is an average of independent centered sub-Gaussian vectors.

To obtain a data-dependent upper bound for
\(N^{-1}\operatorname{tr}(\Sigma_XM_A^{-1})\), define
the evaluation-fold trace
\begin{equation}
\label{eq:conditional-evaluation-trace}
 \widehat r_{B_A\mid A}
 :=\operatorname{tr}\{\widehat\Sigma_{X,B_A}M_A^{-1}\}.
\end{equation}
Fix \(C_{\rm tr}(\bar\kappa),C_{\rm cr}(\bar\kappa)\ge1\), depending only on
\(\bar\kappa\), as constructed in
\SuppSection{supp:subsec:pilot-conditional-confidence}.  The general
sub-Gaussian argument gives existence constants rather than numerically
optimized values; Corollaries~\ref{cor:gaussian-one-split-linear-confidence}
and~\ref{cor:bounded-feature-one-split-linear-confidence} give explicit
Gaussian and bounded-feature alternatives.  In
\eqref{eq:conditional-trace-relative-error}--\eqref{eq:balanced-quadratic-confidence-inflation}, \(R\ge0\) is a
deterministic upper bound to be verified for the
pilot-dependent ratio \(\rho_A\).  For \(t\ge1\), let
\begin{align}
 \varepsilon_N(t)
 &:=C_{\rm tr}(\bar\kappa)
 \left(\sqrt{\frac tN}+\frac tN+\frac1N\right),
 \label{eq:conditional-trace-relative-error}\\
 \overline r_{B_A\mid A}(t;R)
 &:=\frac{\widehat r_{B_A\mid A}+C_{\rm tr}(\bar\kappa)Rt/N}
 {1-\varepsilon_N(t)},
 \qquad \varepsilon_N(t)<1,
 \label{eq:conditional-evaluation-trace-ucb}\\
\widehat\Delta_{X,A,\lambda}(t;R)
 &:=C_{\rm cr}(\bar\kappa)\left[
 \sqrt{\frac{Rt\,\widehat T_{X,A,\lambda}}N}
 +\frac{\sqrt{R\overline r_{B_A\mid A}(t;R)t}+Rt}{N}
 \right].
 \label{eq:balanced-quadratic-confidence-inflation}
\end{align}

Conditionally on \(S_A\),
\cite[Corollary~2.6]{mollenhauer2023concentration} applied to \(\xi_A\)
bounds the evaluation-fold error \(\|\xi_A\|^2\); the threshold in the cited
corollary uses \(r_A,s_A,\rho_A\), all involving the unknown \(\Sigma_X\).
Equation~
\eqref{eq:balanced-calibration-cross-term} writes the full-sample error as
\(\widehat T_{X,A,\lambda}\) plus the scalar
\(-\langle b_A,\xi_A\rangle\).
Theorem~\ref{thm:balanced-one-split-linear-confidence} bounds
\(\|PX-P_nX\|_{M_A^{-1}}^2\) by
\(\widehat T_{X,A,\lambda}+\widehat\Delta_{X,A,\lambda}(t;R)\).

\begin{theorem}[Balanced pilot-conditional quadratic concentration]
\label{thm:balanced-one-split-linear-confidence}
Assume \eqref{eq:main-covariance-matched-subgaussian}.  Fix a balanced split
\(A\), \(R\ge0\), and \(t\ge1\) with \(\varepsilon_N(t)<1\).  For almost
every pilot realization satisfying \(\rho_A\le R\), conditionally on
\(S_A\), with probability at least \(1-3e^{-t}\), the quadratic-error and
trace inequalities hold simultaneously:
\begin{align}
\label{eq:balanced-one-split-quadratic-confidence}
 \|PX-P_nX\|_{M_A^{-1}}^2
 &\le \widehat T_{X,A,\lambda}
 +\widehat\Delta_{X,A,\lambda}(t;R),\\
 r_A&\le\overline r_{B_A\mid A}(t;R).\notag
\end{align}
\end{theorem}

For a Gaussian random element \(X-PX\), the conditional trace and cross-term
formulas define explicit bounds that replace \(C_{\rm tr}(\bar\kappa)\) and
\(C_{\rm cr}(\bar\kappa)\) by expressions involving the supplied Gaussian
bound \(L\).  The bounded-feature formula in
\eqref{eq:bounded-feature-quadratic-confidence-inflation} instead uses the
supplied feature bound \(B\).  The Gaussian and bounded-feature
specializations do not use the existential constants from the general
sub-Gaussian bounds:
\begin{align}
 a_N^{\rm G}(t;R)
 &:=\left(\frac{Rt}{N-1}\right)^{1/2},\notag\\
 \overline r_{B_A\mid A}^{\rm G}(t;R)
 &:=\left\{a_N^{\rm G}(t;R)
 +\sqrt{\{a_N^{\rm G}(t;R)\}^2+
 \widehat r_{B_A\mid A}}\right\}^2,
 \label{eq:gaussian-conditional-trace-ucb}\\
 \widehat\Delta_{X,A,\lambda}^{\rm G}(t;R)
 &:=2\sqrt{\frac{2Rt\,\widehat T_{X,A,\lambda}}N}
 +\frac{2\{\sqrt{R\overline r_{B_A\mid A}^{\rm G}(t;R)t}
 +Rt\}}N.
 \label{eq:gaussian-quadratic-confidence-inflation}
\end{align}

\begin{corollary}[Gaussian one-split ellipsoid]
\label{cor:gaussian-one-split-linear-confidence}
Fix \(\alpha\in(0,1)\).  Suppose \(X-PX\) is Gaussian and a known
\(0\le L<\infty\) satisfies
\(\|\Sigma_X\|_{\rm op}\le L\).  Put
\[
 R_0:=L/\lambda,\qquad
 t_\alpha^{\rm G}:=\log\frac3\alpha,
\]
and define
\begin{equation}
\label{eq:gaussian-finite-split-inflated-intercept}
 \widehat U_{X,A,\lambda,\alpha}^{\rm G}
 :=\widehat T_{X,A,\lambda}
 +\widehat\Delta_{X,A,\lambda}^{\rm G}(t_\alpha^{\rm G};R_0).
\end{equation}
Conditional on the balanced split \(A\) and the pilot sample \(S_A\), with
probability at least \(1-\alpha\),
\[
 r_A\le\overline r_{B_A\mid A}^{\rm G}(t_\alpha^{\rm G};R_0),\qquad
 \|PX-P_nX\|_{M_A^{-1}}^2
 \le\widehat U_{X,A,\lambda,\alpha}^{\rm G},
\]
and hence, simultaneously for every \(h\in\mathcal H\),
\[
 |(P-P_n)f_h|
 \le\{\langle h,M_Ah\rangle
 \widehat U_{X,A,\lambda,\alpha}^{\rm G}\}^{1/2}.
\]
\end{corollary}

For the bounded specialization, let a known \(B<\infty\) satisfy
\(\|X\|\le B\) almost surely and define
\begin{equation}
\label{eq:bounded-feature-quadratic-confidence-inflation}
 \widehat\Delta_{X,A,\lambda}^{\rm bd}(t;B)
 :=2B\sqrt{\frac{2t\,\widehat T_{X,A,\lambda}}{\lambda N}}
 +\frac{B^2\{\sqrt{2t}+2t\}}{\lambda N}.
\end{equation}

\begin{corollary}[One-split ellipsoid for bounded \(X\)]
\label{cor:bounded-feature-one-split-linear-confidence}
Fix \(\alpha\in(0,1)\).  Suppose \(\|X\|\le B\) almost surely for a known
\(B<\infty\), and define
\[
 \widehat U_{X,A,\lambda,\alpha}^{\rm bd}
 :=\widehat T_{X,A,\lambda}
 +\widehat\Delta_{X,A,\lambda}^{\rm bd}\bigl(\log(2/\alpha);B\bigr).
\]
Conditional on the balanced split \(A\) and the pilot sample \(S_A\), with
probability at least \(1-\alpha\),
\[
 \|PX-P_nX\|_{M_A^{-1}}^2
 \le\widehat U_{X,A,\lambda,\alpha}^{\rm bd},
\]
and, simultaneously for every \(h\in\mathcal H\),
\[
 |(P-P_n)f_h|
 \le
 \left\{\langle h,M_Ah\rangle
 \widehat U_{X,A,\lambda,\alpha}^{\rm bd}\right\}^{1/2}.
\]
\end{corollary}

Proofs of Theorem~\ref{thm:balanced-one-split-linear-confidence} and
Corollaries~\ref{cor:gaussian-one-split-linear-confidence}
and~\ref{cor:bounded-feature-one-split-linear-confidence} are in
\SuppSection{supp:subsec:pilot-conditional-confidence}.

There is a constant \(C_{\rm cov}(\bar\kappa)\ge1\), depending only on
\(\bar\kappa\), such that, for every \(t\ge1\), the covariance-concentration rate
\begin{equation}
\label{eq:main-pilot-stability-rate}
\vartheta_{N,\lambda}(t)
 :=C_{\rm cov}(\bar\kappa)\left[
 \sqrt{\frac{d_{X,\lambda}+t}{N-1}}
 +\frac{d_{X,\lambda}+t}{N-1}
 \right]
\end{equation}
satisfies \eqref{eq:main-regularized-covariance-concentration}.  The constant
\(C_{\rm cov}(\bar\kappa)\) is existential in the
general sub-Gaussian case.  Define the self-adjoint regularized covariance
error operator
\[
 E_A:=M_\lambda^{-1/2}
 (\widehat\Sigma_{X,A}-\Sigma_X)M_\lambda^{-1/2}.
\]
For every \(t\ge1\),
\begin{equation}
\label{eq:main-regularized-covariance-concentration}
 \Pr\left\{\|E_A\|_{\rm op}>
 \vartheta_{N,\lambda}(t)\right\}\le e^{-t}.
\end{equation}
\SuppResult{Lemma}{lem:regularized-covariance-concentration} proves
\eqref{eq:main-regularized-covariance-concentration}.
If \(\vartheta_{N,\lambda}(t)\le1/2\), then
\[
 \{\|E_A\|_{\rm op}>1/2\}
 \subseteq
 \{\|E_A\|_{\rm op}>\vartheta_{N,\lambda}(t)\}.
\]
Hence, with probability at least \(1-e^{-t}\),
\[
 \|E_A\|_{\rm op}\le\frac12.
\]
On \(\{\|E_A\|_{\rm op}\le1/2\}\),
\[
 M_A=M_\lambda^{1/2}(I+E_A)M_\lambda^{1/2},
\]
so
\(\tfrac12M_\lambda\preceq M_A\preceq\tfrac32M_\lambda\) and
\(M_A^{-1}\preceq2M_\lambda^{-1}\).  Consequently,
\[
 r_A
 =\operatorname{tr}(\Sigma_X^{1/2}M_A^{-1}\Sigma_X^{1/2})
 \le2d_{X,\lambda}.
\]
Also \(\Sigma_X\preceq M_\lambda\preceq2M_A\), which gives
\(\rho_A\le2\), and then
\(s_A\le\rho_Ar_A\le4d_{X,\lambda}\).  We have proved that, outside an
event of probability at most \(e^{-t}\),
\begin{equation}
\label{eq:main-pilot-stability-consequences}
 \tfrac12M_\lambda\preceq M_A\preceq\tfrac32M_\lambda,
 \qquad \rho_A\le2,\qquad r_A\le2d_{X,\lambda},\qquad
 s_A\le4d_{X,\lambda}.
\end{equation}

\Needspace{8\baselineskip}
Equation~\eqref{eq:main-pilot-stability-consequences} gives
\(\Pr\{\rho_A>2\}\le e^{-t}\).  Conditional on \(S_A\) with
\(\rho_A\le2\), Theorem~\ref{thm:balanced-one-split-linear-confidence} with
\(R=2\) gives \eqref{eq:balanced-one-split-quadratic-confidence} with
conditional failure probability at most \(3e^{-t}\); taking expectation over
\(S_A\) and adding \(\Pr\{\rho_A>2\}\) bounds the unconditional failure
probability by \(4e^{-t}\).

\begin{theorem}[Balanced one-split linear confidence field]
\label{thm:balanced-one-split-linear-confidence-field}
Fix \(c_1\ge4\) and \(\alpha\in(0,1)\), and put
\begin{equation}
\label{eq:balanced-finite-split-confidence-level}
 t_\alpha:=\log\frac{c_1}{\alpha}.
\end{equation}
Define
\begin{equation}
\label{eq:balanced-finite-split-inflated-intercept}
 \widehat U_{X,A,\lambda,\alpha}
 :=\widehat T_{X,A,\lambda}
 +\widehat\Delta_{X,A,\lambda}(t_\alpha;2).
\end{equation}
Assume \eqref{eq:main-covariance-matched-subgaussian}, and suppose
\[
 \varepsilon_N(t_\alpha)<1,
 \qquad
 \vartheta_{N,\lambda}(t_\alpha)\le\frac12.
\]
For any fixed balanced split \(A\), with probability at least \(1-\alpha\),
\begin{equation}
\label{eq:balanced-finite-split-confidence-ellipsoid}
 \|PX-P_nX\|_{M_A^{-1}}^2
 \le \widehat U_{X,A,\lambda,\alpha}.
\end{equation}
Whenever \eqref{eq:balanced-finite-split-confidence-ellipsoid} holds,
Cauchy--Schwarz gives, simultaneously for every \(h\in\mathcal H\),
\begin{equation}
\label{eq:balanced-finite-split-linear-confidence}
 |(P-P_n)f_h|
 \le\{\langle h,M_Ah\rangle
 \widehat U_{X,A,\lambda,\alpha}\}^{1/2}.
\end{equation}
If \(A\) is instead randomized independently of the observations,
\eqref{eq:balanced-finite-split-confidence-ellipsoid}--
\eqref{eq:balanced-finite-split-linear-confidence} both hold with probability
at least \(1-\alpha\), where probability is taken over both \(A\) and the
observations.
\end{theorem}

For a prespecified sequence \(\boldsymbol A\) of balanced splits,
\SuppResult{Theorem}{supp:thm:balanced-finite-split-linear-confidence}
averages the split metrics and squared radii and uses
\(t_\alpha=\log\{c_1K_\circ(\boldsymbol A)/\alpha\}\).

Theorem~\ref{thm:balanced-one-split-linear-confidence-field} gives a
finite-sample guarantee for a prespecified sample size \(n=2N\), under
\eqref{eq:main-covariance-matched-subgaussian},
\(\varepsilon_N(t_\alpha)<1\), and
\(\vartheta_{N,\lambda}(t_\alpha)\le1/2\).
The center \(P_nX\) uses all observations; \(M_A\) determines the ellipsoid's
shape, while \(\widehat T_{X,A,\lambda}\) and
\(\widehat r_{B_A\mid A}\) determine the ellipsoid's radius.
The classical finite-dimensional Gaussian mean ellipsoid is centered at the
sample mean
\cite{hotelling1931generalization}.
Regularized finite-dimensional mean regions also regularize the covariance.
The symmetric-case mean interpretation in Issouani et al. uses finite
covariance, with no additional higher-moment condition, while the
general-distribution result in Issouani et al. requires eighth moments
together with dimension-ratio and covariance-spectrum
conditions
\cite[Theorem~2; Theorem~3]{issouani2024regularizedhotelling}.
Theorem~\ref{thm:balanced-one-split-linear-confidence-field} instead allows a
separable Hilbert space under covariance-matched sub-Gaussian tails.

Theorem~\ref{thm:balanced-one-split-linear-confidence-field} gives an
ellipsoid in the empirical pilot metric.
Corollary~\ref{cor:balanced-one-split-oracle-width} compares \(M_A\) with
\(M_\lambda\) and converts the ellipsoid's random radius to the population
effective-dimension scale.
\begin{corollary}[Finite-sample population-geometry width]
\label{cor:balanced-one-split-oracle-width}
Fix \(\alpha\in(0,1)\) and \(c_1\ge6\), define \(t_\alpha\) by
\eqref{eq:balanced-finite-split-confidence-level}, and assume
\eqref{eq:main-covariance-matched-subgaussian}.  Suppose
\[
 \varepsilon_N(t_\alpha)\le\frac12,
 \qquad \vartheta_{N,\lambda}(t_\alpha)\le\frac12.
\]
There is a constant \(C_{\rm or}(\bar\kappa)\), depending only on
\(\bar\kappa\), such that, for any fixed split \(A\), with probability at
least \(1-\alpha\), the metric bound and intercept bound hold simultaneously:
\begin{align}
 M_A
 &\preceq\tfrac32M_\lambda,
 \label{eq:balanced-finite-split-oracle-metric}\\
 \widehat U_{X,A,\lambda,\alpha}
 &\le C_{\rm or}(\bar\kappa)
 \frac{d_{X,\lambda}+t_\alpha}{n}.
 \label{eq:balanced-finite-split-oracle-intercept}
\end{align}
In addition, the confidence ellipsoid
\eqref{eq:balanced-finite-split-confidence-ellipsoid} holds; and,
simultaneously for every \(h\in\mathcal H\),
\begin{equation}
\label{eq:balanced-finite-split-oracle-width}
 |(P-P_n)f_h|
 \le C_{\rm or}(\bar\kappa)
 \left\{\frac{\langle h,M_\lambda h\rangle
 (d_{X,\lambda}+t_\alpha)}n\right\}^{1/2}.
\end{equation}
If \(A\) is randomized independently of the observations,
\eqref{eq:balanced-finite-split-oracle-metric}--
\eqref{eq:balanced-finite-split-oracle-width} all hold with probability at
least \(1-\alpha\), where probability is taken over both \(A\) and the
observations.
\end{corollary}

\begin{proof}
Apply
\SuppResult{Corollary}{supp:cor:balanced-finite-split-oracle-width}
to the one-element split sequence \(\boldsymbol A=(A)\), for which
\(K_{\rm sp}=K_\circ(\boldsymbol A)=1\).
\end{proof}

For an affine loss
\(\ell(\theta,z)=c(z)+\langle X(z),\theta\rangle\), the linear confidence
field applies directly with \(h=v\).

\par\medskip
\noindent\textit{Selecting and certifying an affine update.}
When \(\mathcal H=\mathbb R^d\), define the affine upper field
\begin{equation}
\label{eq:affine-certification-field}
 \widehat{\mathsf U}_{\rm aff}(v)
 :=\overline X^\top v
 +\{v^\top M_Av\,\widehat U_{X,A,\lambda,\alpha}\}^{1/2}.
\end{equation}
On the event in
\eqref{eq:balanced-finite-split-confidence-ellipsoid}, any update
\(\widehat v\in\mathcal D\) chosen measurably using the training sample satisfies
\(P\delta_{\widehat v}\le
\widehat{\mathsf U}_{\rm aff}(\widehat v)\); a negative upper endpoint
therefore certifies \(P\delta_{\widehat v}<0\).
If \(\mathcal D=\{v\in\mathbb R^d:\|v\|\le1\}\), then
\(\inf_{v\in\mathcal D}\widehat{\mathsf U}_{\rm aff}(v)<0\) if and only if
\begin{equation}
\label{eq:affine-certification-criterion}
 \overline X^\top M_A^{-1}\overline X
 >\widehat U_{X,A,\lambda,\alpha}.
\end{equation}
For every \(v\in\mathbb R^d\), weighted Cauchy--Schwarz gives
\[
 \overline X^\top v
 \ge-\{v^\top M_Av\}^{1/2}
 \{\overline X^\top M_A^{-1}\overline X\}^{1/2},
\]
which proves necessity in \eqref{eq:affine-certification-criterion}.
When \eqref{eq:affine-certification-criterion} holds, choose
\(v=-cM_A^{-1}\overline X\), with \(c>0\) small enough that \(\|v\|\le1\).
Equation~\eqref{eq:affine-certification-field} then gives
\[
 \widehat{\mathsf U}_{\rm aff}(v)
 =c\left[-\overline X^\top M_A^{-1}\overline X
 +\{\overline X^\top M_A^{-1}\overline X\,
 \widehat U_{X,A,\lambda,\alpha}\}^{1/2}\right]<0,
\]
which proves sufficiency.
 \section{Nonlinear Extension for Smooth and Piecewise-Smooth Losses}
\label{sec:local}
The linear field in Section~\ref{sec:finite-sample-certificates} applies to
the vector of Taylor coefficients.  For a piecewise-smooth loss,
\(\ell(\cdot,z)\) is locally a pointwise selection from finitely many smooth
functions, which we call \emph{branches}.  We expand the branch selected near
the base point while holding the branch index fixed, and control separately
the frozen branch's Taylor remainder and the discrepancy created when the update path meets a
branch interface.  Smooth losses form the one-branch case.

Let \(\mathcal Z\) be a measurable observation space and let
\(Z_1,\ldots,Z_n\) be i.i.d.\ \(\mathcal Z\)-valued observations with law
\(P\).  Equip \(\Theta\subset\mathbb R^d\) with the Borel \(\sigma\)-field
induced by the Euclidean subspace topology, and let
\(\ell:\Theta\times\mathcal Z\to\mathbb R\) be measurable.  Fix a deterministic
base point \(\theta_0\in\Theta\) and a nonempty deterministic Borel candidate
set \(\mathcal D\subset\mathbb R^d\) whose update segments
\(\{\theta_0+tv:0\le t\le1\}\) lie in \(\Theta\).  For \(\rho>0\), write
\(\mathcal D_\rho:=\{v\in\mathcal D:\|v\|\le\rho\}\).
Equation~\eqref{eq:risk-increment} defines the risk increment
\(\delta_v(z)=\ell(\theta_0+v,z)-\ell(\theta_0,z)\).  Unsubscripted norms
are ambient, while \(\|\cdot\|_{\rm op}\) and \(\|\cdot\|_{\rm HS}\) denote
operator and Hilbert--Schmidt norms.

The union of the candidate update segments from \(\theta_0\) is
\[
 K_{\mathcal D}
 :=\{\theta_0+t v:v\in\mathcal D,\ 0\le t\le1\}.
\]
\par\medskip
\noindent\textbf{Standing condition (PS).}
Fix integers \(m\ge1\) and \(N_{\rm br}\ge1\).  Let \(U_{\mathcal D}\) be an
open subset of \(\mathbb R^d\) satisfying
\(K_{\mathcal D}\subset U_{\mathcal D}\subset\Theta\).  Assume there are
jointly measurable branch extensions
\[
 \bar\ell_r:U_{\mathcal D}\times\mathcal Z\to\mathbb R,
 \qquad 1\le r\le N_{\rm br},
\]
and a measurable set \(\mathcal Z_{\rm PS}\) with
\(P(\mathcal Z_{\rm PS})=1\) such that, for every
\(z\in\mathcal Z_{\rm PS}\), the sections
\(\bar\ell_{r,z}:=\bar\ell_r(\,\cdot\,,z)\) are \(C^{m+1}\) and
\(\ell(\cdot,z)\) is a continuous selection from the branch extensions on
\(U_{\mathcal D}\).  For \(z\notin\mathcal Z_{\rm PS}\), set every branch
extension \(\bar\ell_r(\cdot,z)\) equal to the zero function.  The zero
extension preserves joint measurability and makes every branch section
\(C^{m+1}\); integrals under measures supported on
\(\mathcal Z_{\rm PS}\) retain the original branch values.  Let
\(\mathcal T\subset\Theta\times\mathcal Z\) be a measurable set containing the
branch switches, write
\(\mathcal T_z:=\{\theta:(\theta,z)\in\mathcal T\}\), and, when
\(\theta_0\notin\mathcal T_z\), let \(\mathcal C_{0,z}\) be the connected
component of \(U_{\mathcal D}\setminus\mathcal T_z\) containing \(\theta_0\).
Assume that there is a measurable index map
\(r_0:\mathcal Z\to\{1,\ldots,N_{\rm br}\}\) such that, for every
\(z\in\mathcal Z_{\rm PS}\) with \(\theta_0\notin\mathcal T_z\),
\[
 \ell(\theta,z)=\bar\ell_{r_0(z),z}(\theta),
 \qquad \text{for every }\theta\in\mathcal C_{0,z}.
\]
The index map has the default value \(r_0(z)=1\) when
\(z\notin\mathcal Z_{\rm PS}\) or \(\theta_0\in\mathcal T_z\).
Continuing the branch indexed by \(r_0(z)\) defines the \emph{frozen branch}
\(g_0(\theta,z)=\bar\ell_{r_0(z),z}(\theta)\); write
\(g_{0,z}:=g_0(\cdot,z)\).
Since \(r_0\) has finite range,
\(g_0(\theta,z)=\sum_{r=1}^{N_{\rm br}}\mathbf1\{r_0(z)=r\}
\bar\ell_r(\theta,z)\) is jointly measurable.

The smooth-branch requirement in Condition (PS) is local to
\(U_{\mathcal D}\).  For each \(z\in\mathcal Z_{\rm PS}\),
\(\ell(\cdot,z)\) on \(U_{\mathcal D}\) belongs to the finite-family version of
the standard local piecewise-\(C^{m+1}\) class, which is closed under finite
sums and products, composition, and pointwise finite maxima and minima
\cite[Section~4.1]{scholtes2012introduction}.  A loss obtained from
\(C^{m+1}\) maps by finitely many of the listed operations therefore has a finite
\(C^{m+1}\) branch family on \(U_{\mathcal D}\), provided that every branch
composition generated by the listed operations is defined on \(U_{\mathcal D}\).

Smooth finite-dimensional losses have one branch.  As functions of the scalar
residual, absolute and quantile-check losses have two affine branches
\cite[Section~1.3]{koenker2005quantile}.  As functions of the finite score
vectors used by the losses, finite-label multiclass hinge losses and finite-output structured
hinge losses are finite maxima of affine branches
\cite[Section~3]{crammer2001algorithmic}
\cite[Section~2.2]{tsochantaridis2005large}, while the observationwise
\(K\)-means loss with fixed \(K\), viewed as a function of the \(K\) centers,
is a finite minimum of quadratic branches \cite{pollard1981strong}.  Bolte and
Pauwels give a finite-branch representation for fixed finite programs built
from a class containing logistic loss, Gaussian likelihood, ReLU, maximum,
and the \(k\)th-largest operation
\cite[Definitions~1--3 and Propositions~1 and~5]{bolte2020mathematical}.  For a
fixed finite neural network whose activations and outer loss are \(C^{m+1}\) or
continuous selections from finitely many \(C^{m+1}\) maps, the branch index is
the finite tuple of activation and outer-loss branch choices.  With
activations having finitely many polynomial pieces, Bartlett et al. construct,
at each fixed input, a finite parameter-space refinement on which the network
output is polynomial \cite[Section~4]{bartlett2019nearly}.

\subsection{Frozen-branch decomposition}
\label{subsec:frozen-branch-decomposition}
Each local loss increment splits into the Taylor approximation of the frozen
branch, masked when the base point lies on the interface, the frozen branch's Taylor
remainder, and the discrepancy between the actual loss and the frozen branch.

For \(v\in\mathcal D\), define
\begin{align*}
b_v(z)&:=\mathbf 1\left\{\exists s\in[0,1]:
(\theta_0+s v,z)\in\mathcal T\right\},\\
b_0(z)&:=\mathbf 1\{(\theta_0,z)\in\mathcal T\},&
\chi_0(z)&:=1-b_0(z).
\end{align*}
Assume \((v,z)\mapsto b_v(z)\) is jointly measurable.  A simple sufficient
condition is that \(\mathcal Z\) is standard Borel and, under some compatible
Polish topology, \(\mathcal T\cap(U_{\mathcal D}\times\mathcal Z)\) is closed.
The inverse image of
\(\mathcal T\cap(U_{\mathcal D}\times\mathcal Z)\) under
\((v,z,s)\mapsto(\theta_0+s v,z)\) is closed, and the projection of the
inverse image over the
compact coordinate \([0,1]\) is Borel in
\(\mathcal D\times\mathcal Z\).  The map \((v,z)\mapsto b_v(z)\) is also
jointly measurable when
\(\mathcal T\cap(U_{\mathcal D}\times\mathcal Z)\) is open, since coordinate
projections are open.
Since the path includes \(s=0\), \(b_0\le b_v\).  The default value of
\(r_0\) on the base interface is
masked by \(\chi_0\), while \(P(\mathcal Z_{\rm PS})=1\) makes \(P\)-integrals
depend only on the values of \(r_0\) on \(\mathcal Z_{\rm PS}\).

For a multi-index \(\alpha=(\alpha_1,\ldots,\alpha_d)\in\mathbb N_0^d\), write
\[
 |\alpha|:=\sum_{j=1}^d\alpha_j,\qquad
 \alpha!:=\prod_{j=1}^d\alpha_j!,\qquad
 v^\alpha:=\prod_{j=1}^dv_j^{\alpha_j}.
\]
For each integer \(1\le k\le m\) and each multi-index with \(|\alpha|=k\),
define \(\binom{k}{\alpha}:=k!/\alpha!\) and set
\begin{align*}
 a_{k,\alpha}(v)&:=\sqrt{\binom{k}{\alpha}}\,v^\alpha,&
 \xi_{k,\alpha}(\theta_0,z)
 &:=\frac{\partial_\theta^\alpha g_{0,z}(\theta_0)}{\sqrt{k!\alpha!}}.
\end{align*}
Collect the coordinates indexed by \(\{\alpha\in\mathbb N_0^d:|\alpha|=k\}\)
as vectors \(a_k(v),\xi_k(\theta_0,z)\in
\mathbb R^{\binom{d+k-1}{k}}\).  Then
\begin{align*}
 \|a_k(v)\|^2=\|v\|^{2k},\qquad
 \frac1{k!}D_\theta^kg_{0,z}(\theta_0)[v^k]
 =\langle\xi_k(\theta_0,z),a_k(v)\rangle.
\end{align*}
Indeed, the multinomial formula gives
\[
 \|a_k(v)\|^2
 =\sum_{|\alpha|=k}\binom{k}{\alpha}v^{2\alpha}
 =\left(\sum_{j=1}^dv_j^2\right)^k,
\]
and the coordinate definitions give
\[
 \langle\xi_k(\theta_0,z),a_k(v)\rangle
 =\sum_{|\alpha|=k}
 \frac{\partial_\theta^\alpha g_{0,z}(\theta_0)}{\alpha!}v^\alpha
 =\frac1{k!}D_\theta^kg_{0,z}(\theta_0)[v^k].
\]
Concatenate the blocks into two vectors of Taylor coefficients and normalized
monomials:
\[
 \Xi_0(\theta_0,z)
 :=\bigl(\xi_1(\theta_0,z),\ldots,\xi_m(\theta_0,z)\bigr),
 \qquad
 a(v):=\bigl(a_1(v),\ldots,a_m(v)\bigr).
\]
The vectors \(\Xi_0(\theta_0,z)\) and \(a(v)\) lie in
\(\mathbb R^{\binom{d+m}{m}-1}\), and
\(\|a(v)\|^2=\sum_{k=1}^m\|v\|^{2k}\).
Define the frozen-branch increment by
\[
G_v(z):=g_{0,z}(\theta_0+v)-g_{0,z}(\theta_0)
\]
and define the order-\(m\) Taylor approximation by
\[
J_v(z):=\sum_{k=1}^m \frac1{k!}
D_\theta^k g_{0,z}(\theta_0)[v^k]
=\langle \Xi_0(\theta_0,z),a(v)\rangle,
\]
Define the base-masked Taylor term by
\[
J_v^0(z):=\chi_0(z)J_v(z).
\]
The base-masked Taylor feature vector is
\(X_0(z):=\chi_0(z)\Xi_0(\theta_0,z)\in
\mathbb R^{\binom{d+m}{m}-1}\), so that
\(J_v^0(z)=\langle X_0(z),a(v)\rangle\).

Define the frozen-branch remainder and interface-crossing term by
\[
R_v^{\rm fr}(z)
:=\chi_0(z)\{G_v(z)-J_v(z)\},
\qquad
C_v(z):=\delta_v(z)-\chi_0(z)G_v(z).
\]
Equivalently,
\[
C_v=b_0\delta_v+\chi_0(\delta_v-G_v).
\]
The term \(C_v\) is the discrepancy between the actual increment and the
masked frozen-branch increment.  We call \(C_v\) the \emph{interface-crossing
term} because, on the full-probability piecewise-smooth set, \(C_v\) vanishes for
every update segment that avoids \(\mathcal T\).
For \(z\in\mathcal Z_{\rm PS}\) and \(v\in\mathcal D\), \(b_v(z)=0\)
implies that the segment from \(\theta_0\) to \(\theta_0+v\) is connected,
contains \(\theta_0\), and lies in
\(U_{\mathcal D}\setminus\mathcal T_z\).  The update segment is therefore contained in
\(\mathcal C_{0,z}\), so \(\delta_v(z)=G_v(z)\) and
\[
 C_v(z)=b_v(z)C_v(z),
 \qquad (v,z)\in\mathcal D\times\mathcal Z_{\rm PS}.
\]
Substituting the definitions gives, pointwise,
\begin{align}
 J_v^0+R_v^{\rm fr}+C_v
 &=\chi_0J_v+\chi_0(G_v-J_v)+\delta_v-\chi_0G_v\notag\\
 &=\delta_v.
\label{eq:frozen-branch-decomposition}
\end{align}
Equation~\eqref{eq:frozen-branch-decomposition} holds for every
\(v\in\mathcal D\).

Define the frozen-branch Taylor-remainder envelope
\begin{equation}
\label{eq:frozen-branch-taylor-envelope}
h_v^{\rm fr}(z)
:=\frac{\chi_0(z)}{(m+1)!}
\sup_{0\le s\le1}
\left|D_\theta^{m+1}g_{0,z}(\theta_0+s v)[v^{m+1}]\right|.
\end{equation}
For \(z\in\mathcal Z_{\rm PS}\) and \(v\in\mathcal D\), Taylor's theorem
applied to the fixed extension \(g_{0,z}\) gives
\begin{equation}
\label{eq:frozen-branch-taylor-remainder}
R_v^{\rm fr}(z)
=\frac{\chi_0(z)}{m!}\int_0^1(1-s)^m
D_\theta^{m+1}g_{0,z}(\theta_0+s v)[v^{m+1}]\,ds.
\end{equation}
Since \(\int_0^1(1-s)^m\,ds=1/(m+1)\), taking absolute values in
\eqref{eq:frozen-branch-taylor-remainder} gives
\[
 |R_v^{\rm fr}(z)|
 \le\frac{\chi_0(z)}{(m+1)!}
 \sup_{0\le s\le1}
 \left|D_\theta^{m+1}g_{0,z}(\theta_0+s v)[v^{m+1}]\right|
 =h_v^{\rm fr}(z).
\]
Assume \((v,z)\mapsto(J_v^0,R_v^{\rm fr},C_v,h_v^{\rm fr})\) is jointly
measurable.  Whenever an indexed supremum is not measurable, interpret the
expectation of the indexed supremum as an outer expectation.

The pointwise decomposition
\eqref{eq:frozen-branch-decomposition} and remainder bound reduce population
control to the Taylor, crossing, and envelope expectations:
\begin{proposition}[Frozen-branch population reduction]
\label{prop:fixed-mask-crossing-reduction}
For every \(v\in\mathcal D\) and every probability measure \(T\) on
\(\mathcal Z\) satisfying \(T(\mathcal Z_{\rm PS})=1\) and for which the
terms are finite,
\[
TJ_v^0-T h_v^{\rm fr}+TC_v
\le
T\delta_v
\le
TJ_v^0+T h_v^{\rm fr}+TC_v.
\]
\end{proposition}
\begin{proof}
Because \(T(\mathcal Z_{\rm PS})=1\), integrating
\(-h_v^{\rm fr}\le R_v^{\rm fr}\le h_v^{\rm fr}\) and
Equation~\eqref{eq:frozen-branch-decomposition} gives the lower and upper
inequalities in Proposition~\ref{prop:fixed-mask-crossing-reduction}.
\end{proof}

For an integer \(q\ge1\), a measurable bounded operator
\(Y(z):\mathcal H\to\mathbb R^q\), a nonempty
\(J\subset\{1,\ldots,n\}\), and \(\lambda>0\), define the empirical Gram
operator \(\widehat V_{Y,J}\) and ridge-regularized operator
\(M_{Y,J,\lambda}\) by
\(\widehat V_{Y,J}:=|J|^{-1}\sum_{i\in J}Y(Z_i)^*Y(Z_i)\) and
\(M_{Y,J,\lambda}:=\widehat V_{Y,J}+\lambda I\).  Put
\(\kappa_q=1\) for \(q=1\) and \(\kappa_q=\sqrt2\) for \(q\ge2\).
For a partition \(A,B=A^c\) of \(\{1,\ldots,n\}\) with \(|A|=|B|\), define
the average regularized Gram operator
\(M_{Y,A\leftrightarrow B,\lambda}\) and the sum
\(r_{Y,A\leftrightarrow B,\lambda}\) of the two orientation-specific
empirical traces by
\begin{align*}
 M_{Y,A\leftrightarrow B,\lambda}
 &:=\tfrac12\{M_{Y,A,\lambda}+M_{Y,B,\lambda}\},\\
 r_{Y,A\leftrightarrow B,\lambda}
 &:=\operatorname{tr}\{\widehat V_{Y,B}M_{Y,A,\lambda}^{-1}\}
 +\operatorname{tr}\{\widehat V_{Y,A}M_{Y,B,\lambda}^{-1}\}.
\end{align*}
The two trace summands defining \(r_{Y,A\leftrightarrow B,\lambda}\) are
nonnegative, so
\(r_{Y,A\leftrightarrow B,\lambda}\ge0\).

\subsection{Crossing-slope operators}
\label{subsec:crossing-slope-operators}

For the crossing term
\(C_v=\delta_v-\chi_0G_v\) from
Subsection~\ref{subsec:frozen-branch-decomposition}, fix \(\rho>0\) with
\(0\in\mathcal D_\rho\) and an integer \(q_C\ge1\).  Call a measurable operator
\(Y_C(z):\mathbb R^d\to\mathbb R^{q_C}\) a \emph{crossing-slope operator} if,
almost surely and simultaneously for \(v,w\in\mathcal D_\rho\),
\begin{equation}
\label{eq:pullback-lipschitz-increments}
 |C_v(z)-C_w(z)|\le\|Y_C(z)(v-w)\|.
\end{equation}
The crossing-slope condition~\eqref{eq:pullback-lipschitz-increments} contains
ordinary local Lipschitz continuity.  If a measurable
\(L:\mathcal Z\to[0,\infty)\) satisfies
\(|C_v(z)-C_w(z)|\le L(z)\|v-w\|\) almost surely and simultaneously on
\(\mathcal D_\rho\), then \(Y_C(z)=L(z)I_d\), with \(q_C=d\), is a
crossing-slope operator.

\begin{proposition}[A finite-selection sufficient condition]
\label{prop:finite-selection-crossing-slope}
Assume Condition~(PS), and suppose that a compact convex set
\(K_\rho\subset U_{\mathcal D}\) contains
\(\theta_0+\mathcal D_\rho\).  Define
\[
 \mathsf B_\rho(z)
 :=\max_{1\le r\le N_{\rm br}}\sup_{\theta\in K_\rho}
 \|\nabla_\theta\bar\ell_{r,z}(\theta)\|,
 \qquad
 \mathsf D_\rho(z)
 :=\max_{1\le r\le N_{\rm br}}\sup_{\theta\in K_\rho}
 \|\nabla_\theta\bar\ell_{r,z}(\theta)
       -\nabla_\theta g_{0,z}(\theta)\|.
\]
The two envelopes are measurable and finite.  The scalar multiplier and
operator
\begin{equation}
\label{eq:finite-selection-crossing-envelope}
 L_{C,\rho}^{\rm sel}
 :=b_0\mathsf B_\rho+\chi_0\mathsf D_\rho,
 \qquad
 Y_C:=L_{C,\rho}^{\rm sel}I_d
\end{equation}
satisfy \eqref{eq:pullback-lipschitz-increments} with \(q_C=d\).  Moreover,
\(L_{C,\rho}^{\rm sel}\le2\mathsf B_\rho\).  If a known constant
\(\overline B_\rho<\infty\) satisfies
\(\mathsf B_\rho\le\overline B_\rho\) almost surely, then
\(\|Y_C\|_{\rm HS}\le2\sqrt d\,\overline B_\rho\) almost surely.
\end{proposition}

\begin{proof}
See \SuppSection{app:worked-examples}.
\end{proof}

Proposition~\ref{prop:finite-selection-crossing-slope} separates pointwise
existence of a crossing-slope operator from the deterministic bound used in
Theorem~\ref{thm:componentwise-hybrid-local-confidence}.  Compact containment
makes \(\mathsf B_\rho(z)\) finite for each observation.  Using the generic
operator in \eqref{eq:finite-selection-crossing-envelope} in
Theorem~\ref{thm:componentwise-hybrid-local-confidence}
additionally requires a known almost-sure bound on \(\mathsf B_\rho\).
Structured operators can avoid the \(\sqrt d\) Hilbert--Schmidt factor in the
generic construction.

\medskip
\noindent\textit{Check and hinge losses.}
For the check loss
\(\ell(\theta;(x,y))=\rho_\tau(y-x^\top\theta)\), where
\(\rho_\tau(u)=u\{\tau-\mathbf1(u<0)\}\) and \(0<\tau<1\)
\cite[Section~1.3]{koenker2005quantile}, write
\(e=y-x^\top\theta_0\) and \(t_v=x^\top v\).  The affine frozen branch of the
check loss has
zero Taylor remainder, and direct evaluation gives
\begin{equation}
\label{eq:check-crossing-term}
 C_v=
 \begin{cases}
  (t_v-e)_+,&e>0,\\
  (e-t_v)_+,&e<0,\\
  \rho_\tau(-t_v),&e=0.
 \end{cases}
\end{equation}
On \(\mathcal D_\rho\), let
\begin{equation}
\label{eq:check-crossing-multiplier}
 A_{\tau,\rho}
 :=\mathbf1\{0<|e|<\rho\|x\|\}
  +\max(\tau,1-\tau)\mathbf1\{e=0\}.
\end{equation}
The case formula in \eqref{eq:check-crossing-term} and the multiplier in
\eqref{eq:check-crossing-multiplier} give
\(|C_v-C_w|\le A_{\tau,\rho}|x^\top(v-w)|\).  Hence the rank-at-most-one operator
\(Y_Ca=A_{\tau,\rho}x^\top a\), with \(q_C=1\), is valid; if
\(\|x\|\le K\) almost surely, one may take \(K_C=K\).

For the binary hinge loss
\(\ell(\theta;(x,y))=(1-yx^\top\theta)_+\), with
\(y\in\{-1,1\}\), set \(e=1-yx^\top\theta_0\) and
\(t_v=yx^\top v\).  The three cases for \(C_v\) are
\((t_v-e)_+\) when \(e>0\), \((e-t_v)_+\) when \(e<0\), and
\((-t_v)_+\) when \(e=0\).  The scalar operator
\begin{equation}
\label{eq:hinge-crossing-operator}
 Y_Ca=\{\mathbf1(e=0)+\mathbf1(0<|e|<\rho\|x\|)\}\,yx^\top a
\end{equation}
satisfies \eqref{eq:pullback-lipschitz-increments}.  Under
\(\|x\|\le K\) almost surely,
the operator \(Y_C\) satisfies \(\|Y_C\|_{\rm HS}\le K\).

\medskip
\noindent\textit{Finite affine maxima and minima.}
Let \(\gamma_r:\mathcal Z\to\mathbb R\) and
\(\beta_r:\mathcal Z\to\mathbb R^d\) be measurable, and let
\(\ell(\theta,z)=\max_{1\le r\le R}
\{\gamma_r(z)+\beta_r(z)^\top\theta\}\).  Form the \(R\)-row matrix
\(W_C(z)\) whose \(r\)th row is
\begin{equation}
\label{eq:affine-maximum-crossing-operator}
 b_0(z)\beta_r(z)^\top
 +\chi_0(z)\{\beta_r(z)-\beta_{r_0(z)}(z)\}^\top.
\end{equation}
The measurability of the coefficients, \(b_0\), and \(r_0\) makes \(W_C\)
measurable.  The inequality
\(|\max_r u_r-\max_r \widetilde u_r|
 \le\max_r|u_r-\widetilde u_r|\)
gives
\(|C_v-C_w|\le\|W_C(v-w)\|_2\).  Thus \(Y_C=W_C\) is valid with
\(q_C=R\).
The construction \(Y_C=W_C\) applies to finite-label multiclass hinge losses and
finite-output structured hinge losses when the score vectors of the losses are affine in
\(\theta\), as in
\cite[Section~3]{crammer2001algorithmic} and
\cite[Section~2.2]{tsochantaridis2005large}.

\medskip
\noindent\textit{Least absolute deviations (LAD).}
For LAD in linear prediction,
\(\ell(\theta;(x,y))=|y-x^\top\theta|\), write
\(e:=y-x^\top\theta_0\) and \(t_v:=x^\top v\).  Take the residual-zero set
\(\mathcal T=\{(\theta,(x,y)):y-x^\top\theta=0\}\) as the interface set,
where the absolute-value loss changes affine branch.  Use
\(\operatorname{sgn}(0)=0\).  The frozen branch is affine, so
\(J_v^0=-\operatorname{sgn}(e)t_v\) and
\(R_v^{\rm fr}=0\), whereas the crossing term is
\begin{equation}
\label{eq:lad-crossing-term}
 C_v
 =|e-t_v|-|e|+\operatorname{sgn}(e)t_v
 =
 \begin{cases}
  2\{\operatorname{sgn}(e)t_v-|e|\}_+, & e\ne0,\\
  |t_v|, & e=0.
 \end{cases}
\end{equation}
The two cases in \eqref{eq:lad-crossing-term} follow by setting
\(u=\operatorname{sgn}(e)t_v\) when \(e\ne0\) and evaluating directly when
\(e=0\); hence the nonsmooth contribution is entirely in \(C_v\).  On
\(\mathcal D_\rho\), define the nonnegative scalar multiplier
\begin{equation}
\label{eq:lad-crossing-multiplier}
 L_\rho(x,y):=\mathbf 1\{e=0\}
 +2\mathbf 1\{0<|e|<\rho\|x\|\}.
\end{equation}
For \(e=0\), the reverse triangle inequality bounds \(|C_v-C_w|\) by
\(|x^\top(v-w)|\).  For \(0<|e|<\rho\|x\|\), applying the one-Lipschitz
map \(r\mapsto r_+\) to \(\operatorname{sgn}(e)t_v-|e|\) and
\(\operatorname{sgn}(e)t_w-|e|\) gives
\[
 |C_v-C_w|
 \le 2|\operatorname{sgn}(e)(t_v-t_w)|
 =2|x^\top(v-w)|.
\]
For \(|e|\ge\rho\|x\|\), \(C_v=C_w=0\).  Therefore
\(|C_v-C_w|\le L_\rho|x^\top(v-w)|\) on \(\mathcal D_\rho\).
The operator \(Y_\rho(x,y):\mathbb R^d\to\mathbb R\), defined by
\(Y_\rho(x,y)a:=L_\rho(x,y)x^\top a\), satisfies
\eqref{eq:pullback-lipschitz-increments} with \(q_C=1\).

\subsection{Componentwise assembled coverage}
\label{subsec:high-probability-endpoint-lifts}
Let \(\Lambda_\rho:\mathcal Z\to[0,\infty]\) be an
observation-level derivative envelope, uniform over update directions in
\(\mathcal D_\rho\), satisfying
\begin{equation}
\label{eq:radial-remainder-smoothness-envelope}
 \Lambda_\rho(z)\ge
 \frac{\chi_0(z)}{(m+1)!}
 \sup_{\substack{u\in\mathcal D_\rho\\0\le t\le1}}
 \|D_\theta^{m+1}g_{0,z}(\theta_0+t u)\|_{\rm op}.
\end{equation}
For \(v\in\mathcal D_\rho\), the definition of the operator norm gives
\begin{align*}
 \left|D_\theta^{m+1}g_{0,z}(\theta_0+t v)[v^{m+1}]\right|
 &\le\left\|D_\theta^{m+1}g_{0,z}(\theta_0+t v)\right\|_{\rm op}
 \|v\|^{m+1}\\
 &\le\rho^{m+1}
 \left\|D_\theta^{m+1}g_{0,z}(\theta_0+t v)\right\|_{\rm op}.
\end{align*}
Thus \(h_v^{\rm fr}\le\rho^{m+1}\Lambda_\rho\).
Fix \(\rho,s,\lambda_J,\eta_C,\lambda_C>0\), and define the restricted
candidate set
\[
 \mathcal D_{\rho,s}
 :=\{v\in\mathcal D:\|v\|\le\rho,\ \|a(v)\|\le s\},
\]
and condition
on one balanced split \(A,B=A^c\), independent of the observations.  Let
\(z\mapsto Y_C(z):\mathbb R^d\to\mathbb R^{q_C}\) be a measurable
matrix-valued crossing-slope operator satisfying
\eqref{eq:pullback-lipschitz-increments}.
For \(v\in\mathcal D_{\rho,s}\), define the nonnegative crossing penalty
\begin{equation}
\label{eq:finite-split-crossing-penalty}
\widehat{\mathsf F}_{C,\rho}(v)
 :=\eta_C v^\top M_{Y_C,A\leftrightarrow B,\lambda_C}v
 +\frac{\kappa_{q_C}^2 r_{Y_C,A\leftrightarrow B,\lambda_C}}
 {\eta_Cn}.
\end{equation}
\par\medskip
\noindent\textit{Assembled confidence endpoints.}
The assembled construction combines the Taylor ellipsoid with separate
confidence events for \(C_v\) and \(R_v^{\rm fr}\).  If \(C_v=0\) for every
\(v\in\mathcal D_{\rho,s}\) on a common full-probability set, the crossing
correction and crossing failure-probability allocation can be omitted.  If
\(R_v^{\rm fr}=0\) for every \(v\in\mathcal D_{\rho,s}\) on a common
full-probability set, the remainder correction and remainder
failure-probability allocation can be omitted.  Let
\(\iota_C,\iota_R\in\{0,1\}\) be deterministic.  Set \(\iota_C=1\) when the
crossing correction is included and \(\iota_C=0\) when the crossing correction
is omitted; set \(\iota_R=1\) when the remainder correction is included and
\(\iota_R=0\) when the remainder correction is omitted.  For a smooth
one-branch loss, \(\mathcal T=\varnothing\) gives the required \(C_v=0\) on
\(\mathcal Z_{\rm PS}\); choosing \(\iota_C=0\) avoids both the \(\alpha_C\)
allocation and the ridge term \(\eta_C\lambda_C\|v\|^2\) that remains in
\eqref{eq:finite-split-crossing-penalty} even when \(Y_C=0\).  An exact order-\(m\)
frozen-branch Taylor expansion gives the required \(R_v^{\rm fr}=0\); for
\(m=1\), an affine frozen branch has \(D_\theta^2g_{0,z}=0\), so
\eqref{eq:radial-remainder-smoothness-envelope} permits \(\Lambda_\rho=0\), while
\(\iota_R=0\) additionally removes the \(\alpha_R\) allocation.  Setting
\(\iota_C=1\) retains the crossing correction, and setting \(\iota_R=1\)
retains the remainder correction, so \((\iota_C,\iota_R)\) covers all four
inclusion patterns.
Fix
\(\alpha,\alpha_J,\alpha_C,\alpha_R\in(0,1)\), put
\(\boldsymbol\alpha=(\alpha_J,\alpha_C,\alpha_R)\), and suppose
\begin{equation}
\label{eq:hybrid-confidence-allocation}
 \alpha_J+\iota_C\alpha_C+\iota_R\alpha_R\le\alpha.
\end{equation}
The superscript
\({\rm hp}\) marks half-widths attached to high-probability events, and
\({\rm hyb}\) marks endpoints obtained by adding the component half-widths.
\Needspace{5\baselineskip}
For the Taylor component, let
\(\widehat U_{J,A,\lambda_J,\alpha_J}\ge0\) be measurable, and define
\begin{equation}
\label{eq:hybrid-taylor-confidence-width}
 \widehat{\mathsf G}_{J,\alpha_J}^{\rm hp}(v)
 :=\left\{
 \langle a(v),M_{X_0,A,\lambda_J}a(v)\rangle
 \widehat U_{J,A,\lambda_J,\alpha_J}
 \right\}^{1/2}.
\end{equation}
When \(\iota_C=1\), fix \(K_C<\infty\) and define
\begin{align}
 \beta_{C,\rho}
 &:=\frac{2\rho K_C+\eta_C\rho^2K_C^2}{n}
 +\frac{2\kappa_{q_C}^2}{\eta_Cn^2}
 \left(\frac{K_C^2}{\lambda_C}+\frac{K_C^4}{\lambda_C^2}\right),
 \label{eq:hybrid-balanced-crossing-modulus}\\
 \widehat\Gamma_{C,\rho,\alpha_C}
 &:=\beta_{C,\rho}
 \left\{\frac n2\log\frac2{\alpha_C}\right\}^{1/2}.
 \label{eq:hybrid-crossing-confidence-correction}
\end{align}
Suppose there is a measurable \(\mathcal Z_0\subseteq\mathcal Z\) with
\(P(\mathcal Z_0)=1\) on which \eqref{eq:pullback-lipschitz-increments} and
\(\|Y_C\|_{\rm HS}\le K_C\) hold.  For every \(i\in[n]\), consider samples
\(S,S'\in\mathcal Z_0^n\) that agree outside coordinate \(i\).  Write
\(P_{n,S}\) and \(\widehat{\mathsf F}_{C,\rho,S}\) for the empirical mean and
crossing penalty computed from \(S\), and use analogous notation for \(S'\).
The replacement calculation in
\SuppResult{Proposition}{prop:balanced-pullback-conditional-concentration}
gives
\begin{equation}
\label{eq:hybrid-crossing-field-replacement-modulus}
 \sup_{v\in\mathcal D_{\rho,s}}
 \left|
 \{P_{n,S}C_v\pm\widehat{\mathsf F}_{C,\rho,S}(v)\}
 -\{P_{n,S'}C_v\pm\widehat{\mathsf F}_{C,\rho,S'}(v)\}
 \right|
 \le\beta_{C,\rho}
 \qquad\text{for either sign}.
\end{equation}
When \(\iota_C=1\), define
\begin{equation}
\label{eq:hybrid-crossing-confidence-width}
 \widehat{\mathsf G}_{C,\rho,\alpha_C}^{\rm hp}(v)
 :=\widehat{\mathsf F}_{C,\rho}(v)
 +\widehat\Gamma_{C,\rho,\alpha_C}.
\end{equation}
When \(\iota_C=0\), set
\(\widehat{\mathsf G}_{C,\rho,\alpha_C}^{\rm hp}\equiv0\).
When \(\iota_R=1\), fix \(K_R<\infty\) and put
\begin{equation}
\label{eq:hybrid-remainder-confidence-width}
 \widehat{\mathsf R}_{\rho,\alpha_R}^{\rm hp}
 :=\rho^{m+1}\left\{P_n\Lambda_\rho
 +K_R\sqrt{\frac{\log(1/\alpha_R)}{2n}}\right\}.
\end{equation}
For independent observations satisfying \(0\le\Lambda_\rho\le K_R\),
Theorem~2 of \cite{hoeffding1963probability}, applied to
\(-\Lambda_\rho(Z_i)\in[-K_R,0]\), gives
\begin{equation}
\label{eq:hybrid-remainder-empirical-mean-event}
 \Pr\left\{P\Lambda_\rho>
 P_n\Lambda_\rho+
 K_R\sqrt{\frac{\log(1/\alpha_R)}{2n}}\right\}
 \le\alpha_R.
\end{equation}
Set \(\widehat{\mathsf R}_{\rho,\alpha_R}^{\rm hp}:=0\) when
\(\iota_R=0\).
The assembled lower and upper endpoints are
\begin{align}
 \widehat{\mathsf U}_{\rho,s,\boldsymbol\alpha}^{\rm hyb}(v)
 &:={}P_n\{J_v^0+C_v\}
 +\widehat{\mathsf G}_{J,\alpha_J}^{\rm hp}(v)
 +\widehat{\mathsf G}_{C,\rho,\alpha_C}^{\rm hp}(v)
 +\widehat{\mathsf R}_{\rho,\alpha_R}^{\rm hp},
 \label{eq:hybrid-upper-endpoint}\\
\widehat{\mathsf L}_{\rho,s,\boldsymbol\alpha}^{\rm hyb}(v)
&:={}P_n\{J_v^0+C_v\}
-\widehat{\mathsf G}_{J,\alpha_J}^{\rm hp}(v)
 -\widehat{\mathsf G}_{C,\rho,\alpha_C}^{\rm hp}(v)
 -\widehat{\mathsf R}_{\rho,\alpha_R}^{\rm hp}.
 \label{eq:hybrid-lower-endpoint}
\end{align}
\begin{theorem}[Assembled local confidence endpoints]
\label{thm:componentwise-hybrid-local-confidence}
Let \(n=2N\) with \(N\ge2\).  Let \(A\) be either a fixed balanced split or a
balanced split randomized independently of the observations, and put
\(B=A^c\).  Fix deterministic \(\iota_C,\iota_R\in\{0,1\}\), where
\(\iota_C=1\) indicates inclusion of the crossing component and \(\iota_R=1\)
indicates inclusion of the remainder component.
Let \(\alpha,\alpha_J,\alpha_C,\alpha_R\in(0,1)\), write
\(\boldsymbol\alpha=(\alpha_J,\alpha_C,\alpha_R)\), and suppose
\eqref{eq:hybrid-confidence-allocation} holds.

For the Taylor component, suppose one of three alternatives holds.
\begin{enumerate}
\item The second moment \(P\|X_0\|^2\) is finite, and
\eqref{eq:main-covariance-matched-subgaussian} holds with \(X=X_0\) for a known
\(\bar\kappa_J\ge1\).  For some
\(c_1\ge4\), let \(t_{\alpha_J}:=\log(c_1/\alpha_J)\).
The quantity \(\varepsilon_N(t_{\alpha_J})\) is defined by
\eqref{eq:conditional-trace-relative-error} with \(\bar\kappa=\bar\kappa_J\);
\(\vartheta_{N,\lambda_J}(t_{\alpha_J})\) is defined by
\eqref{eq:main-pilot-stability-rate} with
\((X,\lambda,\bar\kappa)=(X_0,\lambda_J,\bar\kappa_J)\); and
\(\widehat U_{X_0,A,\lambda_J,\alpha_J}\) is defined by
\eqref{eq:balanced-finite-split-inflated-intercept} with
\((X,\lambda,\alpha,\bar\kappa)
=(X_0,\lambda_J,\alpha_J,\bar\kappa_J)\).  Suppose
\begin{equation}
\label{eq:hybrid-taylor-subgaussian-conditions}
 \varepsilon_N(t_{\alpha_J})<1,
 \qquad
 \vartheta_{N,\lambda_J}(t_{\alpha_J})\le\frac12,
 \qquad
 \widehat U_{J,A,\lambda_J,\alpha_J}
 =\widehat U_{X_0,A,\lambda_J,\alpha_J}.
\end{equation}
\item The random vector \(X_0-PX_0\) is Gaussian, a known \(L_J<\infty\)
satisfies \(\|\operatorname{Cov}(X_0)\|_{\rm op}\le L_J\), and
\[
 \widehat U_{J,A,\lambda_J,\alpha_J}
 =\widehat U_{X_0,A,\lambda_J,\alpha_J}^{\rm G},
\]
where \(\widehat U_{X_0,A,\lambda_J,\alpha_J}^{\rm G}\) is defined by
\eqref{eq:gaussian-finite-split-inflated-intercept} with
\((X,\lambda,\alpha,L)=(X_0,\lambda_J,\alpha_J,L_J)\).
\item A known \(B_J<\infty\) satisfies \(\|X_0\|\le B_J\) almost surely, and
\[
 \widehat U_{J,A,\lambda_J,\alpha_J}
 =\widehat U_{X_0,A,\lambda_J,\alpha_J}^{\rm bd},
\]
where \(\widehat U_{X_0,A,\lambda_J,\alpha_J}^{\rm bd}\) is defined in
Corollary~\ref{cor:bounded-feature-one-split-linear-confidence} with
\((X,\lambda,\alpha,B)=(X_0,\lambda_J,\alpha_J,B_J)\).
\end{enumerate}

For an included crossing component (\(\iota_C=1\)), assume
\(0\in\mathcal D_{\rho,s}\),
\eqref{eq:pullback-lipschitz-increments} holds on \(\mathcal D_\rho\) with
crossing-slope operator \(Y_C\), and a known \(K_C<\infty\) satisfies
\(\|Y_C\|_{\rm HS}\le K_C\) almost surely.  For an absent crossing component
(\(\iota_C=0\)), assume \(C_v=0\) for every
\(v\in\mathcal D_{\rho,s}\) on a single \(P\)-probability-one set.

For an included remainder component (\(\iota_R=1\)), assume
\eqref{eq:radial-remainder-smoothness-envelope} holds and a known
\(K_R<\infty\) satisfies \(0\le\Lambda_\rho\le K_R\) almost surely.
For an absent remainder component (\(\iota_R=0\)), assume
\(R_v^{\rm fr}=0\) for every \(v\in\mathcal D_{\rho,s}\) on a single
\(P\)-probability-one set.

Then
\begin{equation}
\label{eq:hybrid-local-confidence-coverage}
 \Pr\left\{
 \widehat{\mathsf L}_{\rho,s,\boldsymbol\alpha}^{\rm hyb}(v)
 \le P\delta_v\le
 \widehat{\mathsf U}_{\rho,s,\boldsymbol\alpha}^{\rm hyb}(v)
 \text{ for every }v\in\mathcal D_{\rho,s}
 \right\}\ge1-\alpha.
\end{equation}
\end{theorem}

The proof is in \SuppSection{app:crossfit-composite-lemmas}.

Along even sample sizes \(n\to\infty\), fix \(\iota_C,\iota_R\in\{0,1\}\), take
\(0<\rho_n\le1\) and \(s_n>0\), and let
\(\boldsymbol\alpha_n=(\alpha_{J,n},\alpha_{C,n},\alpha_{R,n})\) and
\(\alpha_n\in(0,1)\) satisfy
\(\alpha_{J,n}+\iota_C\alpha_{C,n}+\iota_R\alpha_{R,n}\le\alpha_n\).
Under the uniform conditions in
\SuppResult{Proposition}{supp:prop:assembled-confidence-width-rate}, write
\begin{align*}
 W_n(v)&:=\frac12\left\{
 \widehat{\mathsf U}_{\rho_n,s_n,\boldsymbol\alpha_n}^{\rm hyb}(v)
 -\widehat{\mathsf L}_{\rho_n,s_n,\boldsymbol\alpha_n}^{\rm hyb}(v)
 \right\},\\
 \Sigma_{J,n}&:=\operatorname{Cov}(X_{0,n}),\qquad
 d_{J,n,\lambda_{J,n}}
 :=\operatorname{tr}\{\Sigma_{J,n}
 (\Sigma_{J,n}+\lambda_{J,n}I)^{-1}\},\\
 t_{J,n}&:=\log\frac{c_1}{\alpha_{J,n}},\qquad
 \tau_n:=d_{J,n,\lambda_{J,n}}+t_{J,n}
 +\iota_C\log\frac2{\alpha_{C,n}}
 +\iota_R\log\frac1{\alpha_{R,n}}.
\end{align*}
The uniform half-width then satisfies
\begin{equation}
\label{eq:assembled-local-simple-width-rate}
 \sup_{v\in\mathcal D_{\rho_n,s_n}}W_n(v)
 =O_p\left\{\rho_n\sqrt{\frac{\tau_n}{n}}
 +\iota_R\rho_n^{m+1}\right\}.
\end{equation}

The confidence diameter vanishes in probability when
\(\rho_n\sqrt{\tau_n/n}+\iota_R\rho_n^{m+1}\to0\).  For a smooth loss,
\(\rho_n\sqrt{\tau_n/n}\) is the Taylor fluctuation and
\(\iota_R\rho_n^{m+1}\) is the Taylor remainder; for LAD,
\(\iota_R\rho_n^{m+1}=0\).

For a deterministic \(u\) satisfying
\(\rho_nu\in\mathcal D_{\rho_n,s_n}\), suppose
\(P\delta_{\rho_nu}=-c\rho_n^2+o(\rho_n^2)\) for some \(c>0\).  If
\(m\ge2\), \(\rho_n\to0\), and
\(\sqrt{\tau_n/n}=o(\rho_n)\), then
\eqref{eq:assembled-local-simple-width-rate} gives
\(W_n(\rho_nu)=o_p(\rho_n^2)\).  Because the reported interval at
\(\rho_nu\) has diameter \(2W_n(\rho_nu)\), the event in
\eqref{eq:hybrid-local-confidence-coverage} gives
\[
 \widehat{\mathsf U}_{\rho_n,s_n,\boldsymbol\alpha_n}^{\rm hyb}(\rho_nu)
 \le P\delta_{\rho_nu}+2W_n(\rho_nu)<0
\]
with probability at least \(1-\alpha_n-o(1)\).

Theorem~\ref{thm:componentwise-hybrid-local-confidence} covers the full
increment, including \(C_v\) and \(R_v^{\rm fr}\).  For a globally smooth loss,
or for a piecewise-smooth loss whose reported update paths avoid
\(\mathcal T\) almost surely, \(C_v\equiv0\), and one may take
\(\iota_C=0\); for LAD take \(\iota_R=0\); and for an affine loss take both
indicators zero and \(\alpha_J=\alpha\).  A general piecewise-smooth loss
uses \(J_v^0\), \(C_v\), and \(R_v^{\rm fr}\), and therefore requires
\eqref{eq:pullback-lipschitz-increments} and
\eqref{eq:radial-remainder-smoothness-envelope} in addition to condition
(PS).
The LAD calculation in \eqref{eq:lad-crossing-term} and
\eqref{eq:lad-crossing-multiplier} supplies the component bounds needed to
specialize Theorem~\ref{thm:componentwise-hybrid-local-confidence}.  Under
\(\|x\|\le K\) almost surely,
Corollary~\ref{cor:finite-sample-lad-endpoints} covers \(P\delta_v\)
simultaneously for every \(\|v\|\le\rho\) under arbitrary residual
distributions.  The proof of Corollary~\ref{cor:finite-sample-lad-endpoints}
and the explicit LAD matrices are in
\SuppSection{app:worked-examples}; the finite-grid empirical-Bernstein
alternative is in
\SuppResult{Proposition}{prop:supp-lad-active-empirical-bernstein}.

\begin{corollary}[Simultaneous LAD confidence endpoints]
\label{cor:finite-sample-lad-endpoints}
Let \(n=2N\) with \(N\ge2\), and let \(A\subset[n]\) be a balanced split,
either fixed or randomized independently of the observations; put \(B=A^c\).
Fix \(\theta_0\in\mathbb R^d\) and
\(\rho,\lambda_J,\eta_C,\lambda_C>0\).  On
\(\Theta=\mathcal D=\mathbb R^d\), consider
\(\ell(\theta;(x,y))=|y-x^\top\theta|\) and
\(\mathcal D_\rho=\{v:\|v\|\le\rho\}\).  Suppose a known \(K<\infty\)
satisfies \(\|x\|\le K\) almost surely.

Choose \(\alpha,\alpha_J,\alpha_C\in(0,1)\) with
\(\alpha_J+\alpha_C\le\alpha\).  Let \(\alpha_R\in(0,1)\) be arbitrary and
put \(\boldsymbol\alpha=(\alpha_J,\alpha_C,\alpha_R)\).
\Needspace{13\baselineskip}
Write \(e=y-x^\top\theta_0\), with \(\operatorname{sgn}(0)=0\), and define
\(Y_\rho a:=L_\rho x^\top a\), where \(L_\rho\) is given in
\eqref{eq:lad-crossing-multiplier}.  Use the specified component choices in the
assembled endpoints, evaluating the bounded-feature intercept with the norm
bound \(K\) as in
Corollary~\ref{cor:bounded-feature-one-split-linear-confidence}:
\[
 \begin{gathered}
 m=1,\qquad s=\rho,\qquad
 X_0=-\operatorname{sgn}(e)x,\qquad Y_C=Y_\rho,\\
 (K_C,q_C,\iota_C,\iota_R)=(2K,1,1,0),\qquad
 \widehat U_{J,A,\lambda_J,\alpha_J}
 =\widehat U_{X_0,A,\lambda_J,\alpha_J}^{\rm bd}.
 \end{gathered}
\]
Then, with probability at least \(1-\alpha\),
\begin{equation}
\label{eq:lad-simultaneous-confidence-coverage}
 \widehat{\mathsf L}_{\rho,\rho,\boldsymbol\alpha}^{\rm hyb}(v)
 \le P\delta_v\le
 \widehat{\mathsf U}_{\rho,\rho,\boldsymbol\alpha}^{\rm hyb}(v),
 \qquad \|v\|\le\rho.
\end{equation}
The probability is over the observations when \(A\) is fixed and jointly over
\(A\) and the observations when \(A\) is randomized.
\end{corollary} 
\clearpage
\appendix
\renewcommand{\thesection}{S.\arabic{section}}
\numberwithin{equation}{section}
\numberwithin{figure}{section}
\numberwithin{table}{section}

\section*{Guide to the Supplement}

\begin{center}
\small
\renewcommand{\arraystretch}{1.00}
\begin{tabular}{@{}
 >{\raggedright\arraybackslash}p{0.30\textwidth}
 >{\raggedright\arraybackslash}p{0.64\textwidth}@{}}
\toprule
Main-paper result & Proof location and content \\
\midrule
Propositions~\MainXRefNumber{prop:mean-valid-decision}--
\MainXRefNumber{prop:orientation-selection-mean-cost} &
Section~\ref{supp:sec:decomposition-validity}:
Decision consequences of mean validity, the Gaussian calculation, and the
effect of anchoring an endpoint at the null update. \\

Theorem~\MainXRefNumber{thm:centered-linear-expected-certificate} &
Section~\ref{supp:subsec:linear-mean-valid-proof}:
Complete-subset identities, conditional second-moment calibration, and the
mean-valid linear field. \\

Theorem~\MainXRefNumber{thm:balanced-one-split-linear-confidence};
Corollaries~\MainXRefNumber{cor:gaussian-one-split-linear-confidence}--
\MainXRefNumber{cor:bounded-feature-one-split-linear-confidence} &
Section~\ref{supp:subsec:pilot-conditional-confidence}:
Hilbert-space concentration, observable trace bounds, pilot geometry, and the
Gaussian and bounded-feature specializations. \\

Theorem~\MainXRefNumber{thm:balanced-one-split-linear-confidence-field};
Corollary~\MainXRefNumber{cor:balanced-one-split-oracle-width} &
Sections~\ref{supp:subsec:pilot-conditional-confidence}--
\ref{supp:subsec:finite-split-proofs-rates}:
Regularized covariance control, the one-split ellipsoid, the population-width
bound for the one-split ellipsoid, and the finite prespecified-split extension. \\

Theorem~\MainXRefNumber{thm:componentwise-hybrid-local-confidence};
Equation~(\MainXRefNumber{eq:assembled-local-simple-width-rate}) &
Sections~\ref{app:finite-split-constants}--
\ref{supp:subsec:assembly-uniform-width}:
Bounded-difference conversion, pullback contraction, replacement stability,
component assembly, and the uniform half-width rate. \\

Proposition~\MainXRefNumber{prop:finite-selection-crossing-slope};
Corollary~\MainXRefNumber{cor:finite-sample-lad-endpoints} &
Section~\ref{app:worked-examples}:
Finite-selection crossing slopes, gradient-envelope calculations, and the
LAD specialization, including the direct full-ball width comparison. \\
\bottomrule
\end{tabular}
\end{center}
 \setcounter{section}{0}
\section{Mean validity and decision consequences}
\label{supp:sec:decomposition-validity}

In this section, \(\mathcal D\) is the deterministic update space,
\(\delta_v\) is the population-risk increment indexed by
\(v\in\mathcal D\), and \(P\) and \(P_n\) are the population and empirical
averages from \MainSection{sec:finite-sample-certificates}.  Upper and lower
mean validity refer to
\MainResult{Definition}{def:expected-local-endpoint}.

\begin{proof}[Proof of
\MainResult{Proposition}{prop:mean-valid-decision}]
Put
\[
 G_{\mathsf U}:=\sup_{w\in\mathcal D}
 \{P\delta_w-\widehat{\mathsf U}_{\mathcal D}(w)\}.
\]
Mean validity implies
\((G_{\mathsf U})_+\in L^1\) and \(\mathbb E G_{\mathsf U}\le0\).

For the oracle bound, assume that \(P\delta_{\widehat v}\) is integrable and
fix \(v\in\mathcal D\).  On the probability-one optimization event,
\[
 P\delta_{\widehat v}
 -\widehat{\mathsf U}_{\mathcal D}(v)-\varepsilon_{\rm opt}
 \le P\delta_{\widehat v}
 -\widehat{\mathsf U}_{\mathcal D}(\widehat v)
 \le G_{\mathsf U}.
\]
Because \(P\delta_{\widehat v}\in L^1\),
\(P_n\delta_v\in L^1\) for deterministic \(v\), and
\(\widehat{\mathsf A}_{\mathcal D}(v)\in L^1\),
\(P\delta_{\widehat v}
-\widehat{\mathsf U}_{\mathcal D}(v)-\varepsilon_{\rm opt}\in L^1\), so
\((G_{\mathsf U})_-
\le\{P\delta_{\widehat v}
 -\widehat{\mathsf U}_{\mathcal D}(v)-\varepsilon_{\rm opt}\}_-
\in L^1\); hence \(G_{\mathsf U}\in L^1\).  Taking expectations and using
\(\mathbb E P_n\delta_v=P\delta_v\) gives
\[
 \mathbb E\{P\delta_{\widehat v}\}
 \le P\delta_v+\mathbb E\widehat{\mathsf A}_{\mathcal D}(v)
 +\varepsilon_{\rm opt}+\mathbb E G_{\mathsf U}
 \le P\delta_v+\mathbb E\widehat{\mathsf A}_{\mathcal D}(v)
 +\varepsilon_{\rm opt}.
\]
Taking the infimum over \(v\in\mathcal D\) proves
\MainEquation{eq:mean-valid-risk-oracle}.

For the selected-gap bound, let \(\widetilde v\) be a measurable same-sample
selector for which
\((P-P_n)\delta_{\widetilde v}\) and
\(\widehat{\mathsf A}_{\mathcal D}(\widetilde v)\) are integrable, and set
\[
 H_{\widetilde v}:=(P-P_n)\delta_{\widetilde v}
 -\widehat{\mathsf A}_{\mathcal D}(\widetilde v)
 =P\delta_{\widetilde v}
 -\widehat{\mathsf U}_{\mathcal D}(\widetilde v)
 \le G_{\mathsf U}.
\]
Since \(H_{\widetilde v}\in L^1\), the inequality
\((G_{\mathsf U})_-\le(H_{\widetilde v})_-\) and mean validity imply
\(G_{\mathsf U}\in L^1\).  Therefore
\[
 \mathbb E\{(P-P_n)\delta_{\widetilde v}\}
 -\mathbb E\widehat{\mathsf A}_{\mathcal D}(\widetilde v)
 =\mathbb E H_{\widetilde v}
 \le\mathbb E G_{\mathsf U}\le0.
\]
Adding \(\mathbb E\widehat{\mathsf A}_{\mathcal D}(\widetilde v)\) to both
sides proves
\MainEquation{eq:mean-valid-selected-gap}.
\end{proof}

\begin{proof}[Proof of
\MainResult{Proposition}{prop:orientation-selection-mean-cost}]
Write
\[
 m=\frac{\overline X_A+\overline X_B}{2},
 \qquad D=\overline X_A-\overline X_B.
\]
For either individual orientation,
\[
 \sup_{h\in\mathbb R}\{\mu h-U_A(h)\}
 =\frac{(\mu-\overline X_A)^2}{2\eta M}-c_N,
\]
and
\[
 \mathbb E\frac{(\mu-\overline X_A)^2}{2\eta M}
 =\frac{\sigma^2/N}{2\eta M}=c_N.
\]
Exchanging \(A\) and \(B\) gives
\(\mathbb E\sup_h\{\mu h-U_B(h)\}=0\), so both individual orientations are
exactly mean-valid.  The lower envelope satisfies
\[
 \min\{\overline X_Ah,\overline X_Bh\}
 =mh-\frac{|D|}{2}|h|,
\]
and hence
\[
 \sup_{h\in\mathbb R}\{\mu h-U_\wedge(h)\}
 =\frac{\{ |\mu-m|+|D|/2\}^2}{2\eta M}-c_N.
\]
The variables \(\mu-m\) and \(D/2\) are jointly Gaussian, and
    \begin{equation}
    \label{eq:orientation-selection-variance}
     \operatorname{Var}(\mu-m)=\operatorname{Var}(D/2)
     =\frac{\sigma^2}{2N},\qquad
     \operatorname{Cov}(\mu-m,D/2)=0.
    \end{equation}
The zero covariance in \eqref{eq:orientation-selection-variance} and joint
Gaussianity make \(\mu-m\) and \(D/2\) independent.  Writing
\(\tau^2=\sigma^2/(2N)\) and using
\(\mathbb E|N(0,\tau^2)|=\tau\sqrt{2/\pi}\) gives
\begin{align*}
 \mathbb E\{|\mu-m|+|D|/2\}^2
 &=2\tau^2+2\{\tau\sqrt{2/\pi}\}^2\\
 &=\frac{\sigma^2}{N}\left(1+\frac2\pi\right).
\end{align*}
Therefore
\[
 \mathbb E\sup_h\{\mu h-U_\wedge(h)\}
 =\frac{\sigma^2}{2\eta MN}\left(1+\frac2\pi\right)-c_N
 =\frac{\sigma^2}{\pi\eta MN}=d_N.
\]
The equality
\(\mathbb E\sup_h\{\mu h-U_\wedge(h)\}=d_N\) proves the defect in
\MainResult{Proposition}{prop:orientation-selection-mean-cost}.  For the
additive correction,
\begin{equation}
\label{eq:orientation-selection-shift}
 \sup_h\{\mu h-[U_\wedge(h)+d_N]\}
 =\sup_h\{\mu h-U_\wedge(h)\}-d_N.
\end{equation}
The expectation of
\(\sup_h\{\mu h-U_\wedge(h)\}-d_N\) is zero, so \(U_\wedge+d_N\) is
exactly mean-valid.

Finally,
\begin{equation}
\label{eq:orientation-cross-calibrated-supremum}
 \sup_{h\in\mathbb R}\{\mu h-U_{\rm cf}(h)\}
 =\frac{(\mu-m)^2}{2\eta M}-\frac{D^2}{8\eta M}.
\end{equation}
Equation~\eqref{eq:orientation-selection-variance} gives
\[
 \mathbb E(\mu-m)^2=\frac{\sigma^2}{2N}
 =\mathbb E\frac{D^2}{4}.
\]
Consequently,
\[
 \mathbb E\left\{
 \frac{(\mu-m)^2}{2\eta M}-\frac{D^2}{8\eta M}
 \right\}=0,
\]
so \eqref{eq:orientation-cross-calibrated-supremum} makes the
cross-calibrated field exactly mean-valid.
\end{proof}

If an upper endpoint satisfies
\(\widehat{\mathsf U}_{\mathcal D}(0)=0\), then
\(G_{\mathsf U}\ge0\); the analogous lower-endpoint condition gives
\(G_{\mathsf L}\ge0\).  Lemma~\ref{supp:lem:null-anchor-obstruction} shows
that mean validity forces \(G_{\mathsf U}=0\) almost surely for the upper
field and \(G_{\mathsf L}=0\) almost surely for the lower field.

\begin{lemma}
\label{supp:lem:null-anchor-obstruction}
Suppose \(0\in\mathcal D\), \(\delta_0\equiv0\), and
\(\widehat{\mathsf U}_{\mathcal D}(0)=0\) almost surely.  If
\(\widehat{\mathsf U}_{\mathcal D}\) is mean-valid, then
\(P\delta_v\le\widehat{\mathsf U}_{\mathcal D}(v)\) simultaneously for all
\(v\in\mathcal D\) on one event of probability one.  If a mean-valid lower
field satisfies \(\widehat{\mathsf L}_{\mathcal D}(0)=0\) almost surely, then
\(\widehat{\mathsf L}_{\mathcal D}(v)\le P\delta_v\) simultaneously for all
\(v\in\mathcal D\) on one event of probability one.
\end{lemma}

\begin{proof}
After modifying \(\widehat{\mathsf U}_{\mathcal D}\) on a measurable null
set, assume that
\(\widehat{\mathsf U}_{\mathcal D}(0)=0\) everywhere.  Changing
\(\widehat{\mathsf U}_{\mathcal D}\) only on the measurable null set
preserves mean validity and the almost-sure conclusion.  Let
\(
 H:=\sup_{v\in\mathcal D}
 \{P\delta_v-\widehat{\mathsf U}_{\mathcal D}(v)\}
\).
The measurability requirement in
\MainResult{Definition}{def:expected-local-endpoint} applies to \(H\).
The variable \(H\) is nonnegative because the indexed gap is zero at \(v=0\),
while mean validity gives \(\mathbb EH\le0\).  Hence \(H=0\) almost surely,
proving the simultaneous assertion.

For the lower field, let
\[
 H_-:=\sup_{v\in\mathcal D}
 \{\widehat{\mathsf L}_{\mathcal D}(v)-P\delta_v\}.
\]
The equality \(\widehat{\mathsf L}_{\mathcal D}(0)=0\) gives \(H_-\ge0\),
and lower mean validity gives
\(\mathbb EH_-\le0\).  Hence \(H_-=0\) almost surely.  The identity
\(H_-=0\) is equivalent to
\(\widehat{\mathsf L}_{\mathcal D}(v)\le P\delta_v\) simultaneously for all
\(v\).
\end{proof}
 \Needspace{6\baselineskip}
\section{Linear Cross-Fitted Calibration}
\label{app:crossfit-linear-lemmas}

Throughout this section we use the notation of
\MainSection{sec:finite-sample-certificates}; the observations
\(Z_1,\ldots,Z_n\) remain independent with common law \(P\).  Thus \(n\ge4\),
\([n]=\{1,\ldots,n\}\), \(n_{\rm p}\in\{2,\ldots,\lfloor n/2\rfloor\}\),
\(n_{\rm e}=n-n_{\rm p}\),
\(\mathfrak A_{n_{\rm p}}=\{A\subset[n]:|A|=n_{\rm p}\}\), and
\(B_A=[n]\setminus A\).  For a split-indexed quantity \(g_A\),
\[
 \operatorname{Av}_A g_A
 :=\binom{n}{n_{\rm p}}^{-1}
 \sum_{A\in\mathfrak A_{n_{\rm p}}}g_A.
\]
For \(X=X(Z)\in L^2(P;\mathcal H)\), let \(\overline X_I\) and
\(\widehat\Sigma_{X,I}\) be the empirical mean and unbiased covariance on a
fold \(I\), put
\(M_{X,A,\lambda}:=\widehat\Sigma_{X,A}+\lambda I\) for \(\lambda>0\),
and write \(f_h(z):=\langle X(z),h\rangle\).

\subsection{Second-moment calibration and the mean-valid field}
\label{supp:subsec:linear-mean-valid-proof}

The quadratic penalty in
\MainResult{Theorem}{thm:centered-linear-expected-certificate} is recovered by
averaging over all pilot folds.

\begin{lemma}[Complete-subset averaging identities]
\label{lem:crossfit-u-average-identities}
For every realized sample,
\begin{align}
 \operatorname{Av}_{A}\widehat\Sigma_{X,A}&=\widehat\Sigma_X,
 \label{eq:crossfit-average-pilot-covariance}\\
 \operatorname{Av}_{A}(\overline X_{B_A}-\overline X_A)^{\otimes2}
 &=\frac{n}{n_{\rm p}n_{\rm e}}\widehat\Sigma_X.
 \label{eq:crossfit-average-mean-difference}
\end{align}
Equation~\eqref{eq:crossfit-average-pilot-covariance} also holds with \(A\)
replaced by \(B_A\).
\end{lemma}

\begin{proof}
In the pairwise covariance formula
\MainEquation{eq:sample-covariance-pairwise}, a fixed
pair \(\{i,j\}\) belongs to
exactly \(\binom{n-2}{n_{\rm p}-2}\) pilot sets.  Dividing by
\(\binom{n}{n_{\rm p}}\) gives the inclusion probability
\(n_{\rm p}(n_{\rm p}-1)/\{n(n-1)\}\).  Therefore
\[
 \operatorname{Av}_A\widehat\Sigma_{X,A}
 =\frac1{n_{\rm p}(n_{\rm p}-1)}
 \frac{n_{\rm p}(n_{\rm p}-1)}{n(n-1)}
 \sum_{\{i,j\}\subset[n]}(X_i-X_j)^{\otimes2}
 =\widehat\Sigma_X,
\]
which is \eqref{eq:crossfit-average-pilot-covariance}.

To prove \eqref{eq:crossfit-average-mean-difference}, put
\(\widetilde X_i:=X_i-\overline X\), so that
\(\sum_i\widetilde X_i=0\), and use
\[
 \overline X_A-\overline X=\frac1{n_{\rm p}}\sum_{i\in A}\widetilde X_i.
\]
A fixed index belongs to a uniform pilot set with probability
\(n_{\rm p}/n\), while a fixed ordered pair of distinct indices belongs with
probability \(n_{\rm p}(n_{\rm p}-1)/\{n(n-1)\}\).  Hence
\begin{align*}
 \operatorname{Av}_{A}
 (\overline X_A-\overline X)^{\otimes2}
 &=\frac1{n_{\rm p}^2}\left[
 \frac{n_{\rm p}}n\sum_i\widetilde X_i^{\otimes2}
 +\frac{n_{\rm p}(n_{\rm p}-1)}{n(n-1)}
 \sum_{i\ne j}\widetilde X_i\otimes\widetilde X_j\right]
 =\frac{n_{\rm e}}{n_{\rm p}n}\widehat\Sigma_X,
\end{align*}
where
\(\sum_{i\ne j}\widetilde X_i\otimes\widetilde X_j
 =-\sum_i\widetilde X_i^{\otimes2}\).
Multiplying the relation
\(\overline X_{B_A}-\overline X_A=n(\overline X-\overline X_A)/n_{\rm e}\)
by \(n^2/n_{\rm e}^2\) yields
\eqref{eq:crossfit-average-mean-difference}.
\end{proof}

To identify the conditional expectation of \(\widehat T_{X,A,\lambda}\),
condition on a fixed pilot fold.

\begin{lemma}[Conditional second-moment calibration]
\label{lem:conditional-second-moment-calibration}
For every fixed split \(A\in\mathfrak A_{n_{\rm p}}\),
\[
 \mathbb E_{B_A}\widehat T_{X,A,\lambda}
 =\mathbb E_{B_A}\|PX-\overline X\|_{M_{X,A,\lambda}^{-1}}^2
\quad\text{almost surely in }S_A.
\]
\end{lemma}

\begin{proof}
Condition on \(S_A\).  Since
\[
 PX-\overline X=\frac{n_{\rm p}}n(PX-\overline X_A)
 +\frac{n_{\rm e}}n(PX-\overline X_{B_A}),
\]
and \(\mathbb E_{B_A}(PX-\overline X_{B_A})=0\), the cross term vanishes after
conditioning on \(S_A\).  Also
\(\mathbb E_{B_A}\|PX-\overline X_{B_A}\|_{M_{X,A,\lambda}^{-1}}^2
=n_{\rm e}^{-1}\operatorname{tr}\{\Sigma_XM_{X,A,\lambda}^{-1}\}\).
Hence
\begin{equation}
\label{eq:conditional-full-mean-error}
 \mathbb E_{B_A}\|PX-\overline X\|_{M_{X,A,\lambda}^{-1}}^2
 =\frac{n_{\rm p}^2}{n^2}
 \|PX-\overline X_A\|_{M_{X,A,\lambda}^{-1}}^2
 +\frac{n_{\rm e}}{n^2}
 \operatorname{tr}\{\Sigma_XM_{X,A,\lambda}^{-1}\}.
\end{equation}
The two terms in \(\widehat T_{X,A,\lambda}\) have conditional expectations
\begin{align}
 \mathbb E_{B_A}
 \|\overline X_{B_A}-\overline X_A\|_{M_{X,A,\lambda}^{-1}}^2
 &=\|PX-\overline X_A\|_{M_{X,A,\lambda}^{-1}}^2
 +\frac1{n_{\rm e}}
 \operatorname{tr}\{\Sigma_XM_{X,A,\lambda}^{-1}\},
 \label{eq:conditional-mean-difference-moment}\\
 \mathbb E_{B_A}
 \operatorname{tr}\{\widehat\Sigma_{X,B_A}M_{X,A,\lambda}^{-1}\}
 &=\operatorname{tr}\{\Sigma_XM_{X,A,\lambda}^{-1}\}.
 \label{eq:conditional-evaluation-covariance-moment}
\end{align}
Equality \eqref{eq:conditional-mean-difference-moment} uses the decomposition
\[
 \overline X_{B_A}-\overline X_A=(PX-\overline X_A)+(\overline X_{B_A}-PX)
\]
and a vanishing cross term.  Equality
\eqref{eq:conditional-evaluation-covariance-moment} is the unbiasedness of the
sample covariance, which also follows directly by taking expectations in
\MainEquation{eq:sample-covariance-pairwise}.  Substituting
\eqref{eq:conditional-mean-difference-moment} and
\eqref{eq:conditional-evaluation-covariance-moment} into the single-split
calibration statistic
\MainEquation{eq:crossfit-single-split-calibration}, the
coefficient of
\(\operatorname{tr}\{\Sigma_XM_{X,A,\lambda}^{-1}\}\) becomes
\[
 \frac{n_{\rm p}^2}{n^2n_{\rm e}}
 +\frac{n_{\rm e}-n_{\rm p}}{n_{\rm e}n}=\frac{n_{\rm e}}{n^2}.
\]
Thus
\[
 \mathbb E_{B_A}\widehat T_{X,A,\lambda}
 =\mathbb E_{B_A}
 \|PX-\overline X\|_{M_{X,A,\lambda}^{-1}}^2,
\]
which is the equality in \eqref{eq:conditional-full-mean-error}.
\end{proof}

Proposition~\ref{prop:conditional-linear-intercept-optimality} substitutes the
conditional equality in Lemma~\ref{lem:conditional-second-moment-calibration}
into a completion of squares over \(K\).  The completed square yields the
conditional mean-validity bound and quantifies the cost of restricting the
optimization domain to \(K\).

\Needspace{0.70\textheight}
\begin{proposition}[Conditional candidate-set identity]
\label{prop:conditional-linear-intercept-optimality}
Assume \(X\in L^2(P;\mathcal H)\).  Fix a pilot fold
\(A\in\mathfrak A_{n_{\rm p}}\), \(\eta,\lambda>0\), a sign
\(\sigma\in\{-1,1\}\), and a nonempty deterministic set
\(K\subset\mathcal H\).  Since
\(M_{X,A,\lambda}\succeq\lambda I\), the operator
\(M_{X,A,\lambda}\) induces a norm; write
\[
 d_{M_{X,A,\lambda}}(x,K)
 :=\inf_{h\in K}\|x-h\|_{M_{X,A,\lambda}},
 \qquad x\in\mathcal H,
\]
for the distance from \(x\) to \(K\) in the
\(M_{X,A,\lambda}\)-norm,
and define the Hilbert-space element
\[
 h_{A,\sigma}^\star
 :=\frac{\sigma}{\eta}M_{X,A,\lambda}^{-1}(PX-P_nX).
\]
If \(I_A\) is an almost surely finite, full-sample-measurable random
variable that is constant in \(h\) and conditionally integrable over
\(B_A\), then, almost surely,
\begin{align}
 &\mathbb E_{B_A}\sup_{h\in K}
 \left\{\sigma(P-P_n)f_h
 -\frac{\eta}{2}\langle h,M_{X,A,\lambda}h\rangle-I_A\right\}
 \notag\\
 &\quad=
 \frac1{2\eta}\mathbb E_{B_A}\widehat T_{X,A,\lambda}
 -\frac\eta2\mathbb E_{B_A}
 d_{M_{X,A,\lambda}}^2(h_{A,\sigma}^\star,K)
 -\mathbb E_{B_A}I_A.
\label{eq:conditional-linear-intercept-optimality}
\end{align}
\end{proposition}

\begin{proof}[Proof of
Proposition~\ref{prop:conditional-linear-intercept-optimality}]
By the definition of \(h_{A,\sigma}^\star\), for every \(h\in\mathcal H\),
\begin{align*}
 \frac\eta2\|h-h_{A,\sigma}^\star\|_{M_{X,A,\lambda}}^2
 &=\frac\eta2\|h\|_{M_{X,A,\lambda}}^2
 -\eta\langle h,M_{X,A,\lambda}h_{A,\sigma}^\star\rangle
 +\frac\eta2\|h_{A,\sigma}^\star\|_{M_{X,A,\lambda}}^2\\
 &=\frac\eta2\|h\|_{M_{X,A,\lambda}}^2
 -\sigma\langle PX-P_nX,h\rangle
 +\frac1{2\eta}\|PX-P_nX\|_{M_{X,A,\lambda}^{-1}}^2.
\end{align*}
Rearranging gives
\[
 \sigma\langle PX-P_nX,h\rangle
 -\frac\eta2\|h\|_{M_{X,A,\lambda}}^2
 =
 \frac1{2\eta}\|PX-P_nX\|_{M_{X,A,\lambda}^{-1}}^2
 -\frac\eta2
 \|h-h_{A,\sigma}^\star\|_{M_{X,A,\lambda}}^2.
\]
Taking the supremum over \(K\) therefore gives
\begin{align*}
 &\sup_{h\in K}\left\{
 \sigma\langle PX-P_nX,h\rangle
 -\frac\eta2\|h\|_{M_{X,A,\lambda}}^2\right\}\\
 &\quad=
 \frac1{2\eta}\|PX-P_nX\|_{M_{X,A,\lambda}^{-1}}^2
 -\frac\eta2d_{M_{X,A,\lambda}}^2(h_{A,\sigma}^\star,K).
\end{align*}
Lemma~\ref{lem:conditional-second-moment-calibration} gives
\[
 \mathbb E_{B_A}\|PX-P_nX\|_{M_{X,A,\lambda}^{-1}}^2
 =\mathbb E_{B_A}\widehat T_{X,A,\lambda}.
\]
Taking conditional expectations proves
\eqref{eq:conditional-linear-intercept-optimality}.
\end{proof}

\begin{proof}[Proof of
\MainResult{Theorem}{thm:centered-linear-expected-certificate}]
For each \(A\in\mathfrak A_{n_{\rm p}}\), write
\[
 G_A(h):=\sigma(P-P_n)f_h
 -\frac\eta2\langle h,M_{X,A,\lambda}h\rangle
 -\frac{\widehat T_{X,A,\lambda}}{2\eta}.
\]
Since \(\mathcal H\) is separable, \(K\) has a deterministic countable dense
subset \(K^\circ\).  For every realized sample, \(G_A\) is continuous in
\(h\), so \(\sup_{h\in K}G_A(h)=\sup_{h\in K^\circ}G_A(h)\).  Thus all
splitwise suprema in this proof are suprema of countably many measurable
random variables and are therefore measurable.
Proposition~\ref{prop:conditional-linear-intercept-optimality}, with
\(I_A=\widehat T_{X,A,\lambda}/(2\eta)\), gives
\[
 \mathbb E_{B_A}\sup_{h\in K}G_A(h)
 =-\frac\eta2\mathbb E_{B_A}
 d_{M_{X,A,\lambda}}^2(h_{A,\sigma}^\star,K)
 \le0
\]
almost surely in \(S_A\).  Hence
\(\mathbb E\sup_{h\in K}G_A(h)\le0\).
Moreover, pointwise in the sample,
\[
 \sup_{h\in K}\operatorname{Av}_{A}G_A(h)
 \le\operatorname{Av}_{A}\sup_{h\in K}G_A(h).
\]
Taking expectations and applying the splitwise bound gives
\[
 \mathbb E\sup_{h\in K}\operatorname{Av}_{A}G_A(h)
 \le\operatorname{Av}_{A}\mathbb E\sup_{h\in K}G_A(h)
 \le0.
\]
Using Lemma~\ref{lem:crossfit-u-average-identities} to identify
\(\operatorname{Av}_{A}M_{X,A,\lambda}
 =\widehat\Sigma_X+\lambda I\), together with
the definition of the cross-fitted linear effective dimension in
\MainEquation{eq:centered-linear-crossfit-effective-dimension}, proves the two
signed expected-supremum inequalities in
\MainResult{Theorem}{thm:centered-linear-expected-certificate}.
\end{proof}

\subsection{Pilot-conditional confidence}
\label{supp:subsec:pilot-conditional-confidence}

The high-probability argument needs bounds for the evaluation-fold sample
mean and the evaluation-fold sample-covariance trace, both in the random
pilot metric.
For a real random variable \(W\), use the sub-Gaussian Orlicz norm
\begin{equation}
\label{eq:supp-psi-two-norm}
\|W\|_{\psi_2}:=
\inf\{c>0:\mathbb E\exp(W^2/c^2)\le2\}.
\end{equation}

\begin{lemma}[Quadratic concentration for a Hilbert-valued mean]
\label{lem:hilbert-subgaussian-quadratic-concentration}
Let \(\mathcal H\) be a real separable Hilbert space, let \(U\) be a centered
\(\mathcal H\)-valued random element satisfying
\(\mathbb E\|U\|^2<\infty\), and let
\(V=\mathbb E(U\otimes U)\) denote the covariance operator of \(U\).  Suppose
that a constant \(\kappa\ge1\) satisfies
\begin{equation}
\label{eq:hilbert-covariance-matched-subgaussian}
 \|\langle u,U\rangle\|_{\psi_2}
 \le\kappa\langle u,Vu\rangle^{1/2},
 \qquad u\in\mathcal H.
\end{equation}
Write
\[
 r:=\operatorname{tr}(V)=\mathbb E\|U\|^2,\qquad
 s:=\operatorname{tr}(V^2),\qquad
 \rho:=\|V\|_{\rm op}.
\]
Here \(r\) is the covariance trace, \(s\) is the squared-spectral trace, and
\(\rho\) is the operator norm.
There is a constant \(\gamma_\kappa\ge1\), depending only on \(\kappa\),
such that, for every integer \(N\ge1\), every collection of independent
copies \(U_1,\ldots,U_N\), and every \(t\ge1\), the sample mean
\(\overline U_N:=N^{-1}\sum_{i=1}^NU_i\) satisfies
\begin{equation}
\label{eq:hilbert-subgaussian-mean-quadratic-bound}
 \Pr\left\{
 \|\overline U_N\|^2>
 \frac{\gamma_\kappa}{N}
 \{r+2\sqrt{s t}+2\rho t\}
 \right\}\le e^{-t}.
\end{equation}
Moreover, there is a constant \(C_\kappa\ge1\), depending only on \(\kappa\),
such that
\begin{equation}
\label{eq:hilbert-norm-square-psi-one}
 \bigl\|\|U\|^2-r\bigr\|_{\psi_1}\le C_\kappa r,
\end{equation}
where, for a real random variable \(W\),
\(\|W\|_{\psi_1}:=\inf\{c>0:\mathbb E\exp(|W|/c)\le2\}\),
and, for every integer \(N\ge1\), independent copies
\(U_1,\ldots,U_N\), and \(t\ge1\),
\begin{equation}
\label{eq:raw-trace-bernstein}
 \Pr\left\{
 \left|\frac1N\sum_{i=1}^N\|U_i\|^2-r\right|
 >C_\kappa r\left(\sqrt{\frac tN}+\frac tN\right)
 \right\}\le2e^{-t}.
\end{equation}
\end{lemma}

\begin{proof}
If \(r=0\), then \(\mathbb E\|U\|^2=0\), so every \(U_i=0\) almost surely.
The random quantities on the left sides of
\eqref{eq:hilbert-subgaussian-mean-quadratic-bound},
\eqref{eq:hilbert-norm-square-psi-one}, and
\eqref{eq:raw-trace-bernstein} are then zero.  We henceforth assume \(r>0\).

For fixed \(u\in\mathcal H\), put \(Z_u:=\langle u,U\rangle\).  The random
variable \(Z_u\) is centered and satisfies
\(\|\langle u,U\rangle\|_{\psi_2}
\le\kappa\langle u,Vu\rangle^{1/2}\) by
\eqref{eq:hilbert-covariance-matched-subgaussian}.
Proposition~2.6.1 of \cite{vershynin2025high} then implies
\begin{equation}
\label{eq:scalar-subgaussian-exponential-bound}
 \mathbb E\exp(\theta Z_u)
 \le \exp\{C\kappa^2\theta^2\langle u,Vu\rangle\},
 \qquad \theta\in\mathbb R,
\end{equation}
for a universal constant \(C>0\).  Applying
\eqref{eq:scalar-subgaussian-exponential-bound} with
\(\theta=N^{-1}\) and using independence yields
\begin{equation}
\label{eq:hilbert-sample-mean-exponential-bound}
 \mathbb E\exp\langle u,\overline U_N\rangle
 =\prod_{i=1}^N
   \mathbb E\exp\!\left\{\frac{\langle u,U_i\rangle}{N}\right\}
 \le
 \exp\left\{\frac{\gamma_\kappa}{2N}\langle u,Vu\rangle\right\},
\end{equation}
after taking \(\gamma_\kappa\ge2C\kappa^2\).

We now apply Theorem~1 of \cite{hsu2012tail} in finite dimensions.  Let
\((\lambda_j,e_j)_{1\le j\le J}\), with
\(J\in\mathbb N\cup\{\infty\}\), be the positive
eigenpairs of \(V\), ordered so that the eigenvalues are nonincreasing.  For
\(m\le J\), let \(P_m\) project onto
\(\operatorname{span}\{e_1,\ldots,e_m\}\), and write
\[
 V_m:=P_mVP_m,\qquad
 r_m:=\operatorname{tr}(V_m),\qquad
 s_m:=\operatorname{tr}(V_m^2),\qquad
 \rho_m:=\|V_m\|_{\rm op}.
\]
Define \(Y_m\in\mathbb R^m\) by
\[
 (Y_m)_j
 :=\left(\frac{N}{\gamma_\kappa\lambda_j}\right)^{1/2}
   \langle e_j,\overline U_N\rangle,\qquad 1\le j\le m.
\]
For \(a\in\mathbb R^m\), set
\[
 u_a:=\sum_{j=1}^m
 a_j\left(\frac{N}{\gamma_\kappa\lambda_j}\right)^{1/2}e_j.
\]
The two quantities needed in
\eqref{eq:hilbert-sample-mean-exponential-bound} are
\begin{align*}
 \langle a,Y_m\rangle
 &=\sum_{j=1}^ma_j
 \left(\frac{N}{\gamma_\kappa\lambda_j}\right)^{1/2}
 \langle e_j,\overline U_N\rangle
 =\langle u_a,\overline U_N\rangle,\\
 \langle u_a,Vu_a\rangle
 &=\sum_{j=1}^ma_j^2
 \frac{N}{\gamma_\kappa\lambda_j}\langle e_j,Ve_j\rangle
 =\frac{N}{\gamma_\kappa}\sum_{j=1}^ma_j^2
 =\frac{N}{\gamma_\kappa}\|a\|^2.
\end{align*}
Here the cross terms vanish because \(Ve_j=\lambda_je_j\) and the
\(e_j\)'s are orthonormal.
Substitution in
\eqref{eq:hilbert-sample-mean-exponential-bound} therefore gives
\[
 \mathbb E\exp\langle a,Y_m\rangle\le\exp(\|a\|^2/2),
 \qquad a\in\mathbb R^m.
\]
Apply Theorem~1 of \cite{hsu2012tail} to \(Y_m\) with
\[
 A_m:=\left(\frac{\gamma_\kappa}{N}\right)^{1/2}
       \operatorname{diag}(\lambda_1^{1/2},\ldots,\lambda_m^{1/2}).
\]
For \(1\le j\le m\), the definitions of \(A_m\) and \(Y_m\) give
\[
 (A_mY_m)_j
 =\left(\frac{\gamma_\kappa\lambda_j}{N}\right)^{1/2}
  \left(\frac{N}{\gamma_\kappa\lambda_j}\right)^{1/2}
  \langle e_j,\overline U_N\rangle
 =\langle e_j,\overline U_N\rangle.
\]
Hence
\[
 \|A_mY_m\|^2
 =\sum_{j=1}^m\langle e_j,\overline U_N\rangle^2.
\]
Because \(P_m\) is the orthogonal projection onto
\(\operatorname{span}\{e_1,\ldots,e_m\}\) and the \(e_j\)'s are
orthonormal,
\[
 P_m\overline U_N
 =\sum_{j=1}^m\langle e_j,\overline U_N\rangle e_j,
 \qquad
 \|P_m\overline U_N\|^2
 =\sum_{j=1}^m\langle e_j,\overline U_N\rangle^2.
\]
Thus
\(\|A_mY_m\|^2
 =\|P_m\overline U_N\|^2\).  Moreover,
\[
 \operatorname{tr}(A_m^\top A_m)=\frac{\gamma_\kappa}{N}r_m,\quad
 \operatorname{tr}\{(A_m^\top A_m)^2\}
   =\frac{\gamma_\kappa^2}{N^2}s_m,\quad
 \|A_m^\top A_m\|_{\rm op}
   =\frac{\gamma_\kappa}{N}\rho_m.
\]
Theorem~1 of \cite{hsu2012tail} gives
\[
 \Pr\left\{\|P_m\overline U_N\|^2>
 \frac{\gamma_\kappa}{N}
 \bigl(r_m+2\sqrt{s_mt}+2\rho_mt\bigr)\right\}\le e^{-t}.
\]
Moreover, \(U\) has zero projection onto \(\ker(V)\) almost surely because
\(\mathbb E\|P_{\ker(V)}U\|^2=0\).  If \(J<\infty\), take \(m=J\).
If \(J=\infty\), then
\(\|P_m\overline U_N\|^2\uparrow\|\overline U_N\|^2\), while
\(r_m\uparrow r\), \(s_m\uparrow s\), and \(\rho_m\uparrow\rho\).
For any
\[
 x>\frac{\gamma_\kappa}{N}
 \{r+2\sqrt{st}+2\rho t\},
\]
the convergences
\(\|P_m\overline U_N\|^2\uparrow\|\overline U_N\|^2\),
\(r_m\uparrow r\), \(s_m\uparrow s\), and \(\rho_m\uparrow\rho\) imply
\[
 \{\|\overline U_N\|^2>x\}
 \subseteq
 \bigcup_{m_0\ge1}\bigcap_{m\ge m_0}
 \left\{\|P_m\overline U_N\|^2>
 \frac{\gamma_\kappa}{N}
 \{r_m+2\sqrt{s_mt}+2\rho_mt\}\right\}.
\]
If \(E_m\) denotes the finite-dimensional event inside braces, then
\[
 \Pr\left(\bigcup_{m_0\ge1}\bigcap_{m\ge m_0}E_m\right)
 =\lim_{m_0\to\infty}\Pr\left(\bigcap_{m\ge m_0}E_m\right)
 \le\lim_{m_0\to\infty}\inf_{m\ge m_0}\Pr(E_m)
 \le e^{-t}.
\]
Thus \(\Pr\{\|\overline U_N\|^2>x\}\le e^{-t}\).  Letting \(x\) decrease to
\(\gamma_\kappa\{r+2\sqrt{st}+2\rho t\}/N\) proves
\eqref{eq:hilbert-subgaussian-mean-quadratic-bound}.

Indeed, if \((\lambda_j)_j\) are the nonnegative eigenvalues of \(V\), then
\[
 s=\sum_j\lambda_j^2
 \le\left(\sum_j\lambda_j\right)^2=r^2,
 \qquad
 \rho=\sup_j\lambda_j\le\sum_j\lambda_j=r.
\]
Taking \(N=1\) in
\eqref{eq:hilbert-subgaussian-mean-quadratic-bound} therefore shows that
\(\Pr\{\|U\|^2>C_\kappa r(1+t)\}\le e^{-t}\) for \(t\ge1\).
For \(W=\|U\|^2/(C_\kappa r)\), the tail bound for \(\|U\|^2\) gives
\(\Pr\{W>y\}\le e^{1-y}\) for \(y\ge2\).  Integration of the tail gives
\begin{align*}
 \mathbb E e^{W/8}
 &=1+\frac18\int_0^\infty e^{y/8}\Pr\{W>y\}\,dy\\
 &\le1+\frac18\int_0^2e^{y/8}\,dy
 +\frac e8\int_2^\infty e^{-7y/8}\,dy
 <2.
\end{align*}
After replacing \(8C_\kappa\) by \(C_\kappa\), the exponential-moment estimate is
\[
 \mathbb E\exp\left\{\frac{\|U\|^2}{C_\kappa r}\right\}\le2.
\]
Since \(|\|U\|^2-r|\le\|U\|^2+r\), Cauchy--Schwarz gives
\[
 \mathbb E\exp\left\{\frac{|\|U\|^2-r|}{2C_\kappa r}\right\}
 \le e^{1/(2C_\kappa)}
 \left[\mathbb E\exp\left\{\frac{\|U\|^2}{C_\kappa r}\right\}\right]^{1/2}
 \le e^{1/(2C_\kappa)}\sqrt2.
\]
Increasing \(C_\kappa\) so that
\(e^{1/(2C_\kappa)}\sqrt2\le2\), and retaining the name \(C_\kappa\) for the
enlarged constant, gives
\begin{equation}
\label{eq:centered-norm-square-exponential-bound}
 \mathbb E\exp\left\{\frac{|\|U\|^2-r|}{C_\kappa r}\right\}\le2,
\end{equation}
which is \eqref{eq:hilbert-norm-square-psi-one}.

To derive \eqref{eq:raw-trace-bernstein}, set
\(Z_i=\|U_i\|^2-r\) and
\(K=C_\kappa r\).  Since the \(p\)th term in the exponential series is
nonnegative, \eqref{eq:centered-norm-square-exponential-bound} gives
\[
 \frac{\mathbb E|Z_i|^p}{p!K^p}
 \le\mathbb E\exp\left(\frac{|Z_i|}{K}\right)\le2,
 \qquad p\ge2,
\]
and hence \(\mathbb E|Z_i|^p\le2p!K^p\).
Because \(\mathbb EZ_i=0\), expansion of the exponential series shows that,
for \(|\theta|K\le1/2\),
\[
 \mathbb E e^{\theta Z_i}
 \le1+2\sum_{p=2}^{\infty}(|\theta|K)^p
 \le\exp(4\theta^2K^2).
\]
Independence therefore gives
\[
 \mathbb E\exp\left\{\theta\sum_{i=1}^NZ_i\right\}
 \le\exp(4N\theta^2K^2),
 \qquad |\theta|K\le\frac12.
\]
For \(x>0\) and \(0<\theta\le1/(2K)\),
\[
 \Pr\left\{\sum_{i=1}^NZ_i\ge x\right\}
 =\Pr\left\{\exp\left(\theta\sum_{i=1}^NZ_i\right)
 \ge e^{\theta x}\right\}
 \le\exp(-\theta x+4N\theta^2K^2).
\]
Choose
\[
 \theta=\min\left\{\frac{x}{8NK^2},\frac1{4K}\right\}.
\]
If \(x\le2NK\), substituting \(\theta=x/(8NK^2)\) gives
\[
 \exp\left\{
 -\frac{x^2}{8NK^2}+\frac{x^2}{16NK^2}\right\}
 =\exp\left\{-\frac{x^2}{16NK^2}\right\}.
\]
If \(x>2NK\), the exponent is
\[
 -\frac{x}{4K}+\frac N4
 <-\frac{x}{8K}.
\]
The estimate for \(\mathbb E e^{\theta Z_i}\) was proved for both signs of
\(\theta\), so replacing \(Z_i\) by \(-Z_i\) leaves the exponential-moment
bound unchanged.  Applying the two
one-sided bounds and the union bound therefore gives
\[
 \Pr\left\{\left|\sum_{i=1}^NZ_i\right|>x\right\}
 \le2\exp\left\{-\min\left(
 \frac{x^2}{16NK^2},\frac{x}{8K}\right)\right\}.
\]
For \(x=8K(\sqrt{Nt}+t)\),
\[
 \frac{x^2}{16NK^2}
 =\frac{4(\sqrt{Nt}+t)^2}{N}\ge4t,
 \qquad
 \frac{x}{8K}=\sqrt{Nt}+t\ge t.
\]
Substitution gives
\[
 \Pr\left\{\left|\sum_{i=1}^NZ_i\right|
 >8K(\sqrt{Nt}+t)\right\}\le2e^{-t}.
\]
Dividing by \(N\) and absorbing \(8\) into \(C_\kappa\) proves
\eqref{eq:raw-trace-bernstein}.
\end{proof}

For Gaussian evaluation observations,
Lemma~\ref{lem:hilbert-gaussian-mean-trace-concentration} replaces
\eqref{eq:hilbert-subgaussian-mean-quadratic-bound} and
\eqref{eq:raw-trace-bernstein} by exact weighted-\(\chi^2\) tails in the proof
of \MainResult{Corollary}{cor:gaussian-one-split-linear-confidence}.

\Needspace{0.70\textheight}
\begin{lemma}[Exact Gaussian mean and covariance-trace tails]
\label{lem:hilbert-gaussian-mean-trace-concentration}
Let \(\mathcal H\) be a real separable Hilbert space, and let \(G\) be a
centered Gaussian \(\mathcal H\)-valued random element with covariance \(V\)
satisfying
\(\operatorname{tr}(V)<\infty\); equivalently, the nonnegative eigenvalues of
\(V\) are summable.  Write
\[
 r:=\operatorname{tr}(V),\qquad
 s:=\operatorname{tr}(V^2),\qquad
 \rho:=\|V\|_{\rm op}.
\]
Here \(r\) is the covariance trace, \(s\) is the squared-spectral trace, and
\(\rho\) is the operator norm; all three scalars are nonnegative.
For independent copies \(G_1,\ldots,G_N\), where \(N\ge2\), let
\[
 \overline G_N:=\frac1N\sum_{i=1}^NG_i,\qquad
 \widehat V_N:=\frac1{N-1}\sum_{i=1}^N
 (G_i-\overline G_N)\otimes(G_i-\overline G_N),\qquad
 \widehat r_N:=\operatorname{tr}(\widehat V_N).
\]
Here \(\overline G_N\in\mathcal H\) is the sample mean,
\(\widehat V_N:\mathcal H\to\mathcal H\) is the unbiased sample covariance
operator, and \(\widehat r_N\in[0,\infty)\) is the trace of
\(\widehat V_N\).
Then, for every \(t>0\),
\begin{align}
 \Pr\left\{\|\overline G_N\|^2>
 \frac{r+2\sqrt{s t}+2\rho t}{N}\right\}
 &\le e^{-t},
 \label{eq:hilbert-gaussian-mean-tail}\\
 \Pr\left\{\widehat r_N<
 r-2\sqrt{\frac{s t}{N-1}}\right\}
 &\le e^{-t},
 \label{eq:hilbert-gaussian-trace-lower-tail}\\
 \Pr\left\{\widehat r_N>
 r+2\sqrt{\frac{s t}{N-1}}+
 \frac{2\rho t}{N-1}\right\}
 &\le e^{-t}.
 \label{eq:hilbert-gaussian-trace-upper-tail}
\end{align}
The conclusions remain valid when \(V\) is singular, including \(V=0\).
\end{lemma}

\begin{proof}
Put \(\nu=N-1\).  Choose an orthogonal transformation of the \(N\) sample
coordinates, represented by an orthogonal matrix
\((q_{\ell i})_{0\le\ell\le\nu,\,1\le i\le N}\), whose first row is
\(q_{0i}=N^{-1/2}\), and define
\[
 H_\ell=\sum_{i=1}^Nq_{\ell i}G_i,\qquad 0\le\ell\le\nu.
\]
For \(u,v\in\mathcal H\),
\[
 \mathbb E\{\langle u,H_\ell\rangle\langle v,H_k\rangle\}
 =\sum_{i=1}^Nq_{\ell i}q_{ki}\langle u,Vv\rangle
 =\mathbf1\{\ell=k\}\langle u,Vv\rangle.
\]
The vector \((H_0,\ldots,H_\nu)\) is jointly Gaussian, so
\(\operatorname{Cov}(H_\ell,H_k)=0\) for \(\ell\ne k\) shows that
\(H_0,\ldots,H_\nu\) are independent; each is centered with covariance
\(V\).

To express \(\widehat r_N\) in terms of the independent residuals
\(H_1,\ldots,H_\nu\), let \(P_m\) be increasing finite-rank orthogonal
projections that converge strongly to the identity
(and take \(P_m=I\) eventually if \(\mathcal H\) is finite-dimensional).
Expanding around the sample mean gives
\[
 \sum_{i=1}^N\|P_m(G_i-\overline G_N)\|^2
 =\sum_{i=1}^N\|P_mG_i\|^2-N\|P_m\overline G_N\|^2.
\]
Because the matrix \((q_{\ell i})\) is orthogonal, applying
\((q_{\ell i})\) separately to
each of the \(m\) coordinates of \(P_mG_i\) preserves the sum of the
coordinate squares.  Hence
\[
 \sum_{i=1}^N\|P_mG_i\|^2
 =\sum_{\ell=0}^{\nu}\|P_mH_\ell\|^2.
\]
The \(\ell=0\) term is
\(\|P_mH_0\|^2=N\|P_m\overline G_N\|^2\), and therefore
\[
 \sum_{i=1}^N\|P_m(G_i-\overline G_N)\|^2
 =\sum_{\ell=1}^{\nu}\|P_mH_\ell\|^2.
\]
As \(m\to\infty\), \(P_mx\to x\) for every \(x\in\mathcal H\), so every
projected squared norm converges to the corresponding squared norm without
\(P_m\).  Dividing the limiting identity
by \(N-1=\nu\) gives
\begin{equation}
\label{eq:gaussian-hilbert-residual-representation}
 \widehat r_N=\frac1\nu\sum_{\ell=1}^{\nu}\|H_\ell\|^2
 \qquad\text{almost surely}.
\end{equation}

If \(V=0\), then
\(\mathbb E\|G_i\|^2=\mathbb E\|H_\ell\|^2=\operatorname{tr}(V)=0\);
hence all \(G_i\) and \(H_\ell\) vanish almost surely, and
\eqref{eq:hilbert-gaussian-mean-tail},
\eqref{eq:hilbert-gaussian-trace-lower-tail}, and
\eqref{eq:hilbert-gaussian-trace-upper-tail} hold with zero random terms.
Otherwise, let
\((\mu_j,e_j)_{1\le j\le J}\), where
\(J\in\mathbb N\cup\{\infty\}\), be the positive eigenpairs of \(V\) in
nonincreasing order.  For \(m\le J\), let \(P_m\) project onto
\(\operatorname{span}\{e_1,\ldots,e_m\}\).  A Gaussian vector has zero
projection onto \(\ker(V)\) almost surely.  Thus, for independent standard
normal variables \(g_{\ell j}\),
\[
 \|P_mH_\ell\|^2
 \stackrel{d}=\sum_{j=1}^m\mu_jg_{\ell j}^2.
\]
Set \(r_m=\sum_{j\le m}\mu_j\),
\(s_m=\sum_{j\le m}\mu_j^2\), and
\(\rho_m=\max_{j\le m}\mu_j\).  To control the projected Gaussian sums
\(\|P_mH_\ell\|^2\),
let \(g_1,\ldots,g_q\) be independent standard normal variables, let
\(w_j\ge0\), and define
\[
 S=\sum_{j=1}^qw_j(g_j^2-1),\qquad
 s_w=\sum_{j=1}^qw_j^2,\qquad
 \rho_w=\max_{j\le q}w_j.
\]
If \(s_w=0\), then \(S=0\) almost surely.  Suppose \(s_w>0\).  For
\(0<\theta<1/(2\rho_w)\),
\begin{align*}
 \log\mathbb E e^{\theta S}
 &=\sum_{j=1}^q\left\{-\theta w_j
 -\frac12\log(1-2\theta w_j)\right\}\\
 &\le\frac{\theta^2s_w}{1-2\theta\rho_w}.
\end{align*}
The inequality follows term by term from
\[
 -\log(1-x)-x=\sum_{k=2}^{\infty}\frac{x^k}{k}
 \le\frac{x^2}{2(1-x)},\qquad 0\le x<1.
\]
For \(x>0\) and \(0<\theta<1/(2\rho_w)\),
\[
 \Pr\{S>x\}
 \le e^{-\theta x}\mathbb E e^{\theta S}
 \le\exp\left\{-\theta x+
 \frac{\theta^2s_w}{1-2\theta\rho_w}\right\}.
\]
Taking
\(\theta=\sqrt t/(\sqrt{s_w}+2\rho_w\sqrt t)\) and
\(x=2\sqrt{s_wt}+2\rho_wt\) gives
\[
 -\theta\{2\sqrt{s_wt}+2\rho_wt\}
 +\frac{\theta^2s_w}{1-2\theta\rho_w}=-t,
\]
and hence
\begin{equation}
\label{eq:finite-weighted-square-upper-tail}
 \Pr\{S>2\sqrt{s_wt}+2\rho_wt\}\le e^{-t}.
\end{equation}
For the other tail,
\[
 \log\mathbb E e^{-\theta S}
 =\sum_{j=1}^q\left\{\theta w_j
 -\frac12\log(1+2\theta w_j)\right\}
 \le\theta^2s_w,
\]
where \(x-\log(1+x)\le x^2/2\) for \(x\ge0\).
The inequality \(x-\log(1+x)\le x^2/2\) follows because the derivative of
\(x^2/2-x+\log(1+x)\) is \(x^2/(1+x)\).  Taking
\(\theta=\sqrt{t/s_w}\) gives
\begin{equation}
\label{eq:finite-weighted-square-lower-tail}
 \Pr\{S<-2\sqrt{s_wt}\}
 \le e^{-2\theta\sqrt{s_wt}}\mathbb E e^{-\theta S}
 \le\exp\{-2\theta\sqrt{s_wt}+\theta^2s_w\}=e^{-t}.
\end{equation}

For \(\|P_mH_0\|^2-r_m\), the weights in
\eqref{eq:finite-weighted-square-upper-tail} are
\(\mu_1,\ldots,\mu_m\), so \(s_w=s_m\) and \(\rho_w=\rho_m\).
For
\(\nu^{-1}\sum_{\ell=1}^{\nu}\|P_mH_\ell\|^2-r_m\), the weights are
\(\mu_j/\nu\), once for each
\((\ell,j)\in\{1,\ldots,\nu\}\times\{1,\ldots,m\}\); hence
\[
 s_w=\frac{s_m}{\nu},\qquad
 \rho_w=\frac{\rho_m}{\nu}.
\]
Equations \eqref{eq:finite-weighted-square-upper-tail} and
\eqref{eq:finite-weighted-square-lower-tail} give
\begin{align*}
 \Pr\{\|P_mH_0\|^2>
 r_m+2\sqrt{s_mt}+2\rho_mt\}&\le e^{-t},\\
 \Pr\left\{\frac1\nu\sum_{\ell=1}^{\nu}\|P_mH_\ell\|^2
 <r_m-2\sqrt{\frac{s_mt}{\nu}}\right\}&\le e^{-t},\\
 \Pr\left\{\frac1\nu\sum_{\ell=1}^{\nu}\|P_mH_\ell\|^2
 >r_m+2\sqrt{\frac{s_mt}{\nu}}+
 \frac{2\rho_mt}{\nu}\right\}&\le e^{-t}.
\end{align*}
If \(J<\infty\), take \(m=J\).  If \(J=\infty\), trace-class covariance gives
\(r_m\uparrow r\), \(s_m\uparrow s\), and \(\rho_m\to\rho\); the projected
squared norms increase almost surely to the unprojected squared norms, and
the projected residual traces increase almost surely to the Hilbert-space
residual traces.  To pass
\eqref{eq:hilbert-gaussian-mean-tail},
\eqref{eq:hilbert-gaussian-trace-lower-tail}, and
\eqref{eq:hilbert-gaussian-trace-upper-tail} to infinite dimension,
consider each bound separately.  For each probability bound, denote the
random statistic by \(Z\), the deterministic threshold by \(a\), and the finite-rank
versions as \(Z_m\) and \(a_m\).  Since \(Z_m\to Z\) and \(a_m\to a\),
\(Z>a\) implies
\(Z_m>a_m\) for all sufficiently large \(m\), and therefore
\[
 \{Z>a\}\subseteq
 \bigcup_{m_0=1}^{\infty}\bigcap_{m\ge m_0}\{Z_m>a_m\}.
\]
The intersections on the right increase as \(m_0\) increases, so
\begin{align*}
 \Pr\{Z>a\}
 &\le\lim_{m_0\to\infty}
 \Pr\left\{\bigcap_{m\ge m_0}\{Z_m>a_m\}\right\}\\
 &\le\lim_{m_0\to\infty}\inf_{m\ge m_0}\Pr\{Z_m>a_m\}
 \le e^{-t}.
\end{align*}
Replacing \(>\) by \(<\) gives the lower-tail bound.
For the sample-mean tail, \(H_0=\sqrt N\,\overline G_N\), so division by
\(N\) gives \eqref{eq:hilbert-gaussian-mean-tail}.  For
\eqref{eq:hilbert-gaussian-trace-lower-tail} and
\eqref{eq:hilbert-gaussian-trace-upper-tail},
\eqref{eq:gaussian-hilbert-residual-representation} identifies \(Z\) with
\(\widehat r_N\).
\end{proof}

Lemma~\ref{lem:conditional-evaluation-trace-sandwich} conditions on \(S_A\),
applies Lemma~\ref{lem:hilbert-subgaussian-quadratic-concentration} to
\(M_A^{-1/2}(X-PX)\), and controls the sample-centered trace
\(\operatorname{tr}\{\widehat\Sigma_{X,B_A}M_A^{-1}\}\).

\Needspace{0.84\textheight}
\begin{lemma}[Conditional bounds for the evaluation-fold mean and trace]
\label{lem:conditional-evaluation-trace-sandwich}
Assume \MainEquation{eq:main-covariance-matched-subgaussian}, whose known
constant is \(\bar\kappa\).  Let
\(\mathcal V_A,r_A,s_A,\rho_A\) be defined by
\MainEquation{eq:conditional-pilot-geometry}, and let \(R\ge0\).  For almost
every realization of the pilot sample \(S_A\) satisfying \(\rho_A\le R\), and
for \(i\in B_A\), where \(|B_A|=N\), define the standardized evaluation
vectors \(U_i\) and the sample mean \(\overline U_{B_A}\) by
\[
 U_i:=M_A^{-1/2}(X_i-PX),
 \qquad
 \overline U_{B_A}:=\frac1N\sum_{i\in B_A}U_i.
\]
Let
\[
 \widehat r_{B_A\mid A}
 :=\operatorname{tr}\{\widehat\Sigma_{X,B_A}M_A^{-1}\}
\]
be the evaluation-fold trace.  There is a constant
\(C_{\rm tr}(\bar\kappa)\),
depending only on \(\bar\kappa\).  For \(C_{\rm tr}(\bar\kappa)\) and every
\(t\ge1\), set
\[
 \varepsilon_N(t):=C_{\rm tr}(\bar\kappa)
 \left(\sqrt{\frac tN}+\frac tN+\frac1N\right).
\]
Conditional on \(S_A\), with probability at least \(1-3e^{-t}\),
\begin{align}
 \|\overline U_{B_A}\|^2
 &\le\frac{C_{\rm tr}(\bar\kappa)}{N}
 \{r_A+2\sqrt{Rr_At}+2Rt\},
 \label{eq:conditional-evaluation-mean-norm}\\
 (1-\varepsilon_N(t))r_A-\frac{C_{\rm tr}(\bar\kappa)Rt}{N}
 &\le \widehat r_{B_A\mid A}
 \le(1+\varepsilon_N(t))r_A,
 \label{eq:conditional-evaluation-trace-sandwich}
\end{align}
hold simultaneously.  The mean bound
\eqref{eq:conditional-evaluation-mean-norm} and the lower bound in
\eqref{eq:conditional-evaluation-trace-sandwich} hold simultaneously with
conditional probability at least \(1-2e^{-t}\).  If
\(\varepsilon_N(t)<1\), rearranging the lower bound in
\eqref{eq:conditional-evaluation-trace-sandwich} gives, with conditional
probability at least \(1-2e^{-t}\),
\begin{equation}
\label{eq:conditional-evaluation-trace-upper-confidence}
 r_A\le
 \frac{\widehat r_{B_A\mid A}+C_{\rm tr}(\bar\kappa)Rt/N}
 {1-\varepsilon_N(t)}.
\end{equation}
If \(\varepsilon_N(t)\le1/2\), then, with conditional probability at least
\(1-3e^{-t}\),
\[
 \frac{\widehat r_{B_A\mid A}+C_{\rm tr}(\bar\kappa)Rt/N}
 {1-\varepsilon_N(t)}
 \le 3r_A+\frac{2C_{\rm tr}(\bar\kappa)Rt}{N}.
\]
\end{lemma}

\begin{proof}
Conditional on \(S_A\), the \(U_i\)'s are independent, centered, have
covariance \(\mathcal V_A\), and inherit covariance-matched sub-Gaussianity.
Since \(s_A\le\rho_Ar_A\le Rr_A\),
Lemma~\ref{lem:hilbert-subgaussian-quadratic-concentration} gives
\eqref{eq:conditional-evaluation-mean-norm} outside an event of probability
at most \(e^{-t}\).  Applying
\eqref{eq:raw-trace-bernstein} conditionally with covariance
\(\mathcal V_A\) gives, outside an event of probability at most \(2e^{-t}\),
\begin{equation}
\label{eq:conditional-raw-trace-sandwich}
 \left|\frac1N\sum_{i\in B_A}\|U_i\|^2-r_A\right|
 \le C_{\rm tr}(\bar\kappa)r_A
 \left(\sqrt{\frac tN}+\frac tN\right).
\end{equation}
The evaluation covariance is sample-centered, so
\begin{equation}
\label{eq:conditional-centered-trace-identity}
 \widehat r_{B_A\mid A}
 =\frac{N}{N-1}\left\{
 \frac1N\sum_{i\in B_A}\|U_i\|^2
 -\|\overline U_{B_A}\|^2
 \right\}.
\end{equation}
On \eqref{eq:conditional-raw-trace-sandwich}, dropping the nonnegative
centering term in \eqref{eq:conditional-centered-trace-identity} gives
\begin{equation}
\label{eq:conditional-evaluation-trace-upper-bound}
 \widehat r_{B_A\mid A}
 \le\frac N{N-1}r_A
 \left\{1+C_{\rm tr}(\bar\kappa)
 \left(\sqrt{\frac tN}+\frac tN\right)\right\},
\end{equation}
which is at most \(\{1+\varepsilon_N(t)\}r_A\) after increasing
\(C_{\rm tr}(\bar\kappa)\).  To prove the lower trace bound, use
\(2\sqrt{Rr_At}\le r_A+Rt\) in
\eqref{eq:conditional-evaluation-mean-norm} to obtain
\[
 \|\overline U_{B_A}\|^2
 \le \frac{C_{\rm tr}(\bar\kappa)}N(2r_A+3Rt),
\]
and combine the evaluation-mean bound with
\[
 \frac1N\sum_{i\in B_A}\|U_i\|^2
 \ge r_A-C_{\rm tr}(\bar\kappa)r_A
 \left(\sqrt{\frac tN}+\frac tN\right),
\]
the lower raw-trace inequality in
\eqref{eq:conditional-raw-trace-sandwich}.  Substituting the evaluation-mean
and raw-trace bounds into \eqref{eq:conditional-centered-trace-identity} and
increasing \(C_{\rm tr}(\bar\kappa)\) gives
\begin{equation}
\label{eq:conditional-evaluation-trace-lower-bound}
 \widehat r_{B_A\mid A}
 \ge
 \{1-\varepsilon_N(t)\}r_A
 -\frac{C_{\rm tr}(\bar\kappa)Rt}{N}
\end{equation}
whenever
\(\{1-\varepsilon_N(t)\}r_A-C_{\rm tr}(\bar\kappa)Rt/N\ge0\).  If
\(\{1-\varepsilon_N(t)\}r_A-C_{\rm tr}(\bar\kappa)Rt/N<0\),
\(\widehat r_{B_A\mid A}\ge0\) also proves
\eqref{eq:conditional-evaluation-trace-lower-bound}.  Equations
\eqref{eq:conditional-evaluation-trace-upper-bound} and
\eqref{eq:conditional-evaluation-trace-lower-bound} prove
\eqref{eq:conditional-evaluation-trace-sandwich}.

The one-sided derivation of
\eqref{eq:conditional-raw-trace-sandwich} in the proof of
Lemma~\ref{lem:hilbert-subgaussian-quadratic-concentration} gives the lower
raw-trace inequality in \eqref{eq:conditional-raw-trace-sandwich} with
failure probability at most \(e^{-t}\).  Combining the lower raw-trace event
with the event in \eqref{eq:conditional-evaluation-mean-norm} shows that
\eqref{eq:conditional-evaluation-mean-norm} and
\eqref{eq:conditional-evaluation-trace-lower-bound} hold simultaneously
with failure probability at most \(2e^{-t}\).  When
\(\varepsilon_N(t)<1\), moving the \(r_A\)-term in
\eqref{eq:conditional-evaluation-trace-lower-bound} to the left and dividing
by \(1-\varepsilon_N(t)\) gives
\eqref{eq:conditional-evaluation-trace-upper-confidence}.  If
\(\varepsilon_N(t)\le1/2\), the upper bound in
\eqref{eq:conditional-evaluation-trace-sandwich} gives
\[
 \frac{\widehat r_{B_A\mid A}+C_{\rm tr}(\bar\kappa)Rt/N}
 {1-\varepsilon_N(t)}
 \le
 \frac{1+\varepsilon_N(t)}{1-\varepsilon_N(t)}r_A
 +\frac{C_{\rm tr}(\bar\kappa)Rt}
 {N\{1-\varepsilon_N(t)\}}
 \le3r_A+\frac{2C_{\rm tr}(\bar\kappa)Rt}{N}.
\]
\end{proof}

For balanced folds, Lemma~\ref{lem:balanced-calibration-cross-term} reduces
\(\|PX-P_nX\|_{M_A^{-1}}^2-\widehat T_{X,A,\lambda}\) to the scalar cross term
controlled in
\MainResult{Theorem}{thm:balanced-one-split-linear-confidence}.

\begin{lemma}[Balanced cross-term identity]
\label{lem:balanced-calibration-cross-term}
Let \(n=2N\) and \(|A|=|B_A|=N\).  With
\[
 b_A=M_A^{-1/2}(PX-\overline X_A),
 \qquad
 \xi_A=M_A^{-1/2}(\overline X_{B_A}-PX),
\]
one has
\begin{align*}
 M_A^{-1/2}(PX-P_nX)&=\tfrac12(b_A-\xi_A),\\
 M_A^{-1/2}(\overline X_{B_A}-\overline X_A)&=b_A+\xi_A,
\end{align*}
and therefore
\begin{equation}
\label{eq:supp-balanced-calibration-cross-term}
 \|PX-P_nX\|_{M_A^{-1}}^2-\widehat T_{X,A,\lambda}
 =-\langle b_A,\xi_A\rangle.
\end{equation}
\end{lemma}

\begin{proof}
The two vector identities follow from
\(P_nX=(\overline X_A+\overline X_{B_A})/2\).  For balanced folds, the trace
coefficient in \MainEquation{eq:crossfit-single-split-calibration} is zero,
so \(\widehat T_{X,A,\lambda}=\|b_A+\xi_A\|^2/4\).  Subtracting
\(\widehat T_{X,A,\lambda}\) from
\(\|b_A-\xi_A\|^2/4\) gives
\[
 \frac14\{\|b_A-\xi_A\|^2-\|b_A+\xi_A\|^2\}
 =\frac14\{-4\langle b_A,\xi_A\rangle\}
 =-\langle b_A,\xi_A\rangle,
\]
which is
\eqref{eq:supp-balanced-calibration-cross-term}.
\end{proof}

The finite-split width bound in
Corollary~\ref{supp:cor:balanced-finite-split-oracle-width} requires a
comparison of \(M_A\) with \(M_\lambda\), together with bounds on \(r_A\) and
\(s_A\) in terms of \(d_{X,\lambda}\).
Lemma~\ref{lem:regularized-pilot-geometry} obtains both from a relative
covariance bound in the \(M_\lambda\) metric.

\begin{lemma}[Regularized pilot geometry]
\label{lem:regularized-pilot-geometry}
Suppose
\begin{equation}
\label{eq:pilot-relative-covariance-event}
 \left\|M_\lambda^{-1/2}
 (\widehat\Sigma_{X,A}-\Sigma_X)M_\lambda^{-1/2}
 \right\|_{\rm op}\le\frac12.
\end{equation}
Then
\begin{equation}
\label{eq:supp-pilot-geometry-consequences}
 \tfrac12M_\lambda\preceq M_A\preceq\tfrac32M_\lambda,\qquad
 \rho_A\le2,\qquad r_A\le2d_{X,\lambda},\qquad
 s_A\le4d_{X,\lambda}.
\end{equation}
\end{lemma}

\begin{proof}
Writing \(E_A=M_\lambda^{-1/2}
(\widehat\Sigma_{X,A}-\Sigma_X)M_\lambda^{-1/2}\), one has
\(M_A=M_\lambda^{1/2}(I+E_A)M_\lambda^{1/2}\).  On
\eqref{eq:pilot-relative-covariance-event},
\[
 \tfrac12 I\preceq I+E_A\preceq\tfrac32I.
\]
Conjugating by \(M_\lambda^{1/2}\) gives the two-sided metric comparison in
\eqref{eq:supp-pilot-geometry-consequences}, and the lower metric bound
\(\tfrac12M_\lambda\preceq M_A\) gives
\(M_A^{-1}\preceq2M_\lambda^{-1}\).
For any orthonormal basis \((e_j)_{j\ge1}\), the trace formula for positive
trace-class operators and
\(M_A^{-1}\preceq2M_\lambda^{-1}\) give, term by term,
\[
 \begin{aligned}
 r_A
 &=\sum_{j\ge1}\langle\Sigma_X^{1/2}e_j,
 M_A^{-1}\Sigma_X^{1/2}e_j\rangle\\
 &\le2\sum_{j\ge1}\langle\Sigma_X^{1/2}e_j,
 M_\lambda^{-1}\Sigma_X^{1/2}e_j\rangle
 =2d_{X,\lambda}.
 \end{aligned}
\]
Likewise, since \(\Sigma_X\preceq M_\lambda\preceq2M_A\),
\(\rho_A\le2\), and then
\(s_A\le\rho_Ar_A\le4d_{X,\lambda}\).
\end{proof}

\begin{proof}[Proof of
\MainResult{Theorem}{thm:balanced-one-split-linear-confidence}]
Condition on a pilot realization with \(\rho_A\le R\).  If \(R=0\), then
\(\rho_A=0\), so the positive operator \(\mathcal V_A\) is zero.  Since
\(M_A\) is invertible, \(\mathcal V_A=0\) implies \(\Sigma_X=0\), and
therefore \(X=PX\)
almost surely.  Consequently
\[
 r_A=\widehat r_{B_A\mid A}
 =\widehat T_{X,A,\lambda}
 =\|PX-P_nX\|_{M_A^{-1}}^2=0,
\]
so the theorem holds with
\(\overline r_{B_A\mid A}(t;0)=
\widehat\Delta_{X,A,\lambda}(t;0)=0\).
It remains to consider \(R>0\).

The scalar
\(-\langle b_A,\xi_A\rangle\) is centered sub-Gaussian and satisfies, outside
an event of conditional probability at most \(e^{-t}\),
\begin{equation}
\label{eq:conditional-cross-term-bound}
 -\langle b_A,\xi_A\rangle
 \le C_{\rm cr}(\bar\kappa)\sqrt{\frac{Rt}{N}}\,\|b_A\|.
\end{equation}
Indeed, conditional on the pilot sample \(S_A\), Proposition~2.6.1 of
\cite{vershynin2025high} and independence of the evaluation observations
give, for every \(\theta>0\),
\[
 \mathbb E_{B_A}\exp\{-\theta\langle b_A,\xi_A\rangle\}
 \le
 \exp\left\{\frac{C(\bar\kappa)\theta^2}{N}
 \langle b_A,\mathcal V_Ab_A\rangle\right\}
 \le
 \exp\left\{\frac{C(\bar\kappa)\theta^2R}{N}\|b_A\|^2\right\}.
\]
Consequently, for \(x>0\),
\[
 \Pr_{B_A}\{-\langle b_A,\xi_A\rangle>x\}
 \le
 \exp\left\{-\theta x+
 \frac{C(\bar\kappa)\theta^2R}{N}\|b_A\|^2\right\}.
\]
If \(b_A=0\), the random variable in
\eqref{eq:conditional-cross-term-bound} is zero.  If \(b_A\ne0\), the
exponent is minimized at
\[
 \theta_\star
 =\frac{xN}{2C(\bar\kappa)R\|b_A\|^2},
\]
and substitution gives
\[
 \Pr_{B_A}\{-\langle b_A,\xi_A\rangle>x\}
 \le\exp\left\{-\frac{x^2N}
 {4C(\bar\kappa)R\|b_A\|^2}\right\}.
\]
Taking
\(x=2\sqrt{C(\bar\kappa)Rt/N}\,\|b_A\|\) proves
\eqref{eq:conditional-cross-term-bound} with
\(2\sqrt{C(\bar\kappa)}\) in place of
\(C_{\rm cr}(\bar\kappa)\).
Choose \(C_{\rm cr}(\bar\kappa)\) so that
\[
 C_{\rm cr}(\bar\kappa)
 \ge2\sqrt{C(\bar\kappa)}
 \max\{2,\sqrt{2C_{\rm tr}(\bar\kappa)}\}.
\]
The observable calibration statistic also gives
\begin{equation}
\label{eq:observable-pilot-mean-bound}
 \|b_A\|\le\|b_A+\xi_A\|+\|\xi_A\|
 =2\sqrt{\widehat T_{X,A,\lambda}}+\|\xi_A\|.
\end{equation}
Lemma~\ref{lem:conditional-evaluation-trace-sandwich} gives
\eqref{eq:conditional-evaluation-mean-norm} and
\eqref{eq:conditional-evaluation-trace-lower-bound} simultaneously outside
an event of conditional probability at most \(2e^{-t}\).  Intersecting the
events in \eqref{eq:conditional-evaluation-mean-norm},
\eqref{eq:conditional-evaluation-trace-lower-bound}, and
\eqref{eq:conditional-cross-term-bound} gives conditional failure probability
at most \(3e^{-t}\).  On the intersection of the events in
\eqref{eq:conditional-evaluation-mean-norm},
\eqref{eq:conditional-evaluation-trace-lower-bound}, and
\eqref{eq:conditional-cross-term-bound}, replace
\(r_A\) in \eqref{eq:conditional-evaluation-mean-norm} by
\(\overline r_{B_A\mid A}(t;R)\), giving
\[
 \|\xi_A\|^2
 \le \frac{C_{\rm tr}(\bar\kappa)}N
 \{\overline r_{B_A\mid A}(t;R)
 +2\sqrt{R\overline r_{B_A\mid A}(t;R)t}+2Rt\}.
\]
The tail calculation with
coefficient \(2\sqrt{C(\bar\kappa)}\), together with
\eqref{eq:observable-pilot-mean-bound} and the explicit inequality for
\(\|\xi_A\|^2\), gives
\begin{align*}
 -\langle b_A,\xi_A\rangle
 &\le2\sqrt{C(\bar\kappa)}\bigg[
 2\sqrt{\frac{Rt\widehat T_{X,A,\lambda}}N}\\
 &\qquad+\frac{\sqrt{C_{\rm tr}(\bar\kappa)}}N
 \sqrt{Rt\{\overline r_{B_A\mid A}(t;R)
 +2\sqrt{R\overline r_{B_A\mid A}(t;R)t}+2Rt\}}
 \bigg].
\end{align*}
For \(x,y\ge0\),
\(x^2+2xy+2y^2\le2(x+y)^2\).  Applying
\(x^2+2xy+2y^2\le2(x+y)^2\) with
\(x=\sqrt{\overline r_{B_A\mid A}(t;R)}\) and \(y=\sqrt{Rt}\) gives
\begin{align}
 -\langle b_A,\xi_A\rangle
 &\le2\sqrt{C(\bar\kappa)}\bigg[
 2\sqrt{\frac{Rt\widehat T_{X,A,\lambda}}N}\notag\\
 &\qquad+\frac{\sqrt{2C_{\rm tr}(\bar\kappa)}}N
 \{\sqrt{R\overline r_{B_A\mid A}(t;R)t}+Rt\}\bigg]\notag\\
 &\le C_{\rm cr}(\bar\kappa)\bigg[
 \sqrt{\frac{Rt\widehat T_{X,A,\lambda}}N}
 +\frac1N
 \{\sqrt{R\overline r_{B_A\mid A}(t;R)t}+Rt\}\bigg].
\label{eq:conditional-cross-term-observable-bound}
\end{align}
Combining \eqref{eq:conditional-cross-term-observable-bound} with
Lemma~\ref{lem:balanced-calibration-cross-term} proves the one-split bound.
\end{proof}

\begin{proof}[Proof of
\MainResult{Corollary}{cor:gaussian-one-split-linear-confidence}]
Put \(R=R_0=L/\lambda\) and \(t=t_\alpha^{\rm G}\).  Since
\(M_A\succeq\lambda I\),
\(\rho_A\le L/\lambda=R\) for every pilot realization.  If \(R=0\), then
\(\Sigma_X=0\), so the random quantities and deterministic bounds in the two
inequalities of
\MainResult{Corollary}{cor:gaussian-one-split-linear-confidence} reduce to
zero almost surely.  We henceforth suppose \(R>0\), condition on the pilot realization,
and put
\[
 U_i:=M_A^{-1/2}(X_i-PX),\qquad i\in B_A.
\]
The evaluation variables are conditionally independent Gaussian random
elements with covariance \(\mathcal V_A\); the evaluation-sample covariance
has trace \(\widehat r_{B_A\mid A}\), and the evaluation-sample mean is
\(\xi_A\).  The lower-tail
part of Lemma~\ref{lem:hilbert-gaussian-mean-trace-concentration} gives,
outside an event of conditional probability at most \(e^{-t}\),
\begin{equation}
\label{eq:gaussian-conditional-trace-lower-event}
 \widehat r_{B_A\mid A}
 \ge r_A-2\sqrt{\frac{s_At}{N-1}}
 \ge r_A-2\sqrt{\frac{Rr_At}{N-1}}.
\end{equation}
With \(q=\sqrt{r_A}\) and
\(a=\sqrt{Rt/(N-1)}\),
\eqref{eq:gaussian-conditional-trace-lower-event} becomes
\(q^2-2aq-\widehat r_{B_A\mid A}\le0\).  The nonnegative root of
\(q^2-2aq-\widehat r_{B_A\mid A}=0\) therefore gives
\[
 r_A\le
 \left\{a+\sqrt{a^2+\widehat r_{B_A\mid A}}\right\}^2
 =\overline r_{B_A\mid A}^{\rm G}(t;R),
\]
which is \MainEquation{eq:gaussian-conditional-trace-ucb}.

Equation~\eqref{eq:hilbert-gaussian-mean-tail}, applied conditionally with
covariance \(\mathcal V_A\), gives, outside an event of conditional failure
probability at most \(e^{-t}\),
\begin{equation}
\label{eq:gaussian-conditional-evaluation-mean-event}
 \|\xi_A\|^2
 \le\frac{r_A+2\sqrt{s_At}+2\rho_At}{N}.
\end{equation}
On the intersection of \eqref{eq:gaussian-conditional-evaluation-mean-event}
and \eqref{eq:gaussian-conditional-trace-lower-event},
\[
 \|\xi_A\|^2
 \le\frac{\overline r_{B_A\mid A}^{\rm G}(t;R)
 +2\sqrt{R\overline r_{B_A\mid A}^{\rm G}(t;R)t}+2Rt}{N},
\]
because
\(r_A\le\overline r_{B_A\mid A}^{\rm G}(t;R)\) from
\eqref{eq:gaussian-conditional-trace-lower-event},
\(s_A\le\rho_Ar_A\), and \(\rho_A\le R\); the function
\(r\mapsto r+2\sqrt{Rrt}+2Rt\) is increasing for \(r\ge0\).
Conditional on the pilot,
\(-\langle b_A,\xi_A\rangle\) is centered normal with variance at most
\(R\|b_A\|^2/N\).  For every \(\theta>0\),
\[
 \mathbb E_{B_A}\exp\{-\theta\langle b_A,\xi_A\rangle\}
 =\exp\left\{\frac{\theta^2}{2}
 \operatorname{Var}_{B_A}(\langle b_A,\xi_A\rangle)\right\}
 \le\exp\left\{\frac{\theta^2R\|b_A\|^2}{2N}\right\}.
\]
If \(b_A\ne0\), then, for \(x>0\) and \(\theta>0\),
\[
 \Pr_{B_A}\{-\langle b_A,\xi_A\rangle>x\}
 \le\exp\left\{-\theta x+\frac{\theta^2R\|b_A\|^2}{2N}\right\}.
\]
Choosing \(\theta=xN/(R\|b_A\|^2)\) gives
\[
 \Pr_{B_A}\{-\langle b_A,\xi_A\rangle>x\}
 \le\exp\left\{-\frac{x^2N}{2R\|b_A\|^2}\right\}.
\]
Taking \(x=\sqrt{2Rt/N}\,\|b_A\|\) gives the event
\begin{equation}
\label{eq:gaussian-conditional-cross-event}
 -\langle b_A,\xi_A\rangle
 \le\sqrt{\frac{2Rt}{N}}\,\|b_A\|.
\end{equation}
The conditional failure probability of
\eqref{eq:gaussian-conditional-cross-event} is at most \(e^{-t}\); when
\(b_A=0\),
the inequality holds identically.
On the intersection of
\eqref{eq:gaussian-conditional-trace-lower-event},
\eqref{eq:gaussian-conditional-evaluation-mean-event}, and
\eqref{eq:gaussian-conditional-cross-event}, use
\[
 \|b_A\|\le2\sqrt{\widehat T_{X,A,\lambda}}+\|\xi_A\|,
 \qquad
 \|\xi_A\|^2
 \le\frac{\overline r_{B_A\mid A}^{\rm G}(t;R)
 +2\sqrt{R\overline r_{B_A\mid A}^{\rm G}(t;R)t}+2Rt}{N},
\]
in \eqref{eq:gaussian-conditional-cross-event}.  For \(x,y\ge0\), the inequality
\[
 \sqrt{2y^2(x^2+2xy+2y^2)}\le2y(x+y),
\]
with \(x=\sqrt{\overline r_{B_A\mid A}^{\rm G}(t;R)}\) and
\(y=\sqrt{Rt}\) then gives
\[
 -\langle b_A,\xi_A\rangle
 \le2\sqrt{\frac{2Rt\widehat T_{X,A,\lambda}}N}
 +\frac{2\{\sqrt{R\overline r_{B_A\mid A}^{\rm G}(t;R)t}+Rt\}}N.
\]
Thus
\(-\langle b_A,\xi_A\rangle
\le\widehat\Delta_{X,A,\lambda}^{\rm G}(t;R)\), where
\(\widehat\Delta_{X,A,\lambda}^{\rm G}(t;R)\) is defined in
\MainEquation{eq:gaussian-quadratic-confidence-inflation}.
Lemma~\ref{lem:balanced-calibration-cross-term} and the union bound over the
trace, evaluation-mean, and scalar cross-term events prove the result, with
conditional failure probability at most \(3e^{-t}\).
\end{proof}

The bounded-feature specialization uses an elementary exponential bound for
a scalar random variable \(W\) supported on an interval
of length \(c>0\).  After translating the
interval to \([0,c]\), put \(\mu=\mathbb EW/c\).  For \(s>0\), convexity of
\(x\mapsto e^{-sx}\) gives
\begin{equation}
\label{eq:bounded-centered-exponential-moment}
 \mathbb E e^{s(\mathbb EW-W)}
 \le e^{sc\mu}\{1-\mu+\mu e^{-sc}\}
 \le e^{s^2c^2/8}.
\end{equation}
To prove
\(e^{sc\mu}\{1-\mu+\mu e^{-sc}\}\le e^{s^2c^2/8}\), the logarithm of
\(e^{sc\mu}\{1-\mu+\mu e^{-sc}\}\), as a function of \(y=sc\), has value
and derivative zero at
\(y=0\), while the second derivative with respect to \(y\) is
\(p_y(1-p_y)\le1/4\), where
\(p_y=\mu e^{-y}/(1-\mu+\mu e^{-y})\).  The case \(c=0\) is immediate.

\begin{proof}[Proof of
\MainResult{Corollary}{cor:bounded-feature-one-split-linear-confidence}]
After conditioning on the pilot sample \(S_A\), the scalar variable
\(\langle M_A^{-1/2}b_A,X\rangle\) ranges over an interval of length at most
\(2B\|M_A^{-1/2}b_A\|\).  Set
\[
 Y_i=\langle M_A^{-1/2}b_A,X_i-PX\rangle,\qquad i\in B_A,
\]
apply \eqref{eq:bounded-centered-exponential-moment} independently to the
\(Y_i\)'s with \(s=\theta/N\) and interval length
\(2B\|M_A^{-1/2}b_A\|\).  Since \(\mathbb EY_i=0\), independence of the
\(Y_i\)'s gives
\[
 \mathbb E_{B_A}\exp\left\{-\frac{\theta}{N}
 \sum_{i\in B_A}Y_i\right\}
 \le
 \exp\left\{\frac{\theta^2B^2
 \|M_A^{-1/2}b_A\|^2}{2N}\right\}.
\]
Put \(L_A=B\|M_A^{-1/2}b_A\|\).  For \(x>0\) and \(\theta>0\),
\[
 \Pr_{B_A}\{-\langle b_A,\xi_A\rangle>x\}
 \le
 \exp\left\{-\theta x+\frac{\theta^2L_A^2}{2N}\right\}.
\]
If \(L_A>0\), take
\[
 x=L_A\sqrt{\frac{2t}{N}},
 \qquad
 \theta=\frac{Nx}{L_A^2};
\]
the exponent is then \(-Nx^2/(2L_A^2)=-t\).  If \(L_A=0\), then
\(\langle b_A,\xi_A\rangle=0\), so
\[
 \Pr_{B_A}\left\{-\langle b_A,\xi_A\rangle
 >L_A\sqrt{\frac{2t}{N}}\right\}=0.
\]
Thus, outside an event of probability at most \(e^{-t}\),
\begin{equation}
\label{eq:bounded-feature-conditional-cross-event}
 -\langle b_A,\xi_A\rangle
 =-\langle M_A^{-1/2}b_A,\overline X_{B_A}-PX\rangle
 \le B\|M_A^{-1/2}b_A\|\sqrt{\frac{2t}{N}}
 \le B\sqrt{\frac{2t}{\lambda N}}\,\|b_A\|.
\end{equation}
For the independent evaluation fold, the variance identity and Jensen's
inequality imply
\[
 \mathbb E_{B_A}\|\overline X_{B_A}-PX\|
 \le\left\{\frac{\operatorname{tr}(\Sigma_X)}N\right\}^{1/2}
 \le\frac B{\sqrt N}.
\]
Changing one evaluation observation changes
\(\|\overline X_{B_A}-PX\|\) by at most \(2B/N\).  Put
\(F=\|\overline X_{B_A}-PX\|\), relabel the evaluation observations as
\(X_1,\ldots,X_N\), and let
\[
 D_i=\mathbb E(F\mid X_1,\ldots,X_i)
 -\mathbb E(F\mid X_1,\ldots,X_{i-1})
\]
be the successive conditional differences.  Conditional on the first
\(i-1\) observations, the possible values of \(D_i\) lie in an interval of
length at most \(2B/N\).  Applying
\eqref{eq:bounded-centered-exponential-moment} conditionally to each
difference and iterating expectations gives
\[
 \mathbb E_{B_A}\exp\left[
 \theta\{F-\mathbb E_{B_A}F\}\right]
 \le\exp\left\{\frac{\theta^2}{8}
 N\left(\frac{2B}{N}\right)^2\right\}
 =\exp\left(\frac{\theta^2B^2}{2N}\right).
\]
If \(B>0\), then, for \(x>0\) and \(\theta>0\),
\begin{equation}
\label{eq:bounded-feature-evaluation-mean-tail}
 \Pr_{B_A}\{F-\mathbb E_{B_A}F>x\}
 \le\exp\left\{-\theta x+\frac{\theta^2B^2}{2N}\right\}.
\end{equation}
Taking
\[
 x=B\sqrt{\frac{2t}{N}},
 \qquad
 \theta=\frac{Nx}{B^2}
\]
makes the exponent \(-Nx^2/(2B^2)=-t\).  When \(B=0\), one has \(F=0\)
almost surely.  Equation \eqref{eq:bounded-feature-evaluation-mean-tail} and
\(\mathbb E_{B_A}F\le B/\sqrt N\) show that the event
\begin{equation}
\label{eq:bounded-feature-evaluation-mean-event}
 \|\xi_A\|
 \le\lambda^{-1/2}\|\overline X_{B_A}-PX\|
 \le\frac{B(1+\sqrt{2t})}{\sqrt{\lambda N}}.
\end{equation}
has failure probability at most \(e^{-t}\).  On the intersection of
\eqref{eq:bounded-feature-conditional-cross-event} and
\eqref{eq:bounded-feature-evaluation-mean-event},
\eqref{eq:observable-pilot-mean-bound} yields
\[
 -\langle b_A,\xi_A\rangle
 \le2B\sqrt{\frac{2t\widehat T_{X,A,\lambda}}{\lambda N}}
 +\frac{B^2\{\sqrt{2t}+2t\}}{\lambda N},
\]
which is \MainEquation{eq:bounded-feature-quadratic-confidence-inflation}.
Lemma~\ref{lem:balanced-calibration-cross-term} proves the one-split
statement with failure probability at most \(2e^{-t}\).
\end{proof}

\subsection{Regularized covariance concentration}

Assume that the covariance-matched sub-Gaussian condition in
\MainEquation{eq:main-covariance-matched-subgaussian} holds with the known
constant \(\bar\kappa\); equivalently,
\begin{equation}
\label{eq:covariance-matched-subgaussian}
 \|\langle u,X-PX\rangle\|_{\psi_2}
 \le\bar\kappa\langle u,\Sigma_Xu\rangle^{1/2},
 \qquad u\in\mathcal H,
\end{equation}
where the norm is defined in \eqref{eq:supp-psi-two-norm}.

Lemma~\ref{lem:regularized-covariance-concentration} bounds the pilot
covariance error in the \(M_\lambda\) metric.  When
\(\vartheta_{N,\lambda}(t\vee1)\le1/2\), the conclusion of
Lemma~\ref{lem:regularized-covariance-concentration} verifies
\eqref{eq:pilot-relative-covariance-event}.

\begin{lemma}[Regularized covariance concentration]
\label{lem:regularized-covariance-concentration}
Assume \eqref{eq:covariance-matched-subgaussian}.  Let
\(\vartheta_{N,\lambda}\) be the rate in
\MainEquation{eq:main-pilot-stability-rate}.  There exists
\(C_{\rm cov}(\bar\kappa)\ge1\), depending only on \(\bar\kappa\), such that,
for every \(\lambda>0\),
integer \(N\ge2\),
deterministic \(\mathcal I\subset[n]\) with
\(|\mathcal I|=N\) and every \(t\ge0\),
\begin{equation}
\label{eq:regularized-covariance-concentration}
 \Pr\left\{
 \left\|
 M_\lambda^{-1/2}
 (\widehat\Sigma_{X,\mathcal I}-\Sigma_X)
 M_\lambda^{-1/2}
 \right\|_{\rm op}
 >\vartheta_{N,\lambda}(t\vee1)
 \right\}
 \le e^{-t}.
\end{equation}
\end{lemma}

\begin{proof}
Define the centered \(\mathcal H\)-valued random element \(U_\lambda\) and
the covariance operator \(T_\lambda:\mathcal H\to\mathcal H\) of
\(U_\lambda\) by
\[
 U_\lambda:=M_\lambda^{-1/2}(X-PX),
 \qquad
 T_\lambda:=\mathbb E(U_\lambda\otimes U_\lambda)
 =M_\lambda^{-1/2}\Sigma_XM_\lambda^{-1/2}.
\]
Then \(T_\lambda\) is nonnegative and self-adjoint,
\(\|T_\lambda\|_{\rm op}\le1\), and
\[
 \operatorname{tr}(T_\lambda)=d_{X,\lambda}
 =\mathbb E\|U_\lambda\|^2
 \le\lambda^{-1}\mathbb E\|X-PX\|^2<\infty.
\]
The trace identity shows that the nonnegative eigenvalues of \(T_\lambda\)
are summable, while
\eqref{eq:covariance-matched-subgaussian} gives
\begin{equation}
\label{eq:regularized-subgaussian-coordinate-bound}
 \|\langle u,U_\lambda\rangle\|_{\psi_2}
 \le\bar\kappa\langle u,T_\lambda u\rangle^{1/2},
 \qquad u\in\mathcal H.
\end{equation}
If \(T_\lambda=0\), then
\(\mathbb E\|U_\lambda\|^2=0\).  Since \(M_\lambda^{-1/2}\) is
injective, \(X-PX=0\) almost surely, so both covariances vanish and the
claim is immediate.  We henceforth suppose that \(T_\lambda\ne0\).  Let
\((\tau_j,e_j)_{1\le j\le J}\), with
\(J\in\mathbb N\cup\{\infty\}\), be the positive eigenpairs of
\(T_\lambda\).  Let
\((g_j)_{1\le j\le J}\) be independent \(N(0,1)\) variables.  If
\(J<\infty\), set
\[
 \Gamma_\lambda:=\sum_{j=1}^J\sqrt{\tau_j}\,g_je_j.
\]
If \(J=\infty\), the partial sums
\[
 \Gamma_{\lambda,D}:=\sum_{j=1}^D\sqrt{\tau_j}\,g_je_j
\]
satisfy
\[
 \mathbb E\|\Gamma_{\lambda,D'}-\Gamma_{\lambda,D}\|^2
 =\sum_{j=D+1}^{D'}\tau_j\longrightarrow0
 \qquad(D,D'\to\infty),
\]
because \(\sum_j\tau_j=\operatorname{tr}(T_\lambda)<\infty\); define
\(\Gamma_\lambda\) as the \(L^2(\mathcal H)\) limit of
\((\Gamma_{\lambda,D})_{D\ge1}\).  For every
\(u\in\mathcal H\), the scalar projections converge in \(L^2\), and the
centered normal variances converge to
\(\sum_j\tau_j\langle u,e_j\rangle^2=\langle u,T_\lambda u\rangle\).
Thus every scalar projection of \(\Gamma_\lambda\) is centered normal, and
the covariance form of \(\Gamma_\lambda\) is \(T_\lambda\).  In either case,
\(\Gamma_\lambda\) is therefore centered Gaussian with covariance
\(T_\lambda\).
The bound \(\mathbb E\|U_\lambda\|^2<\infty\) and
\eqref{eq:regularized-subgaussian-coordinate-bound} verify the hypotheses of
\cite[Theorem~9]{koltchinskii2017concentration}; the constant in Theorem~9 is
controlled by \(\bar\kappa\).

For each \(i\in\mathcal I\), define the standardized observation
\[
 U_{\lambda,i}:=M_\lambda^{-1/2}(X_i-PX)\in\mathcal H.
\]
Let \(\widetilde T_{\lambda,\mathcal I}:\mathcal H\to\mathcal H\) be the
uncentered empirical second-moment operator of
\((U_{\lambda,i})_{i\in\mathcal I}\):
\[
 \widetilde T_{\lambda,\mathcal I}
 :=\frac1N\sum_{i\in\mathcal I}U_{\lambda,i}^{\otimes2}.
\]
If \(\Gamma_\lambda\) is centered Gaussian with covariance
\(T_\lambda\), the effective rank used in
\cite[Theorem~9]{koltchinskii2017concentration} is
\[
 r_{\rm KL}(T_\lambda)
 :=\frac{(\mathbb E\|\Gamma_\lambda\|)^2}
 {\|T_\lambda\|_{\rm op}}.
\]
For every \(s\ge1\), Theorem~9 of
\cite{koltchinskii2017concentration} gives, with probability at least
\(1-e^{-s}\),
\begin{align}
 \|\widetilde T_{\lambda,\mathcal I}-T_\lambda\|_{\rm op}
 \le C_{\bar\kappa}\|T_\lambda\|_{\rm op}
 \bigg[
 &\sqrt{\frac{r_{\rm KL}(T_\lambda)}N}
 \vee\frac{r_{\rm KL}(T_\lambda)}N\\
 &\vee\sqrt{\frac sN}\vee\frac sN
 \bigg].
\label{eq:koltchinskii-regularized-second-moment-bound}
\end{align}
Using \(s=t\vee1\) also gives a failure probability at most
\(e^{-t}\) for every \(t\ge0\).
To convert \eqref{eq:koltchinskii-regularized-second-moment-bound} to
\(d_{X,\lambda}\), we control the products containing
\(\|T_\lambda\|_{\rm op}\); the effective rank itself may exceed
\(d_{X,\lambda}\).  Since
\((\mathbb E\|\Gamma_\lambda\|)^2
 \le\operatorname{tr}(T_\lambda)=d_{X,\lambda}\),
\begin{align*}
 \|T_\lambda\|_{\rm op}
 \sqrt{\frac{r_{\rm KL}(T_\lambda)}N}
 &\le\sqrt{\frac{d_{X,\lambda}}N},\\
 \|T_\lambda\|_{\rm op}
 \frac{r_{\rm KL}(T_\lambda)}N
 &\le\frac{d_{X,\lambda}}N.
\end{align*}
The two tail terms satisfy
\[
 \|T_\lambda\|_{\rm op}\sqrt{\frac sN}\le\sqrt{\frac sN},
 \qquad
 \|T_\lambda\|_{\rm op}\frac sN\le\frac sN.
\]
Thus
\begin{equation}
\label{eq:regularized-raw-second-moment-bound}
 \|\widetilde T_{\lambda,\mathcal I}-T_\lambda\|_{\rm op}
 \le C_{\bar\kappa}
 \left[
 \sqrt{\frac{d_{X,\lambda}+s}{N}}
 +\frac{d_{X,\lambda}+s}{N}
 \right].
\end{equation}

Let \(\overline U_{\lambda,\mathcal I}:=N^{-1}
\sum_{i\in\mathcal I}U_{\lambda,i}\).  The unbiased, sample-centered
covariance satisfies
\begin{align}
 &M_\lambda^{-1/2}\widehat\Sigma_{X,\mathcal I}
 M_\lambda^{-1/2}-T_\lambda\\
 &\quad=
 \frac N{N-1}(\widetilde T_{\lambda,\mathcal I}-T_\lambda)
 -\frac N{N-1}\overline U_{\lambda,\mathcal I}^{\otimes2}
 +\frac1{N-1}T_\lambda.
\label{eq:regularized-sample-centering-identity}
\end{align}
For the sample mean, apply the calculation in
\eqref{eq:hilbert-sample-mean-exponential-bound} to \(U_\lambda\), whose
covariance is \(T_\lambda\) and which satisfies
\eqref{eq:regularized-subgaussian-coordinate-bound}.  For a constant
\(c_{\bar\kappa}\) depending only on \(\bar\kappa\), the calculation gives
\begin{equation}
\label{eq:regularized-sample-mean-exponential-bound}
 \mathbb E\exp\langle u,\overline U_{\lambda,\mathcal I}\rangle
 \le
 \exp\left\{\frac{c_{\bar\kappa}}{2N}\langle u,T_\lambda u\rangle\right\},
 \qquad u\in\mathcal H.
\end{equation}
To apply \cite[Theorem~1]{hsu2012tail}, we verify the finite-dimensional
exponential-moment hypothesis of Theorem~1 from
\eqref{eq:regularized-sample-mean-exponential-bound}.  For \(1\le D\le J\),
define the vector \(x_D\in\mathbb R^D\) and diagonal linear map
\(A_D:\mathbb R^D\to\mathbb R^D\) by
\[
 x_D:=\left(
 \sqrt{\frac{N}{c_{\bar\kappa}\tau_j}}
 \langle\overline U_{\lambda,\mathcal I},e_j\rangle
 \right)_{j=1}^D,
 \qquad
 A_D:=\sqrt{\frac{c_{\bar\kappa}}{N}}
 \operatorname{diag}(\sqrt{\tau_1},\ldots,\sqrt{\tau_D}).
\]
For \(\alpha\in\mathbb R^D\), define \(u_\alpha\in\mathcal H\) by
\[
 u_\alpha
 :=\sum_{j=1}^D\alpha_j
 \sqrt{\frac{N}{c_{\bar\kappa}\tau_j}}\,e_j.
\]
Then
\[
 \alpha^\top x_D
 =\langle u_\alpha,\overline U_{\lambda,\mathcal I}\rangle,
 \qquad
 \langle u_\alpha,T_\lambda u_\alpha\rangle
 =\frac N{c_{\bar\kappa}}\sum_{j=1}^D\alpha_j^2
 =\frac N{c_{\bar\kappa}}\|\alpha\|^2.
\]
Substitution in
\eqref{eq:regularized-sample-mean-exponential-bound} gives
\[
 \mathbb E\exp(\alpha^\top x_D)\le\exp(\|\alpha\|^2/2).
\]
In addition,
\[
 (A_Dx_D)_j
 =\langle\overline U_{\lambda,\mathcal I},e_j\rangle,
 \qquad
 \|A_Dx_D\|^2=\sum_{j=1}^D
 |\langle\overline U_{\lambda,\mathcal I},e_j\rangle|^2.
\]
Theorem~1 of \cite{hsu2012tail}, with \(\mu=0\), \(\sigma=1\), and matrix
\(A_D\), therefore gives, with probability at least \(1-e^{-s}\),
\begin{align}
 \|A_Dx_D\|^2
 \le\frac{c_{\bar\kappa}}{N}\left\{
 \operatorname{tr}(T_{\lambda,D})
 +2\sqrt{\operatorname{tr}(T_{\lambda,D}^2)s}
 +2\|T_{\lambda,D}\|_{\rm op}s\right\},
\label{eq:regularized-finite-dimensional-mean-bound}
\end{align}
where \(T_{\lambda,D}:=\operatorname{diag}(\tau_1,\ldots,\tau_D)\).
Since \(\operatorname{tr}(T_\lambda)=d_{X,\lambda}\),
\(\|T_\lambda\|_{\rm op}\le1\), and
\(\operatorname{tr}(T_\lambda^2)\le d_{X,\lambda}\), the right-hand side of
\eqref{eq:regularized-finite-dimensional-mean-bound} is at most
\(C_{\bar\kappa}\{d_{X,\lambda}+s\}/N\), uniformly in \(D\).

If \(P_0\) projects onto \(\ker(T_\lambda)\), then
\(\mathbb E\|P_0U_\lambda\|^2=
\operatorname{tr}(P_0T_\lambda P_0)=0\); hence
\(P_0\overline U_{\lambda,\mathcal I}=0\) almost surely.  If \(J<\infty\),
take \(D=J\) in
\eqref{eq:regularized-finite-dimensional-mean-bound}.  The spectral
decomposition of \(T_\lambda\) and
\(P_0\overline U_{\lambda,\mathcal I}=0\) then give
\(\|\overline U_{\lambda,\mathcal I}\|^2=\|A_Jx_J\|^2\).
If \(J=\infty\), let
\[
 E_D=\left\{\|A_Dx_D\|^2
 \le C_{\bar\kappa}\frac{d_{X,\lambda}+s}{N}\right\}.
\]
The events \(E_D\) decrease with \(D\), and
\eqref{eq:regularized-finite-dimensional-mean-bound} gives
\(\Pr(E_D)\ge1-e^{-s}\) for every \(D\).  Hence
\[
 \Pr\left(\bigcap_{D=1}^\infty E_D\right)
 =\lim_{D\to\infty}\Pr(E_D)\ge1-e^{-s}.
\]
On \(\bigcap_{D=1}^\infty E_D\), the spectral decomposition of \(T_\lambda\) and
\(P_0\overline U_{\lambda,\mathcal I}=0\) give
\[
 \|A_Dx_D\|^2
 =\sum_{j=1}^D
 |\langle\overline U_{\lambda,\mathcal I},e_j\rangle|^2
 \mathrel{\uparrow}\|\overline U_{\lambda,\mathcal I}\|^2.
\]
Thus, in both the finite- and infinite-rank cases,
\begin{equation}
\label{eq:regularized-sample-mean-bound}
 \|\overline U_{\lambda,\mathcal I}\|^2
 \le C_{\bar\kappa}\frac{d_{X,\lambda}+s}{N}
\end{equation}
with probability at least \(1-e^{-s}\).

For \(t>0\), apply
\eqref{eq:regularized-raw-second-moment-bound} and
\eqref{eq:regularized-sample-mean-bound} with
\(s=(t+\log2)\vee1\).  The joint failure probability of
\eqref{eq:regularized-raw-second-moment-bound} and
\eqref{eq:regularized-sample-mean-bound} is at most
\(2e^{-s}\le e^{-t}\).  Equation
\eqref{eq:regularized-sample-centering-identity} and the triangle inequality
give
\begin{align}
 &\left\|
 M_\lambda^{-1/2}\widehat\Sigma_{X,\mathcal I}
 M_\lambda^{-1/2}-T_\lambda
 \right\|_{\rm op}\\
 &\quad\le
 \frac N{N-1}\|\widetilde T_{\lambda,\mathcal I}-T_\lambda\|_{\rm op}
 +\frac N{N-1}\|\overline U_{\lambda,\mathcal I}\|^2
 +\frac1{N-1}\|T_\lambda\|_{\rm op}.
\label{eq:regularized-centered-covariance-bound}
\end{align}
Substituting \eqref{eq:regularized-raw-second-moment-bound} and
\eqref{eq:regularized-sample-mean-bound} into
\eqref{eq:regularized-centered-covariance-bound}, using
\(s\le(1+\log2)(t\vee1)\), \(N/(N-1)\le2\), and
\(\|T_\lambda\|_{\rm op}\le1\), shows that the operator norm in
\eqref{eq:regularized-centered-covariance-bound} is at most
\(\vartheta_{N,\lambda}(t\vee1)\) after choosing
\(C_{\rm cov}(\bar\kappa)\) large enough, with
\(C_{\rm cov}(\bar\kappa)\) depending only on \(\bar\kappa\).  The inequality
\[
 \left\|M_\lambda^{-1/2}
 (\widehat\Sigma_{X,\mathcal I}-\Sigma_X)
 M_\lambda^{-1/2}\right\|_{\rm op}
 \le\vartheta_{N,\lambda}(t\vee1)
\]
proves
\eqref{eq:regularized-covariance-concentration} for \(t>0\).  For \(t=0\),
the bound \(\Pr\{\cdots\}\le e^{-t}\) in
\eqref{eq:regularized-covariance-concentration} is automatic because
\(e^{-0}=1\).
\end{proof}
 
\begin{proof}[Proof of
\MainResult{Theorem}{thm:balanced-one-split-linear-confidence-field}]
Set \(t=t_\alpha=\log(c_1/\alpha)\).  Under
\(\vartheta_{N,\lambda}(t)\le1/2\),
Lemmas~\ref{lem:regularized-covariance-concentration}
and~\ref{lem:regularized-pilot-geometry} give
\(\Pr\{\rho_A>2\}\le e^{-t}\).  Conditional on a pilot realization with
\(\rho_A\le2\),
\MainResult{Theorem}{thm:balanced-one-split-linear-confidence} with \(R=2\)
gives the quadratic bound in
\MainEquation{eq:balanced-finite-split-confidence-ellipsoid} with conditional
failure probability at most \(3e^{-t}\).  The unconditional failure
probability of \MainEquation{eq:balanced-finite-split-confidence-ellipsoid}
is therefore at most
\(4e^{-t}=4\alpha/c_1\le\alpha\).  Cauchy--Schwarz in the \(M_A\) metric gives
\MainEquation{eq:balanced-finite-split-linear-confidence} on the event where
\MainEquation{eq:balanced-finite-split-confidence-ellipsoid} holds.
If \(A\) is randomized independently of the observations, conditioning on
\(A\) and then averaging the conditional failure probability proves the
joint statement.
\end{proof}

\subsection{Finite prespecified split sequences}
\label{supp:subsec:finite-split-extension}
\label{supp:subsec:finite-split-proofs-rates}

The main paper uses one balanced split.  This subsection records the direct
extension to a finite, data-independent sequence; relative to the one-split
field, only the number of distinct oriented pilot folds enters the tail
level.  Let \(n=2N\) with
\(N\ge2\), fix \(\lambda>0\), \(\alpha\in(0,1)\), and an integer
\(K_{\rm sp}\ge1\), and let
\(\boldsymbol A=(A_1,\ldots,A_{K_{\rm sp}})\), where every \(A_k\) is a
balanced pilot fold, \(B_k=A_k^c\), and repetitions and overlaps are allowed.
The sequence may be fixed or randomized independently of the observations.
If \(\boldsymbol A\) is randomized independently of the observations, every
assumption involving the realized split sequence is required to hold almost
surely.  Write
\[
 K_\circ(\boldsymbol A)
 :=|\{A_k:1\le k\le K_{\rm sp}\}|,
\]
where the integer \(K_\circ(\boldsymbol A)\) counts distinct oriented pilot
folds.  Define the averaged pilot metric
\[
 \widetilde M_{X,\lambda}^{(K_{\rm sp})}
 :=\frac1{K_{\rm sp}}\sum_{k=1}^{K_{\rm sp}}M_{X,A_k,\lambda}.
\]
The operator \(\widetilde M_{X,\lambda}^{(K_{\rm sp})}\) is positive definite
on \(\mathcal H\).  For \(c_1\ge4\), define
the tail threshold and averaged squared-radius bound
\[
 t_\alpha:=\log\frac{c_1K_\circ(\boldsymbol A)}{\alpha},
 \qquad
 \widetilde U_{X,\lambda,\alpha}^{(K_{\rm sp})}
 :=\frac1{K_{\rm sp}}\sum_{k=1}^{K_{\rm sp}}
 \{\widehat T_{X,A_k,\lambda}
 +\widehat\Delta_{X,A_k,\lambda}(t_\alpha;2)\}.
\]

\begin{theorem}[Balanced finite-split confidence ellipsoid]
\label{supp:thm:balanced-finite-split-linear-confidence}
Assume the covariance-matched sub-Gaussian condition in
\MainEquation{eq:main-covariance-matched-subgaussian} with known constant
\(\bar\kappa\).  With \(\vartheta_{N,\lambda}\) defined in
\MainEquation{eq:main-pilot-stability-rate}, suppose
\begin{equation}
\label{supp:eq:finite-split-stability-conditions}
 \varepsilon_N(t_\alpha)<1,\qquad
 \vartheta_{N,\lambda}(t_\alpha)\le\frac12.
\end{equation}
For any fixed \(\boldsymbol A\), with probability at least \(1-\alpha\),
\begin{equation}
\label{supp:eq:balanced-finite-split-confidence-ellipsoid}
 \|PX-P_nX\|_{
 \{\widetilde M_{X,\lambda}^{(K_{\rm sp})}\}^{-1}}^2
 \le\widetilde U_{X,\lambda,\alpha}^{(K_{\rm sp})}.
\end{equation}
Whenever \eqref{supp:eq:balanced-finite-split-confidence-ellipsoid} holds,
Cauchy--Schwarz gives, simultaneously for every \(h\in\mathcal H\),
\begin{equation}
\label{supp:eq:balanced-finite-split-linear-confidence}
 |(P-P_n)f_h|
 \le\left\{\langle h,\widetilde M_{X,\lambda}^{(K_{\rm sp})}h\rangle
 \widetilde U_{X,\lambda,\alpha}^{(K_{\rm sp})}\right\}^{1/2}.
\end{equation}
If \(\boldsymbol A\) is randomized independently of the observations and
\eqref{supp:eq:finite-split-stability-conditions} holds almost surely, the
inequalities \eqref{supp:eq:balanced-finite-split-confidence-ellipsoid}--
\eqref{supp:eq:balanced-finite-split-linear-confidence} both hold
with probability at least \(1-\alpha\), where probability is taken over both
\(\boldsymbol A\) and the observations.
\end{theorem}

\begin{proof}[Proof of
Theorem~\ref{supp:thm:balanced-finite-split-linear-confidence}]
Fix a balanced split sequence satisfying
\eqref{supp:eq:finite-split-stability-conditions}.  For a given split,
Lemma~\ref{lem:regularized-covariance-concentration},
Lemma~\ref{lem:regularized-pilot-geometry}, and
\MainResult{Theorem}{thm:balanced-one-split-linear-confidence}
show that the quadratic bound with \(R=2\) fails with probability at most
\(4e^{-t}\).  The event
\eqref{eq:pilot-relative-covariance-event} holds with probability at least
\(1-e^{-t}\) and implies \(\rho_A\le2\) by
\eqref{eq:supp-pilot-geometry-consequences}.  Conditional on a pilot
realization satisfying \(\rho_A\le2\), the conditional failure probability of
\MainEquation{eq:balanced-one-split-quadratic-confidence} is at most
\(3e^{-t}\).  A union bound over the
\(K_\circ(\boldsymbol A)\) distinct oriented pilot folds, with
\(t=\log\{c_1K_\circ(\boldsymbol A)/\alpha\}\) and \(c_1\ge4\), gives all
splitwise quadratic bounds simultaneously with probability at least
\(1-\alpha\).  Write \(M_{A_k}:=M_{X,A_k,\lambda}\).

To average the splitwise quadratic bounds, fix an integer \(K\ge1\), let
\(M_1,\ldots,M_K:\mathcal H\to\mathcal H\) be self-adjoint operators
satisfying \(M_k\succeq\lambda I\), and write
\(\overline M:=K^{-1}\sum_{k=1}^KM_k\).  For every \(x\in\mathcal H\),
\begin{align}
 \langle x,\overline M^{-1}x\rangle
 &=\sup_{y\in\mathcal H}
 \{2\langle x,y\rangle-\langle y,\overline My\rangle\}\notag\\
 &=\sup_{y\in\mathcal H}\frac1K\sum_{k=1}^K
 \{2\langle x,y\rangle-\langle y,M_ky\rangle\}\notag\\
 &\le\frac1K\sum_{k=1}^K
 \sup_{y\in\mathcal H}
 \{2\langle x,y\rangle-\langle y,M_ky\rangle\}
 =\frac1K\sum_{k=1}^K\langle x,M_k^{-1}x\rangle.
\label{eq:inverse-operator-average-quadratic}
\end{align}
The variational identity involving \(\overline M\) uses the maximizer
\(y=\overline M^{-1}x\), and the \(k\)th variational identity involving
\(M_k\) uses the maximizer \(y=M_k^{-1}x\).
Applying \eqref{eq:inverse-operator-average-quadratic} with
\(K=K_{\rm sp}\), \(M_k=M_{A_k}\), and \(x=PX-P_nX\), and then averaging the
simultaneous splitwise quadratic bounds, proves
\eqref{supp:eq:balanced-finite-split-confidence-ellipsoid}.
For every \(h\in\mathcal H\), Cauchy--Schwarz gives
\[
 |(P-P_n)f_h|
 \le
 \{\langle h,\widetilde M_{X,\lambda}^{(K_{\rm sp})}h\rangle
 \widetilde U_{X,\lambda,\alpha}^{(K_{\rm sp})}\}^{1/2},
\]
which is \eqref{supp:eq:balanced-finite-split-linear-confidence}.  If the
split sequence is random and independent of the observations, then, for
almost every realized sequence, the conditional failure probability of
\eqref{supp:eq:balanced-finite-split-confidence-ellipsoid}--
\eqref{supp:eq:balanced-finite-split-linear-confidence} is at most
\(\alpha\).  Integrating over the split sequence proves the unconditional
joint guarantee in
\eqref{supp:eq:balanced-finite-split-confidence-ellipsoid}--
\eqref{supp:eq:balanced-finite-split-linear-confidence}.
\end{proof}

The finite-split ellipsoid uses a random averaged metric.  To bound the width
of the finite-split ellipsoid in the population ridge geometry from
\MainEquation{eq:population-ridge-effective-dimension}, one must control both
\(\widetilde M_{X,\lambda}^{(K_{\rm sp})}\) and
\(\widetilde U_{X,\lambda,\alpha}^{(K_{\rm sp})}\).

\begin{corollary}[Finite-split population-geometry width]
\label{supp:cor:balanced-finite-split-oracle-width}
Under the assumptions of
Theorem~\ref{supp:thm:balanced-finite-split-linear-confidence}, take
\(c_1\ge6\) and suppose
\(\varepsilon_N(t_\alpha)\le1/2\) and
\(\vartheta_{N,\lambda}(t_\alpha)\le1/2\).  There is a constant
\(C_{\rm or}(\bar\kappa)\) such that, for any fixed \(\boldsymbol A\), with
probability at least \(1-\alpha\),
\begin{align}
 \widetilde M_{X,\lambda}^{(K_{\rm sp})}
 &\preceq\tfrac32(\Sigma_X+\lambda I),
 \label{supp:eq:balanced-finite-split-oracle-metric}\\
 \widetilde U_{X,\lambda,\alpha}^{(K_{\rm sp})}
 &\le C_{\rm or}(\bar\kappa)
 \frac{d_{X,\lambda}+t_\alpha}{n},
 \label{supp:eq:balanced-finite-split-oracle-intercept}.
\end{align}
In addition, the confidence ellipsoid
\eqref{supp:eq:balanced-finite-split-confidence-ellipsoid} holds, and,
simultaneously for every \(h\in\mathcal H\),
\begin{equation}
\label{supp:eq:balanced-finite-split-oracle-width}
 |(P-P_n)f_h|
 \le C_{\rm or}(\bar\kappa)
 \left\{\frac{\langle h,(\Sigma_X+\lambda I)h\rangle
 (d_{X,\lambda}+t_\alpha)}n\right\}^{1/2}.
\end{equation}
If \(\boldsymbol A\) is randomized independently of the observations and
\(\varepsilon_N(t_\alpha)\le1/2\) and
\(\vartheta_{N,\lambda}(t_\alpha)\le1/2\) hold almost surely,
\eqref{supp:eq:balanced-finite-split-oracle-metric}--
\eqref{supp:eq:balanced-finite-split-oracle-width} all hold with probability
at least \(1-\alpha\), where probability is taken over both \(\boldsymbol A\)
and the observations.
\end{corollary}

\begin{proof}[Proof of
Corollary~\ref{supp:cor:balanced-finite-split-oracle-width}]
Condition on the split sequence, write
\[
 D:=K_\circ(\boldsymbol A),\qquad
 t:=\log\{c_1D/\alpha\},\qquad d:=d_{X,\lambda},
\]
and fix one of the \(D\) distinct oriented pilot folds in
\(\boldsymbol A\).  The condition on
\(\vartheta_{N,\lambda}(t)\) and
Lemma~\ref{lem:regularized-covariance-concentration} give
\eqref{eq:pilot-relative-covariance-event} with failure probability at most
\(e^{-t}\).  Lemma~\ref{lem:regularized-pilot-geometry} then gives
\begin{equation}
\label{eq:oracle-proof-pilot-geometry}
 \tfrac12M_\lambda\preceq M_A\preceq\tfrac32M_\lambda,\qquad
 r_A\le2d,\qquad s_A\le4d,\qquad\rho_A\le2.
\end{equation}
Separately, apply
Lemma~\ref{lem:hilbert-subgaussian-quadratic-concentration} in the
deterministic \(M_\lambda\) metric to the pilot mean.  Outside an event of
probability at most \(e^{-t}\),
\begin{equation}
\label{eq:oracle-proof-pilot-mean}
 \|M_\lambda^{-1/2}(PX-\overline X_A)\|^2
 \le C_{\bar\kappa}\frac{d+t}{N}.
\end{equation}
A union bound makes \eqref{eq:oracle-proof-pilot-geometry} and
\eqref{eq:oracle-proof-pilot-mean} simultaneous under arbitrary dependence
between the pilot mean and covariance.  On the intersection of
\eqref{eq:oracle-proof-pilot-geometry} and
\eqref{eq:oracle-proof-pilot-mean},
\begin{equation}
\label{eq:oracle-proof-random-metric-pilot-mean}
 \|b_A\|^2\le C_{\bar\kappa}\frac{d+t}{N}.
\end{equation}

Conditional on a pilot satisfying
\eqref{eq:oracle-proof-pilot-geometry}, the evaluation-mean event
\eqref{eq:conditional-evaluation-mean-norm} fails with probability at most
\(e^{-t}\), and gives
\begin{equation}
\label{eq:oracle-proof-evaluation-mean}
 \|\xi_A\|^2\le C_{\bar\kappa}\frac{d+t}{N}.
\end{equation}
The two-sided raw-trace event
\eqref{eq:conditional-raw-trace-sandwich} fails with conditional probability
at most \(2e^{-t}\).  The upper raw-trace bound in
\eqref{eq:conditional-raw-trace-sandwich}, the centered-trace identity
\eqref{eq:conditional-centered-trace-identity}, and
\(\varepsilon_N(t)\le1/2\) imply
\begin{equation}
\label{eq:oracle-proof-observable-trace}
 \widehat r_{B_A\mid A}\le C_{\bar\kappa}d,\qquad
 \overline r_{B_A\mid A}(t;2)
 \le C_{\bar\kappa}\left(d+\frac tN\right).
\end{equation}
The scalar conditional cross-term event
\eqref{eq:conditional-cross-term-bound} fails with probability at most
\(e^{-t}\).  The intersection of
\eqref{eq:conditional-cross-term-bound},
\eqref{eq:conditional-evaluation-mean-norm}, and the lower raw-trace bound in
\eqref{eq:conditional-raw-trace-sandwich} satisfies the conditions of
\MainResult{Theorem}{thm:balanced-one-split-linear-confidence} with \(R=2\)
and therefore implies
\MainEquation{eq:balanced-one-split-quadratic-confidence}.

For balanced folds,
\MainEquation{eq:crossfit-single-split-calibration} and
\eqref{eq:oracle-proof-random-metric-pilot-mean}--
\eqref{eq:oracle-proof-evaluation-mean} give
\begin{equation}
\label{eq:oracle-proof-calibration-statistic}
 \widehat T_{X,A,\lambda}
 =\frac14\|b_A+\xi_A\|^2
 \le C_{\bar\kappa}\frac{d+t}{N}.
\end{equation}
Substituting
\eqref{eq:oracle-proof-observable-trace} and
\eqref{eq:oracle-proof-calibration-statistic} into
\MainEquation{eq:balanced-quadratic-confidence-inflation} gives
\begin{equation}
\label{eq:oracle-proof-split-intercept}
 \widehat T_{X,A,\lambda}
 +\widehat\Delta_{X,A,\lambda}(t;2)
 \le C_{\rm or}(\bar\kappa)\frac{d+t}{N},
\end{equation}
because
\(\sqrt{dt}+t\le d+2t\).

The per-split failure probability is at most \(6e^{-t}\): one \(e^{-t}\)
term each for the pilot covariance, pilot mean, evaluation mean, and scalar
cross term, and \(2e^{-t}\) for the two-sided raw trace.  The union-bound
failure probability over the \(D\) distinct oriented pilot folds is at most
\(6De^{-t}\le\alpha\), since \(c_1\ge6\).  On the intersection of the bounds in
\eqref{eq:oracle-proof-pilot-geometry},
\eqref{eq:oracle-proof-random-metric-pilot-mean},
\eqref{eq:oracle-proof-evaluation-mean},
\eqref{eq:conditional-raw-trace-sandwich}, and
\eqref{eq:conditional-cross-term-bound}, averaging the upper metric comparison in
\eqref{eq:oracle-proof-pilot-geometry} proves
\eqref{supp:eq:balanced-finite-split-oracle-metric}; averaging
\eqref{eq:oracle-proof-split-intercept}, and using \(n=2N\), proves
\eqref{supp:eq:balanced-finite-split-oracle-intercept};
\MainResult{Theorem}{thm:balanced-one-split-linear-confidence} supplies the
splitwise quadratic bounds.  Thus
\[
 \|PX-P_nX\|_{
 \{\widetilde M_{X,\lambda}^{(K_{\rm sp})}\}^{-1}}^2
 \le\frac1{K_{\rm sp}}\sum_{k=1}^{K_{\rm sp}}
 \|PX-P_nX\|_{M_{A_k}^{-1}}^2
 \le\widetilde U_{X,\lambda,\alpha}^{(K_{\rm sp})},
\]
where
\[
 \|PX-P_nX\|_{
 \{\widetilde M_{X,\lambda}^{(K_{\rm sp})}\}^{-1}}^2
 \le\frac1{K_{\rm sp}}\sum_{k=1}^{K_{\rm sp}}
 \|PX-P_nX\|_{M_{A_k}^{-1}}^2
\]
is \eqref{eq:inverse-operator-average-quadratic} with
\(K=K_{\rm sp}\), \(M_k=M_{A_k}\), and \(x=PX-P_nX\).
The two inequalities establish
\eqref{supp:eq:balanced-finite-split-confidence-ellipsoid}.  Applying
Cauchy--Schwarz and then
\eqref{supp:eq:balanced-finite-split-oracle-metric}--
\eqref{supp:eq:balanced-finite-split-oracle-intercept} gives
\begin{align*}
 |(P-P_n)f_h|
 &\le
 \{\langle h,\widetilde M_{X,\lambda}^{(K_{\rm sp})}h\rangle
 \widetilde U_{X,\lambda,\alpha}^{(K_{\rm sp})}\}^{1/2}\\
 &\le C_{\rm or}(\bar\kappa)
 \left\{\frac{\langle h,(\Sigma_X+\lambda I)h\rangle
 (d+t)}n\right\}^{1/2},
\end{align*}
after absorbing the numerical factor \(3/2\) into
\(C_{\rm or}(\bar\kappa)\).  The Cauchy--Schwarz bound proves
\eqref{supp:eq:balanced-finite-split-oracle-width}.  If the split sequence is
randomized independently of the observations, then, for almost every
realized sequence, the conditional failure probability of the joint
conclusions \eqref{supp:eq:balanced-finite-split-oracle-metric}--
\eqref{supp:eq:balanced-finite-split-oracle-width} is at most \(\alpha\).
Integrating over the split sequence gives joint probability at least
\(1-\alpha\) for
\eqref{supp:eq:balanced-finite-split-oracle-metric}--
\eqref{supp:eq:balanced-finite-split-oracle-width} over the split sequence and
the observations.
\end{proof}

Proposition~\ref{supp:prop:finite-split-tail-inflation-rate} gives the
triangular-array rate of the additive correction \(\widehat\Delta\).

\Needspace{0.62\textheight}
\begin{proposition}[Finite-split confidence-correction rate]
\label{supp:prop:finite-split-tail-inflation-rate}
Consider a triangular array indexed by even sample sizes \(n\to\infty\), with random elements
\(X_n\), ridges \(\lambda_n>0\), levels \(\alpha_n\in(0,1)\), and
deterministic data-independent balanced split sequences
\(\boldsymbol A_n=(A_{n,1},\ldots,A_{n,K_{{\rm sp},n}})\).  For every \(n\),
let \(\bar\kappa_n\ge1\) be such that
\MainEquation{eq:main-covariance-matched-subgaussian} holds with \(X=X_n\)
and \(\bar\kappa=\bar\kappa_n\).  Assume
\(\bar\kappa_\star:=\sup_n\bar\kappa_n<\infty\), let \(c_1\ge4\), and put
\[
 d_n:=d_{X_n,\lambda_n},\qquad
 t_n:=\log\{c_1K_\circ(\boldsymbol A_n)/\alpha_n\}.
\]
If \(d_n+t_n=o(n)\), then
\[
 \frac n{K_{{\rm sp},n}}\sum_{k=1}^{K_{{\rm sp},n}}
 \widehat\Delta_{X_n,A_{n,k},\lambda_n}(t_n;2)
 =O_p\{\sqrt{d_nt_n}+t_n\}.
\]
If also \(t_n/d_n\to0\), then
\[
 \frac n{K_{{\rm sp},n}}\sum_{k=1}^{K_{{\rm sp},n}}
 \widehat\Delta_{X_n,A_{n,k},\lambda_n}(t_n;2)
 =o_p(d_n).
\]
\end{proposition}

\begin{proof}[Proof of
Proposition~\ref{supp:prop:finite-split-tail-inflation-rate}]
Write
\[
 K_n:=K_{{\rm sp},n},\qquad
 D_n:=K_\circ(\boldsymbol A_n),\qquad
 N:=n/2.
\]
Every row satisfies
\MainEquation{eq:main-covariance-matched-subgaussian} with
\(\bar\kappa=\bar\kappa_\star\), so all subsequent concentration bounds use
constants depending on \(\bar\kappa_\star\), uniformly in \(n\).
For the remainder of the proof, write \(X=X_n\),
\(A_k=A_{n,k}\), \(B_k=B_{n,k}\), and
\(M_{A_k}:=M_{X_n,A_{n,k},\lambda_n}\) for
\(1\le k\le K_n\).
Assumption
\(d_n+t_n=o(n)\), together with \(N=n/2\), makes both terms in
\MainEquation{eq:main-pilot-stability-rate} and all three nonconstant terms
in \MainEquation{eq:conditional-trace-relative-error} tend to zero.
Consequently, for all sufficiently large \(n\),
\(\vartheta_{N,\lambda_n}(t_n)\le1/2\) and
\(\varepsilon_N(t_n)\le1/2\).

Fix \(x\ge1\) and put \(u_n=t_n+x\).  Because \(x\) is fixed,
\((d_n+u_n)/n=(d_n+t_n+x)/n\to0\), so
\(\vartheta_{N,\lambda_n}(u_n)\le1/2\) and
\(\varepsilon_N(u_n)\le1/2\) also hold for all sufficiently large \(n\).
The deterministic-metric quadratic
bound \eqref{eq:hilbert-subgaussian-mean-quadratic-bound}, applied with
\(U=M_{\lambda_n}^{-1/2}(X-PX)\), controls
\(\|M_{\lambda_n}^{-1/2}(PX-\overline X_{A_k})\|^2\).
Lemma~\ref{lem:regularized-covariance-concentration}, followed by
Lemma~\ref{lem:regularized-pilot-geometry}, gives
\eqref{eq:pilot-relative-covariance-event} and
\(M_{A_k}^{-1}\preceq2M_{\lambda_n}^{-1}\).
For one oriented fold, the mean event from
\eqref{eq:hilbert-subgaussian-mean-quadratic-bound} and the covariance event
\eqref{eq:pilot-relative-covariance-event} each fail with probability at most
\(e^{-u_n}\).  A union bound over the \(D_n\) distinct oriented pilot folds
therefore gives
\begin{equation}
\label{eq:rate-pilot-event-failure}
 2D_ne^{-u_n}
 =\frac{2\alpha_n}{c_1}e^{-x}\le\frac12e^{-x}.
\end{equation}
Equation~\eqref{eq:rate-pilot-event-failure} gives, outside an event of
probability at most \(e^{-x}/2\),
\begin{equation}
\label{eq:rate-pilot-mean-bound}
 \|M_{A_k}^{-1/2}(PX-\overline X_{A_k})\|^2
 \le C_{\bar\kappa_\star}\frac{d_n+u_n}{N},
 \qquad k=1,\ldots,K_n.
\end{equation}
Equation~\eqref{eq:rate-pilot-mean-bound} combines the deterministic-metric
mean bound with \(M_{A_k}^{-1}\preceq2M_{\lambda_n}^{-1}\).  The pilot mean
and covariance may be dependent; the union bound remains valid under
dependence between the pilot mean and covariance.  For \(1\le k\le K_n\), let
\[
 \mathcal C_{n,k}:=\left\{
 \left\|M_{\lambda_n}^{-1/2}
 (\widehat\Sigma_{X,A_k}-\Sigma_X)M_{\lambda_n}^{-1/2}
 \right\|_{\rm op}\le\frac12\right\}.
\]
Conditional on a pilot realization in \(\mathcal C_{n,k}\),
Lemma~\ref{lem:regularized-pilot-geometry} gives
\(\rho_{A_k}\le2\) and \(r_{A_k}\le2d_n\).  The evaluation-mean bound
\eqref{eq:conditional-evaluation-mean-norm}, applied at \(u_n\), then fails
with conditional probability at most \(e^{-u_n}\).  Integrating over the
pilot observations and taking the union over the \(D_n\) distinct oriented
folds gives
\begin{equation}
\label{eq:rate-evaluation-event-failure}
 \begin{split}
 &\Pr\Biggl\{\bigcup_{k=1}^{K_n}\Biggl(
 \mathcal C_{n,k}\cap
 \Biggl\{\|M_{A_k}^{-1/2}(\overline X_{B_k}-PX)\|^2\\
 &\hspace{42mm}>
 C_{\bar\kappa_\star}\frac{d_n+u_n}{N}\Biggr\}\Biggr)\Biggr\}
 \le D_ne^{-u_n}
 =\frac{\alpha_n}{c_1}e^{-x}\le\frac14e^{-x}.
 \end{split}
\end{equation}
Combining \eqref{eq:rate-evaluation-event-failure} with
\eqref{eq:rate-pilot-event-failure}, outside an event of probability at most
\[
 3D_ne^{-u_n}=\frac{3\alpha_n}{c_1}e^{-x}\le\frac34e^{-x},
\]
one has, simultaneously with \eqref{eq:rate-pilot-mean-bound} and the pilot
covariance bounds,
\begin{equation}
\label{eq:rate-evaluation-mean-bound}
 \|M_{A_k}^{-1/2}(\overline X_{B_k}-PX)\|^2
 \le C_{\bar\kappa_\star}\frac{d_n+u_n}{N},
 \qquad k=1,\ldots,K_n.
\end{equation}
On the event where \eqref{eq:rate-pilot-mean-bound},
\eqref{eq:rate-evaluation-mean-bound}, and \(\mathcal C_{n,k}\) hold for every
\(k\), the two mean bounds imply
\begin{equation}
\label{eq:rate-calibration-statistic-bound}
 \widehat T_{X,A_k,\lambda_n}
 =\frac14\|b_{A_k}+\xi_{A_k}\|^2
 \le\frac12\{\|b_{A_k}\|^2+\|\xi_{A_k}\|^2\}
 \le C_{\bar\kappa_\star}\frac{d_n+u_n}{N},
 \qquad k=1,\ldots,K_n.
\end{equation}

For one evaluation fold, the upper raw-trace bound in
\eqref{eq:conditional-raw-trace-sandwich} at \(u_n\) has conditional failure
probability at most \(2e^{-u_n}\).  Integrating over the pilot observations
and taking the union over the \(D_n\) distinct oriented folds gives failure
probability
\[
 2D_ne^{-u_n}
 =\frac{2\alpha_n}{c_1}e^{-x}
 \le\frac12e^{-x}.
\]
On the intersection of the upper-trace events in
\eqref{eq:conditional-raw-trace-sandwich} at \(u_n\) and the covariance events
\eqref{eq:pilot-relative-covariance-event} for every \(A_k\),
\eqref{eq:conditional-centered-trace-identity} and
Lemma~\ref{lem:regularized-pilot-geometry}, applied to \(X,\lambda_n\), give
simultaneously in \(k\)
\[
 \widehat r_{B_k\mid A_k}
 \le\frac N{N-1}r_{A_k}
 \left\{1+C_{\bar\kappa_\star}\left(
 \sqrt{\frac{u_n}{N}}+\frac{u_n}{N}\right)\right\}
 \le C_{\bar\kappa_\star,x}d_n.
\]
The pilot-covariance events
\eqref{eq:pilot-relative-covariance-event} and pilot-mean bounds
\eqref{eq:rate-pilot-mean-bound} have joint failure probability at most
\(e^{-x}/2\) by \eqref{eq:rate-pilot-event-failure}.  Combining the failure
bound in \eqref{eq:rate-pilot-event-failure} with the upper-trace union bound
from \eqref{eq:conditional-raw-trace-sandwich}, and using
\(\varepsilon_N(t_n)\le1/2\), substitution in
\MainEquation{eq:conditional-evaluation-trace-ucb} gives
\begin{equation}
\label{eq:rate-trace-ucb-bound}
 \overline r_{B_k\mid A_k}(t_n;2)
 \le C_{\bar\kappa_\star,x}\left(d_n+\frac{t_n}{N}\right),
 \qquad k=1,\ldots,K_n,
\end{equation}
on an event of probability at least \(1-e^{-x}\).  Because
the covariance, pilot-mean, evaluation-mean, and two-sided raw-trace events
have joint failure probability at most
\[
 5D_ne^{-u_n}=\frac{5\alpha_n}{c_1}e^{-x}\le\frac54e^{-x},
\]
the covariance, pilot-mean, evaluation-mean, and trace bounds in
\eqref{eq:pilot-relative-covariance-event},
\eqref{eq:rate-pilot-mean-bound}, \eqref{eq:rate-evaluation-mean-bound}, and
\eqref{eq:rate-trace-ucb-bound} hold simultaneously outside an event of
probability at most \(5D_ne^{-u_n}\).  Since
\(t_n\ge1\), the covariance, pilot-mean, evaluation-mean, and trace bounds
in \eqref{eq:pilot-relative-covariance-event},
\eqref{eq:rate-pilot-mean-bound}, \eqref{eq:rate-evaluation-mean-bound}, and
\eqref{eq:rate-trace-ucb-bound} may be invoked at \(u_n\), while
\(\widehat\Delta_{X,A_k,\lambda_n}(t_n;2)\) remains evaluated at \(t_n\); for
fixed \(x\), also
\(\varepsilon_N(u_n)\le1/2\) eventually.  Substitution of
\eqref{eq:rate-calibration-statistic-bound}--
\eqref{eq:rate-trace-ucb-bound} into
\MainEquation{eq:balanced-quadratic-confidence-inflation} gives, for every
fixed \(x\),
\begin{align*}
 n\sqrt{\frac{2t_n\widehat T_{X,A_k,\lambda_n}}N}
 &\le C_{\bar\kappa_\star,x}\sqrt{t_n(d_n+u_n)},\\
 \frac nN\left\{
 \sqrt{2\overline r_{B_k\mid A_k}(t_n;2)t_n}+2t_n
 \right\}
 &\le C_{\bar\kappa_\star,x}\{\sqrt{d_nt_n}+t_n\}.
\end{align*}
Also,
\[
 \sqrt{t_n(d_n+u_n)}
 \le\sqrt{d_nt_n}+t_n+\sqrt{x\,t_n}
 \le\sqrt{d_nt_n}+(1+\sqrt x)t_n.
\]
Consequently,
\[
 n\widehat\Delta_{X,A_k,\lambda_n}(t_n;2)
 \le C_{\bar\kappa_\star,x}\{\sqrt{d_nt_n}+t_n\},
 \qquad k=1,\ldots,K_n,
\]
outside an event of probability at most \(Ce^{-x}\).  Averaging over \(k\)
proves
\[
\frac n{K_n}\sum_{k=1}^{K_n}
 \widehat\Delta_{X,A_k,\lambda_n}(t_n;2)
 =O_p\{\sqrt{d_nt_n}+t_n\}.
\]
If \(t_n/d_n\to0\), then
\[
 \frac{\sqrt{d_nt_n}+t_n}{d_n}
 =\sqrt{\frac{t_n}{d_n}}+\frac{t_n}{d_n}\longrightarrow0,
\]
which proves the relative conclusion.
\end{proof}
  \section{Generic Bounded-Difference Lift}
\label{app:finite-split-constants}

Throughout Section~\ref{app:finite-split-constants}, \(\delta_v\) is the loss
increment defined in \MainEquation{eq:risk-increment}, and \(P\delta_v\)
denotes the expectation of \(\delta_v\).

Let \(\mathcal Z_0\subseteq\mathcal Z\) be measurable with
\(P(\mathcal Z_0)=1\).  For samples \(S,S'\in\mathcal Z_0^n\), write
\(S\sim_iS'\) when \(S\) and \(S'\) agree in every coordinate except \(i\).  For a
measurable random field \(T_S\) on a deterministic set \(\mathcal E\), define
\begin{equation}
\label{eq:uniform-replacement-modulus}
 \Delta_i(T;\mathcal E)
 :=\sup_{S\sim_iS'}\sup_{v\in\mathcal E}
 |T_S(v)-T_{S'}(v)|.
\end{equation}
Call \(T\) \((b_1,\ldots,b_n)\)-stable on \(\mathcal E\) when the
\(b_i\) are deterministic and \(\Delta_i(T;\mathcal E)\le b_i\).

Mean validity imposes the expectation inequalities in
\MainResult{Definition}{def:expected-local-endpoint}.
Proposition~\ref{prop:uniform-high-probability-guarantee} uses deterministic
replacement moduli to obtain a simultaneous high-probability correction.

\begin{proposition}[Uniform concentration of mean-valid endpoints]
\label{prop:uniform-high-probability-guarantee}
Let \(Z_1,\ldots,Z_n\) be independent observations with common law \(P\).
Let \(\widehat{\mathsf U}_{\mathcal E}\) and
\(\widehat{\mathsf L}_{\mathcal E}\) be finite-valued mean-valid upper and
lower endpoints on a deterministic Borel set \(\mathcal E\).  For
\(S\in\mathcal Z_0^n\), define
\[
 H^+(S):=\sup_{v\in\mathcal E}
 \{P\delta_v-\widehat{\mathsf U}_{\mathcal E,S}(v)\},\qquad
 H^-(S):=\sup_{v\in\mathcal E}
 \{\widehat{\mathsf L}_{\mathcal E,S}(v)-P\delta_v\}.
\]
Assume that \(H^+\) and \(H^-\) are measurable functions of \(S\) with
values in \(\mathbb R\).  Suppose \(\widehat{\mathsf U}_{\mathcal E}\) is
\((b_1^+,\ldots,b_n^+)\)-stable and \(\widehat{\mathsf L}_{\mathcal E}\) is
\((b_1^-,\ldots,b_n^-)\)-stable.  Then, for every \(\alpha\in(0,1)\), with
probability at least \(1-\alpha\), simultaneously for \(v\in\mathcal E\),
\begin{equation}
\label{eq:two-sided-uniform-high-probability-lift}
 \widehat{\mathsf L}_{\mathcal E}(v)-\Gamma_\alpha^-
 \le P\delta_v
 \le\widehat{\mathsf U}_{\mathcal E}(v)+\Gamma_\alpha^+,
 \qquad
 \Gamma_\alpha^\pm
 :=\left\{\frac12\log\frac2\alpha
 \sum_{i=1}^n(b_i^\pm)^2\right\}^{1/2}.
\end{equation}
For a one-sided statement, \(\log(2/\alpha)\) may be replaced by
\(\log(1/\alpha)\).
\end{proposition}

\begin{proof}
Mean validity gives \(\mathbb EH^+\le0\), while
\eqref{eq:uniform-replacement-modulus} gives
\(|H^+(S)-H^+(S')|\le b_i^+\) whenever \(S\sim_iS'\).
Let
\[
 D_i:=\mathbb E(H^+\mid Z_1,\ldots,Z_i)
 -\mathbb E(H^+\mid Z_1,\ldots,Z_{i-1}).
\]
Conditionally on \(Z_1,\ldots,Z_{i-1}\), the possible values of \(D_i\)
lie in an interval of length at most \(b_i^+\).  Applying
\eqref{eq:bounded-centered-exponential-moment} conditionally to \(-D_i\)
	and iterating expectations gives, for \(\theta>0\),
	\begin{equation}
	\label{eq:uniform-upper-gap-mgf}
	 \mathbb E\exp\{\theta(H^+-\mathbb EH^+)\}
	 \le\exp\left\{\frac{\theta^2}{8}
	 \sum_{i=1}^n(b_i^+)^2\right\}.
	\end{equation}
For \(x>0\) and \(\theta>0\),
\[
 \Pr\{H^+-\mathbb EH^+>x\}
 \le e^{-\theta x}
 \mathbb E e^{\theta(H^+-\mathbb EH^+)}
 \le\exp\left\{-\theta x+\frac{\theta^2}{8}
 \sum_i(b_i^+)^2\right\}.
\]
When \(\sum_i(b_i^+)^2>0\), taking
\(\theta=4x/\sum_i(b_i^+)^2\) yields
	\begin{equation}
	\label{eq:uniform-upper-gap-tail}
	 \Pr\{H^+-\mathbb EH^+>x\}
	 \le\exp\left\{-\frac{2x^2}{\sum_i(b_i^+)^2}\right\}.
	\end{equation}
	If \(\sum_i(b_i^+)^2=0\), then \(H^+\) is almost surely constant and the
	conclusion holds without optimizing over \(\theta\).
	Since \(\mathbb EH^+\le0\), taking \(x=\Gamma_\alpha^+\) bounds
	\(\Pr\{H^+>\Gamma_\alpha^+\}\) by \(\alpha/2\).  Lower mean validity gives
	\(\mathbb EH^-\le0\), and
	\eqref{eq:uniform-replacement-modulus} gives
	\(|H^-(S)-H^-(S')|\le b_i^-\) when \(S\sim_iS'\).  The conditional
	application of \eqref{eq:bounded-centered-exponential-moment} therefore
	gives, for \(\theta>0\),
	\[
	 \mathbb E\exp\{\theta(H^--\mathbb EH^-)\}
	 \le\exp\left\{\frac{\theta^2}{8}
	 \sum_{i=1}^n(b_i^-)^2\right\}.
	\]
	Optimizing the exponential Markov bound at
	\(\theta=4x/\sum_i(b_i^-)^2\), with the zero-sum case handled as for
	\eqref{eq:uniform-upper-gap-tail}, yields
	\[
	 \Pr\{H^--\mathbb EH^->x\}
	 \le\exp\left\{-\frac{2x^2}{\sum_i(b_i^-)^2}\right\}.
	\]
	Hence \(\Pr\{H^->\Gamma_\alpha^-\}\le\alpha/2\).  A union bound proves
	\eqref{eq:two-sided-uniform-high-probability-lift}.
For one endpoint, solving
\(\exp\{-2x^2/\sum_i(b_i^\pm)^2\}=\alpha\) replaces
\(\log(2/\alpha)\) by \(\log(1/\alpha)\).
\end{proof}

\section{Crossing contraction and replacement stability}
\label{app:crossfit-composite-lemmas}

\subsection{Pullback contraction and cross-fitting}
\label{supp:subsec:pullback-contraction}

\noindent\textbf{Standing pullback setup.}
Fix \(n\ge4\) and
\(2\le n_{\rm p}\le n_{\rm e}:=n-n_{\rm p}\).  Set
\([n]:=\{1,\ldots,n\}\),
\(\mathfrak A_{n_{\rm p}}:=\{A\subset[n]:|A|=n_{\rm p}\}\), and
\(B_A:=[n]\setminus A\).  Write
\(\operatorname{Av}_A\) for the uniform average over
\(\mathfrak A_{n_{\rm p}}\), \(S_A:=(Z_i)_{i\in A}\) for the observations in
the pilot fold, and
\(P_If:=|I|^{-1}\sum_{i\in I}f(Z_i)\) for every nonempty \(I\subseteq[n]\);
in particular, \(P_n:=P_{[n]}\).
Fix an integer \(q\ge1\) and a scalar \(s>0\).  Let
\(\mathcal U\) be deterministic with distinguished \(u_0\), let
\(\mathcal H\) be a real separable Hilbert space, and let
\(h:\mathcal U\to\mathcal H\) satisfy \(h(u_0)=0\) and
\(\sup_u\|h(u)\|\le s\).  Let
\((F_u:\mathcal Z\to\mathbb R)_{u\in\mathcal U}\) be a real-valued
function class such that every \(F_u\) is measurable.  Suppose a measurable
bounded operator
\(Y(z):\mathcal H\to\mathbb R^q\) satisfies
\(\mathbb E\|Y\|_{\rm HS}^2<\infty\) and, almost surely and simultaneously
in \(u,u'\),
\begin{equation}
\label{eq:supp-pullback-lipschitz-increments}
 F_{u_0}(Z)=0,\qquad
 |F_u(Z)-F_{u'}(Z)|
 \le\|Y(Z)\{h(u)-h(u')\}\|.
\end{equation}
Put \(\kappa_q=1\) for \(q=1\) and \(\kappa_q=\sqrt2\) for \(q\ge2\),
and set \(Y_i:=Y(Z_i)\).  For each nonempty
\(I\subseteq[n]\) and each \(\lambda>0\), define the empirical second-moment
operator \(\widehat V_{Y,I}=|I|^{-1}\sum_{i\in I}Y_i^*Y_i\) and the
positive-definite ridge operator
\(M_{Y,I,\lambda}=\widehat V_{Y,I}+\lambda I\).  Write
\[
 \widehat V_Y:=\widehat V_{Y,[n]},\qquad
 \widehat q_{Y,\lambda}(u)
 :=\langle h(u),(\widehat V_Y+\lambda I)h(u)\rangle,
\]
where \(\widehat V_Y\) is the full-sample second-moment operator and
\(\widehat q_{Y,\lambda}(u)\) is the full-sample ridge quadratic form at
\(h(u)\).
When \(n\) is even, for a balanced split \(A\subset[n]\), \(B=A^c\), with
\(|A|=|B|=n/2\), define the averaged ridge operator and the associated
two-way cross-trace by
\begin{equation}
\label{eq:balanced-pullback-metric-and-trace}
\begin{aligned}
 M_{Y,A\leftrightarrow B,\lambda}
 &:=\tfrac12\{M_{Y,A,\lambda}+M_{Y,B,\lambda}\},\\
 r_{Y,A\leftrightarrow B,\lambda}
 &:=\operatorname{tr}\{\widehat V_{Y,B}M_{Y,A,\lambda}^{-1}\}
 +\operatorname{tr}\{\widehat V_{Y,A}M_{Y,B,\lambda}^{-1}\}.
\end{aligned}
\end{equation}
The complete-split effective dimension associated with the two orientation-specific
traces in \eqref{eq:balanced-pullback-metric-and-trace} is the scalar
\begin{equation}
\label{eq:composite-effective-dimension}
 \widehat r_{Y,n_{\rm p},\lambda}^{\rm cf}
 :=\operatorname{Av}_{A}\left[
 \operatorname{tr}\{\widehat V_{Y,B_A}M_{Y,A,\lambda}^{-1}\}
 +\operatorname{tr}\{\widehat V_{Y,A}M_{Y,B_A,\lambda}^{-1}\}\right].
\end{equation}

Condition~\eqref{eq:supp-pullback-lipschitz-increments} reduces the nonlinear class
to the geometry learned through \(Y\).

Measurability of the pullback suprema follows from a countable reduction.
Indeed,
separability of \(\mathcal H\) gives a countable
\(\mathcal U^\circ\subset\mathcal U\), containing \(u_0\), such that
\(h(\mathcal U^\circ)\) is dense in \(h(\mathcal U)\).  If
\(h(u_j)\to h(u)\), then, on the probability-one event where
\eqref{eq:supp-pullback-lipschitz-increments} holds simultaneously in
\(u,u'\), the pullback Lipschitz increment condition gives
 \begin{equation}
 \label{eq:pullback-pointwise-continuity}
 |F_{u_j}(z)-F_u(z)|
 \le\|Y(z)\{h(u_j)-h(u)\}\|\longrightarrow0.
 \end{equation}
Moreover,
 \begin{equation}
 \label{eq:pullback-population-continuity}
 |PF_{u_j}-PF_u|
 \le\mathbb E\|Y\|_{\rm op}\,\|h(u_j)-h(u)\|
 \le\{\mathbb E\|Y\|_{\rm HS}^2\}^{1/2}\|h(u_j)-h(u)\|
 \longrightarrow0.
 \end{equation}
Finite empirical and Rademacher sums are continuous by
\eqref{eq:supp-pullback-lipschitz-increments}.
The contracted linear terms are continuous in \(h(u)\), and the quadratic
terms are continuous because every operator \(M_{Y,I,\lambda}\) is bounded.
Hence, for every objective formed from the empirical or Rademacher sums and
the linear and quadratic terms, the supremum over \(\mathcal U^\circ\)
equals the supremum over \(\mathcal U\).  We use the
countable suprema over \(\mathcal U^\circ\), which are measurable, throughout.

Symmetrization leaves the pilot-dependent quadratic penalty inside the
supremum.  The scalar contraction lemma keeps an arbitrary offset in place
while replacing one Lipschitz process by another.

\begin{lemma}[Scalar contraction with an arbitrary offset]
\label{lem:rademacher-contraction-with-offsets}
Fix an integer \(N\ge1\), and let
\(\varepsilon_1,\ldots,\varepsilon_N\) be independent Rademacher signs.
Let \(T\) be a countable index set, let \(A:T\to\mathbb R\), and, for each
\(i\in\{1,\ldots,N\}\), let \(\psi_i,g_i:T\to\mathbb R\) satisfy
\begin{equation}
\label{eq:scalar-contraction-increment-condition}
 |\psi_i(t)-\psi_i(t')|
 \le |g_i(t)-g_i(t')|,
 \qquad t,t'\in T.
\end{equation}
Assume the two expectations
\[
 \mathbb E_\varepsilon\sup_{t\in T}
 \left\{A(t)+\sum_{i=1}^N\varepsilon_i\psi_i(t)\right\},
 \qquad
 \mathbb E_\varepsilon\sup_{t\in T}
 \left\{A(t)+\sum_{i=1}^N\varepsilon_i g_i(t)\right\}
\]
are finite.  Then
\begin{align}
 &\mathbb E_\varepsilon\sup_{t\in T}
 \left\{A(t)+\sum_{i=1}^N\varepsilon_i\psi_i(t)\right\}
\le
 \mathbb E_\varepsilon\sup_{t\in T}
 \left\{A(t)+\sum_{i=1}^N\varepsilon_i g_i(t)\right\}.
\label{eq:rademacher-contraction-with-offsets}
\end{align}
\end{lemma}

\begin{proof}
Fix arbitrary
real families \((a_t,x_t,y_t)_{t\in T}\) satisfying
\(|x_t-x_{t'}|\le|y_t-y_{t'}|\).  Choose \(t_+\) and \(t_-\) within
\(\epsilon>0\) of \(\sup_t(a_t+x_t)\) and \(\sup_t(a_t-x_t)\):
\(a_{t_+}+x_{t_+}\ge\sup_t(a_t+x_t)-\epsilon\) and
\(a_{t_-}-x_{t_-}\ge\sup_t(a_t-x_t)-\epsilon\).
Then
\begin{align}
 &\sup_t\{a_t+x_t\}+\sup_t\{a_t-x_t\}\\
 &\quad\le
 a_{t_+}+a_{t_-}+|x_{t_+}-x_{t_-}|+2\epsilon\\
 &\quad\le
 a_{t_+}+a_{t_-}+|y_{t_+}-y_{t_-}|+2\epsilon\\
 &\quad\le
 \sup_t\{a_t+y_t\}+\sup_t\{a_t-y_t\}+2\epsilon.
\label{eq:scalar-contraction-one-coordinate-comparison}
\end{align}
Let \(\varepsilon\) be one Rademacher sign.  Letting
\(\epsilon\downarrow0\) in
\eqref{eq:scalar-contraction-one-coordinate-comparison} and dividing by two
gives
\begin{equation}
\label{eq:scalar-contraction-one-sign-average}
 \mathbb E_\varepsilon\sup_t\{a_t+\varepsilon x_t\}
 \le\mathbb E_\varepsilon\sup_t\{a_t+\varepsilon y_t\}.
\end{equation}
At step \(i\), condition on all signs except \(\varepsilon_i\) and apply
\eqref{eq:scalar-contraction-one-sign-average} with
\[
 x_t=\psi_i(t),\qquad y_t=g_i(t),\qquad
 a_t=A(t)+\sum_{j<i}\varepsilon_jg_j(t)
 +\sum_{j>i}\varepsilon_j\psi_j(t).
\]
Condition \eqref{eq:scalar-contraction-increment-condition} verifies the
hypothesis.  Iterating over \(i=1,\ldots,N\) yields
\eqref{eq:rademacher-contraction-with-offsets}.
\end{proof}

Condition~\eqref{eq:supp-pullback-lipschitz-increments} compares the scalar
increments of \(F_u\) with the Euclidean norm of
\(Y\{h(u)-h(u')\}\).  Lemma~\ref{lem:vector-contraction-with-offsets}
extends scalar contraction to the norm comparison in
\eqref{eq:supp-pullback-lipschitz-increments} while leaving the offset \(A(t)\)
unchanged.

\begin{lemma}[Vector contraction with an arbitrary offset]
\label{lem:vector-contraction-with-offsets}
Fix integers \(N,q\ge1\).  Let \(T\) be a countable index set and let
\(A:T\to\mathbb R\).  For
\(i=1,\ldots,N\), let \(\psi_i:T\to\mathbb R\) and
\(g_i:T\to\mathbb R^q\) satisfy
\begin{equation}
\label{eq:vector-contraction-increment-condition}
 |\psi_i(t)-\psi_i(t')|
 \le\|g_i(t)-g_i(t')\|,
 \qquad t,t'\in T.
\end{equation}
Let \((\varepsilon_i)_{1\le i\le N}\) and
\((\xi_{ik})_{1\le i\le N,\,1\le k\le q}\) be independent Rademacher
families, and let \(e_1,\ldots,e_q\) be the standard basis of
\(\mathbb R^q\).  Assume the two expectations
\[
 \mathbb E_\varepsilon\sup_{t\in T}
 \left\{A(t)+\sum_{i=1}^N\varepsilon_i\psi_i(t)\right\},
 \qquad
 \mathbb E_\xi\sup_{t\in T}
 \left\{
 A(t)+\kappa_q
 \sum_{i=1}^N\sum_{k=1}^q
 \xi_{ik}\langle g_i(t),e_k\rangle
 \right\}
\]
are finite.  Then
\begin{align}
 &\mathbb E_\varepsilon\sup_{t\in T}
 \left\{A(t)+\sum_{i=1}^N\varepsilon_i\psi_i(t)\right\}
\le
 \mathbb E_\xi\sup_{t\in T}
 \left\{
 A(t)+\kappa_q
 \sum_{i=1}^N\sum_{k=1}^q
 \xi_{ik}\langle g_i(t),e_k\rangle
 \right\}.
\label{eq:vector-contraction-with-offsets}
\end{align}

\end{lemma}

\begin{proof}
For \(q=1\), \eqref{eq:vector-contraction-with-offsets} is
Lemma~\ref{lem:rademacher-contraction-with-offsets}.  For
\(j=0,\ldots,N\), define the scalar interpolation sequence
\begin{align*}
 \mathcal K_j
 :=\mathbb E\sup_{t\in T}\Biggl\{
 &A(t)
 +\sqrt2\sum_{i=1}^j\sum_{k=1}^q
 \xi_{ik}\langle g_i(t),e_k\rangle+\sum_{i=j+1}^N\varepsilon_i\psi_i(t)
 \Biggr\}.
\end{align*}
Condition on every sign except \(\varepsilon_j\) and
\((\xi_{jk})_{k=1}^q\).  Apply \cite[Lemma~7]{maurer2016vector} with all
remaining terms, including \(A(t)\), as the arbitrary function \(f\).
Together with \eqref{eq:vector-contraction-increment-condition}, Lemma~7
gives
\(\mathcal K_{j-1}\le\mathcal K_j\).  Iterating from \(j=1\) to \(N\) proves
\eqref{eq:vector-contraction-with-offsets}.
\end{proof}

For \(\sigma\in\{-1,1\}\), conditional symmetrization of
\(\sigma(P-P_{B_A})F_u\) leaves the pilot quadratic penalty inside a nonlinear
Rademacher supremum.  Given \(S_A\), the penalty is a fixed offset;
conditional on the evaluation observations as well,
Lemma~\ref{lem:vector-contraction-with-offsets} and
\eqref{eq:supp-pullback-lipschitz-increments} replace the nonlinear Rademacher
increments of \(F_u\) by the linear increments of \(Y_i h(u)\) without
changing the pilot quadratic penalty inside the Rademacher supremum.
Completing the contracted linear-quadratic supremum and
taking the Rademacher expectation yields the trace correction in
\eqref{eq:composite-one-fold-validity}.

\begin{lemma}[One evaluation fold against an independent pilot]
\label{lem:composite-one-fold-crossfit}
Under the standing pullback setup in
Subsection~\ref{supp:subsec:pullback-contraction}, fix
\(\eta,\lambda>0\) and \(A\in\mathfrak A_{n_{\rm p}}\).  Condition on
\(S_A\), regard \(B_A\) as the evaluation fold, and let
\(\mathbb E_{B_A}\) denote conditional expectation over the evaluation fold
\(B_A\).  Then, for both
\(\sigma\in\{-1,1\}\),
\begin{align}
 \mathbb E_{B_A}\sup_{u\in\mathcal U}\Biggl\{
 &\sigma(P-P_{B_A})F_u
 -\eta\langle h(u),M_{Y,A,\lambda}h(u)\rangle
-
 \frac{\kappa_q^2}{\eta n_{\rm e}}
 \operatorname{tr}
 \{\widehat V_{Y,B_A}M_{Y,A,\lambda}^{-1}\}
 \Biggr\}
 \le0.
\label{eq:composite-one-fold-validity}
\end{align}
The role-reversed statement holds with \(A\) and \(B_A\) interchanged and
\(n_{\rm e}\) replaced by \(n_{\rm p}\).
\end{lemma}

\begin{proof}
Let \((Z_i')_{i\in B_A}\) be an independent copy of the
observations in the evaluation fold; write \(P_{B_A'}\) for the empirical
measure of \((Z_i')_{i\in B_A}\) and \(\mathbb E_{B_A'}\) for expectation over
\((Z_i')_{i\in B_A}\).
Since \(PF_u=\mathbb E_{B_A'}P_{B_A'}F_u\), conditional Jensen gives
\begin{align}
 &\mathbb E_{B_A}\sup_{u\in\mathcal U}
 \left\{
 \sigma(P-P_{B_A})F_u
 -\eta\langle h(u),M_{Y,A,\lambda}h(u)\rangle
 \right\}
 \nonumber\\
 &\quad\le
 \mathbb E_{B_A,B_A'}\sup_{u\in\mathcal U}
 \left\{
 \sigma(P_{B_A'}-P_{B_A})F_u
 -\eta
 \langle h(u),M_{Y,A,\lambda}h(u)\rangle
 \right\}.
\label{eq:pullback-ghost-jensen}
\end{align}
Pair each \(Z_i\), \(i\in B_A\), with an independent copy \(Z_i'\).
Exchanging the two observations in any pair does not change the joint law of
\((Z_i,Z_i')\),
so independent Rademacher signs may be inserted:
\begin{align}
 &\mathbb E_{B_A,B_A'}\sup_u
 \left\{\frac{\sigma}{n_{\rm e}}\sum_{i\in B_A}
 \{F_u(Z_i')-F_u(Z_i)\}
 -\eta\langle h(u),M_{Y,A,\lambda}h(u)\rangle\right\}\\
 &\quad=
 \mathbb E_{B_A,B_A',\varepsilon}\sup_u
 \left\{\frac{\sigma}{n_{\rm e}}\sum_{i\in B_A}\varepsilon_i
 \{F_u(Z_i')-F_u(Z_i)\}
 -\eta\langle h(u),M_{Y,A,\lambda}h(u)\rangle\right\}.
\label{eq:pullback-pair-symmetrization}
\end{align}
Split the quadratic term equally between the primed and unprimed sums and use
\(\sup_u(U_u+V_u)\le\sup_uU_u+\sup_uV_u\).  The supremum containing
\(F_u(Z_i')\) and half of the quadratic penalty has an expectation equal to
the supremum containing \(-F_u(Z_i)\) and half of the quadratic penalty after
exchanging \(B_A\) and \(B_A'\) and changing all Rademacher signs.  Combining
\eqref{eq:pullback-pair-symmetrization} with
\eqref{eq:pullback-ghost-jensen} gives
\begin{align}
 &\mathbb E_{B_A}\sup_{u\in\mathcal U}
 \left\{
 \sigma(P-P_{B_A})F_u
 -\eta\langle h(u),M_{Y,A,\lambda}h(u)\rangle
 \right\}
 \nonumber\\
 &\quad\le
 2\mathbb E_{B_A,\varepsilon}
 \sup_{u\in\mathcal U}
 \left\{
 \frac\sigma{n_{\rm e}}
 \sum_{i\in B_A}\varepsilon_iF_u(Z_i)
 -\frac\eta2
 \langle h(u),M_{Y,A,\lambda}h(u)\rangle
 \right\}.
\label{eq:pullback-ghost-symmetrization}
\end{align}

Conditional on the evaluation observations, apply
Lemma~\ref{lem:vector-contraction-with-offsets} with
\[
 \psi_i(u):=\frac\sigma{n_{\rm e}}F_u(Z_i),
 \qquad
 g_i(u):=\frac1{n_{\rm e}}Y_i h(u),
\]
and with offset
\(-\frac\eta2\langle h(u),M_{Y,A,\lambda}h(u)\rangle\).
The pullback Lipschitz increment condition
\eqref{eq:supp-pullback-lipschitz-increments} is the required increment
comparison.  In the scalar case, write
\[
 R_{Y,B_A}
 :=\frac1{n_{\rm e}}
 \sum_{i\in B_A}\varepsilon_iY_i^*1,
\]
and in the vector case write
\begin{equation}
\label{eq:vector-rademacher-slope-process}
 R_{Y,B_A}
 :=\frac1{n_{\rm e}}
 \sum_{i\in B_A}\sum_{k=1}^q
 \xi_{ik}Y_i^*e_k.
\end{equation}
Let \(\mathbb E_{\rm Rad}\) denote expectation over
\((\varepsilon_i)_{i\in B_A}\) in the scalar case and over
\((\xi_{ik})_{i\in B_A,\,1\le k\le q}\) in the vector case.
In both cases, vector contraction bounds the scalar Rademacher supremum in
\eqref{eq:pullback-ghost-symmetrization} by
\begin{align}
 2\mathbb E_{B_A,{\rm Rad}}
 \sup_{u\in\mathcal U}
 \left\{
 \kappa_q\langle R_{Y,B_A},h(u)\rangle
 -\frac\eta2
 \langle h(u),M_{Y,A,\lambda}h(u)\rangle
 \right\}.
\label{eq:contracted-pullback-process}
\end{align}

Since \(\lambda>0\), \(M_{Y,A,\lambda}\) has a bounded inverse.  For every
\(a\in\mathcal H\),
\begin{align*}
 &\kappa_q\langle R_{Y,B_A},a\rangle
 -\frac\eta2\langle a,M_{Y,A,\lambda}a\rangle\\
 &\quad=
 -\frac\eta2\left\|
 M_{Y,A,\lambda}^{1/2}a
 -\frac{\kappa_q}{\eta}M_{Y,A,\lambda}^{-1/2}R_{Y,B_A}
 \right\|^2
 +\frac{\kappa_q^2}{2\eta}
 \|R_{Y,B_A}\|_{M_{Y,A,\lambda}^{-1}}^2.
\end{align*}
Hence
\begin{align}
 &\sup_{u\in\mathcal U}
 \left\{
 \kappa_q\langle R_{Y,B_A},h(u)\rangle
 -\frac\eta2
 \langle h(u),M_{Y,A,\lambda}h(u)\rangle
 \right\}
 \nonumber\\
 &\quad\le
 \sup_{a\in\mathcal H}
 \left\{
 \kappa_q\langle R_{Y,B_A},a\rangle
 -\frac\eta2\langle a,M_{Y,A,\lambda}a\rangle
 \right\}
 \nonumber\\
 &\quad=
 \frac{\kappa_q^2}{2\eta}
 \|R_{Y,B_A}\|_{M_{Y,A,\lambda}^{-1}}^2.
\label{eq:pullback-completion-of-square}
\end{align}
Conditionally on \(S_A\) and \(S_{B_A}\), independence of the Rademacher
coordinates gives, in both cases,
\begin{align}
 \mathbb E_{\rm Rad}
 (R_{Y,B_A}\otimes R_{Y,B_A})
 &\;=\frac1{n_{\rm e}^2}
 \sum_{i\in B_A}Y_i^*Y_i
 =\frac1{n_{\rm e}}\widehat V_{Y,B_A}.
\label{eq:pullback-rademacher-second-moment}
\end{align}
Since \(\widehat V_{Y,B_A}\) is finite-rank,
\begin{align}
 \mathbb E_{B_A,{\rm Rad}}
 \|R_{Y,B_A}\|_{M_{Y,A,\lambda}^{-1}}^2
 &=\mathbb E_{B_A}\operatorname{tr}\left[
 M_{Y,A,\lambda}^{-1}
 \mathbb E_{\rm Rad}
 (R_{Y,B_A}\otimes R_{Y,B_A})\right]\\
 &=\frac1{n_{\rm e}}
 \mathbb E_{B_A}\operatorname{tr}
 \{\widehat V_{Y,B_A}M_{Y,A,\lambda}^{-1}\}.
\label{eq:pullback-rademacher-trace-identity}
\end{align}
The trace term in \eqref{eq:composite-one-fold-validity} is independent of
\(u\), so the trace term may be subtracted after the supremum is taken.
Combining \eqref{eq:pullback-rademacher-trace-identity} with
\eqref{eq:contracted-pullback-process} and
\eqref{eq:pullback-completion-of-square} proves
\eqref{eq:composite-one-fold-validity}.
\end{proof}

Averaging the two role-reversed inequalities in
\eqref{eq:composite-one-fold-validity} over all pilot sets produces the
full-sample quadratic form \(\widehat q_{Y,\lambda}\) and the complete-split
trace \(\widehat r_{Y,n_{\rm p},\lambda}^{\rm cf}\).

\begin{theorem}[Mean-valid pullback field]
\label{thm:directional-composite-expected-certificate}
Under the standing pullback setup in
Subsection~\ref{supp:subsec:pullback-contraction}, fix
\(\eta,\lambda>0\) and define
\begin{equation}
\label{eq:pullback-crossfit-penalty}
 \mathcal A_{Y,n_{\rm p},s,\eta,\lambda}^{\rm cf}(u)
 :=\eta\widehat q_{Y,\lambda}(u)
 +\frac{\kappa_q^2
 \widehat r_{Y,n_{\rm p},\lambda}^{\rm cf}}{\eta n}.
\end{equation}
Then, for both \(\sigma\in\{-1,1\}\),
\[
 \mathbb E\sup_{u\in\mathcal U}
 \{\sigma(P-P_n)F_u
 -\mathcal A_{Y,n_{\rm p},s,\eta,\lambda}^{\rm cf}(u)\}\le0.
\]
\end{theorem}

\begin{proof}[Proof of
Theorem~\ref{thm:directional-composite-expected-certificate}]
Fix \(A\in\mathfrak A_{n_{\rm p}}\).  For every \(u\in\mathcal U\),
\[
 (P-P_n)F_u
 =\frac{n_{\rm p}}n(P-P_A)F_u
 +\frac{n_{\rm e}}n(P-P_{B_A})F_u.
\]
Apply Lemma~\ref{lem:composite-one-fold-crossfit} with pilot \(A\) and
evaluation fold \(B_A\), and multiply the inequality in
\eqref{eq:composite-one-fold-validity} by \(n_{\rm e}/n\).  The role-reversed application has pilot \(B_A\),
evaluation fold \(A\), and weight \(n_{\rm p}/n\).  Adding the two bounds
and using \(\sup_u(U_u+V_u)\le\sup_uU_u+\sup_uV_u\) gives, for the fixed
partition,
\begin{align}
 \mathbb E\sup_{u\in\mathcal U}\Biggl\{
 &\sigma(P-P_n)F_u
-\eta\left[
 \frac{n_{\rm e}}n
 \langle h(u),M_{Y,A,\lambda}h(u)\rangle
 +\frac{n_{\rm p}}n
 \langle h(u),M_{Y,B_A,\lambda}h(u)\rangle
 \right]
 \nonumber\\
 &-\frac{\kappa_q^2}{\eta n}\left[
 \operatorname{tr}
 \{\widehat V_{Y,B_A}M_{Y,A,\lambda}^{-1}\}
 +\operatorname{tr}
 \{\widehat V_{Y,A}M_{Y,B_A,\lambda}^{-1}\}
 \right]
 \Biggr\}
 \le0.
\label{eq:fixed-partition-pullback-certificate}
\end{align}

Average \eqref{eq:fixed-partition-pullback-certificate} over \(A\).  If
\(G_A(u)\) denotes the expression inside braces, then, pointwise in the sample,
\[
 \sup_{u\in\mathcal U}\operatorname{Av}_A G_A(u)
 \le
 \operatorname{Av}_A\sup_{u\in\mathcal U}G_A(u).
\]
Each observation belongs to
\(\binom{n-1}{n_{\rm p}-1}\) of the
\(\binom n{n_{\rm p}}\) pilot sets.  Therefore
\begin{align}
 \operatorname{Av}_A\widehat V_{Y,A}
 &=\frac1{\binom n{n_{\rm p}}}
 \sum_A\frac1{n_{\rm p}}\sum_{i\in A}Y_i^*Y_i\\
 &=\frac{\binom{n-1}{n_{\rm p}-1}}
 {n_{\rm p}\binom n{n_{\rm p}}}
 \sum_{i=1}^nY_i^*Y_i
 =\widehat V_Y.
\label{eq:pullback-pilot-second-moment-average}
\end{align}
Likewise, each observation belongs to
\(\binom{n-1}{n_{\rm p}}\) complementary folds, and
\begin{equation}
 \operatorname{Av}_A\widehat V_{Y,B_A}
 =\frac{\binom{n-1}{n_{\rm p}}}
 {n_{\rm e}\binom n{n_{\rm p}}}
 \sum_{i=1}^nY_i^*Y_i
 =\widehat V_Y.
\label{eq:pullback-evaluation-second-moment-average}
\end{equation}
Substituting
\(\operatorname{Av}_A\widehat V_{Y,A}
=\operatorname{Av}_A\widehat V_{Y,B_A}=\widehat V_Y\) from
\eqref{eq:pullback-pilot-second-moment-average} and
\eqref{eq:pullback-evaluation-second-moment-average}, and using
\(n_{\rm p}+n_{\rm e}=n\), gives
\begin{align*}
 \operatorname{Av}_A\left[
 \frac{n_{\rm e}}nM_{Y,A,\lambda}
 +\frac{n_{\rm p}}nM_{Y,B_A,\lambda}
 \right]
 &=\frac{n_{\rm e}}n\widehat V_Y
 +\frac{n_{\rm p}}n\widehat V_Y
 +\frac{n_{\rm e}+n_{\rm p}}n\lambda I\\
 &=\widehat V_Y+\lambda I,
\end{align*}
while the averaged trace is exactly
the cross-fitted pullback effective dimension
\(\widehat r_{Y,n_{\rm p},\lambda}^{\rm cf}\) defined in
\eqref{eq:composite-effective-dimension}.
Therefore
\[
 \mathbb E\sup_{u\in\mathcal U}
 \left\{
 \sigma(P-P_n)F_u
 -\eta\widehat q_{Y,\lambda}(u)
 -\frac{\kappa_q^2
 \widehat r_{Y,n_{\rm p},\lambda}^{\rm cf}}{\eta n}
 \right\}
 \le0,
\]
Hence Theorem~\ref{thm:directional-composite-expected-certificate} follows.
\end{proof}

For a fixed balanced split,
Proposition~\ref{prop:balanced-pullback-conditional-concentration} uses
\(\|Y\|_{\rm HS}\le K\) almost surely to obtain the simultaneous bound
\eqref{eq:balanced-pullback-high-probability}.

\begin{proposition}[Balanced one-split pullback concentration]
\label{prop:balanced-pullback-conditional-concentration}
Under the standing pullback setup in
Subsection~\ref{supp:subsec:pullback-contraction}, suppose
\(n_{\rm p}=n_{\rm e}=N=n/2\), and let \(A,B=A^c\) be a fixed balanced split.
Fix \(\eta,\lambda>0\), and assume
\(\|Y\|_{\rm HS}\le K\) almost surely.  Define
\begin{align*}
 \mathsf F_{Y,A}(u)
 &:=\eta\langle h(u),M_{Y,A\leftrightarrow B,\lambda}h(u)\rangle
 +\frac{\kappa_q^2r_{Y,A\leftrightarrow B,\lambda}}{\eta n}.
\end{align*}
For \(\alpha\in(0,1)\), put
\begin{align}
 \beta_Y(s,\eta,\lambda;K,q)
 &:=
 \frac{2sK+\eta s^2K^2}{n}
 +\frac{2\kappa_q^2}{\eta n^2}
 \left(\frac{K^2}{\lambda}+\frac{K^4}{\lambda^2}\right),
 \label{eq:balanced-pullback-replacement-modulus}\\
 \Gamma_{Y,\alpha}
 &:=\beta_Y(s,\eta,\lambda;K,q)
 \left\{\frac n2\log\frac2\alpha\right\}^{1/2}.
 \label{eq:balanced-pullback-confidence-correction}
\end{align}
Then, with probability at least \(1-\alpha\), simultaneously for
\(u\in\mathcal U\),
\begin{equation}
\label{eq:balanced-pullback-high-probability}
 P_nF_u-\mathsf F_{Y,A}(u)-\Gamma_{Y,\alpha}
 \le PF_u\le
 P_nF_u+\mathsf F_{Y,A}(u)+\Gamma_{Y,\alpha}.
\end{equation}
\end{proposition}

\begin{proof}[Proof of
Proposition~\ref{prop:balanced-pullback-conditional-concentration}]
Proposition~\ref{prop:uniform-high-probability-guarantee} reduces the proof to
the replacement bound
\[
 \Delta_i(P_nF\pm\mathsf F_{Y,A};\mathcal U)
 \le\beta_Y(s,\eta,\lambda;K,q),
\]
with \(\Delta_i\) defined in \eqref{eq:uniform-replacement-modulus}.
Proposition~\ref{prop:uniform-high-probability-guarantee}
then gives \eqref{eq:balanced-pullback-high-probability} with the correction
in \eqref{eq:balanced-pullback-confidence-correction}.
The application of Proposition~\ref{prop:uniform-high-probability-guarantee}
uses the deterministic target \(PF_u\) in place of \(P\delta_v\).  The proof
of Proposition~\ref{prop:uniform-high-probability-guarantee} uses only mean
validity and stability of the two endpoint fields.  Choose a measurable
\(\mathcal Z_0\) of
\(P\)-measure one on which
\eqref{eq:supp-pullback-lipschitz-increments} holds simultaneously in \(u,u'\)
and \(\|Y\|_{\rm HS}\le K\).  The reduction to the countable dense subset
\(\mathcal U^\circ\) makes all pullback suprema in
\eqref{eq:balanced-pullback-high-probability} measurable, and
replacement samples may be restricted
to \(\mathcal Z_0^n\) as in Section~\ref{app:finite-split-constants}.
Fix \(S\sim_iS'\) and put \(W_j=Y_j^*Y_j\).  Only the fold containing
observation \(i\) changes.  The fold containing observation \(i\) has size
\(N=n/2\), so
\[
 M_{Y,A\leftrightarrow B,\lambda}'-
 M_{Y,A\leftrightarrow B,\lambda}
 =\frac{W_i'-W_i}{2N}
 =\frac{W_i'-W_i}{n},
\]
and hence
\[
 \|M_{Y,A\leftrightarrow B,\lambda}'-
 M_{Y,A\leftrightarrow B,\lambda}\|_{\rm op}
 \le\frac{K^2}{n}
\]
because \(0\preceq W_i,W_i'\preceq K^2I\) implies
\(\|W_i'-W_i\|_{\rm op}\le K^2\).
If \(i\in A\), then
\[
 M_{Y,A,\lambda}'^{-1}-M_{Y,A,\lambda}^{-1}
 =M_{Y,A,\lambda}'^{-1}
 (M_{Y,A,\lambda}-M_{Y,A,\lambda}')
 M_{Y,A,\lambda}^{-1},
\]
and
\[
 \|M_{Y,A,\lambda}'^{-1}-M_{Y,A,\lambda}^{-1}\|_{\rm op}
 \le \frac1\lambda\frac{K^2}{N}\frac1\lambda
 =\frac{K^2}{N\lambda^2}.
\]
Because \(\operatorname{tr}(W_j)=\|Y_j\|_{\rm HS}^2\le K^2\), we have
\(\operatorname{tr}(\widehat V_{Y,B})\le K^2\).  The inequality
\(|\operatorname{tr}(ST)|\le\operatorname{tr}(S)\|T\|_{\rm op}\) for
\(S\succeq0\) and self-adjoint \(T\), together with
\(\|M_{Y,A,\lambda}'^{-1}-M_{Y,A,\lambda}^{-1}\|_{\rm op}
\le K^2/(N\lambda^2)\), therefore gives
\[
 \left|\operatorname{tr}\{\widehat V_{Y,B}
 (M_{Y,A,\lambda}'^{-1}-M_{Y,A,\lambda}^{-1})\}\right|
 \le K^2\|M_{Y,A,\lambda}'^{-1}-M_{Y,A,\lambda}^{-1}\|_{\rm op}
 \le\frac{K^4}{N\lambda^2}.
\]
For
\(W\in\{W_i,W_i'\}\),
\[
 0\le\operatorname{tr}(WM_{Y,B,\lambda}^{-1})
 \le\lambda^{-1}\operatorname{tr}(W)
 \le\frac{K^2}{\lambda}.
\]
Consequently,
\[
 \frac1N\left|
 \operatorname{tr}\{(W_i'-W_i)M_{Y,B,\lambda}^{-1}\}\right|
 \le\frac{K^2}{N\lambda}.
\]
If \(i\in B\), the replacement changes the numerator in
\(\operatorname{tr}\{\widehat V_{Y,B}M_{Y,A,\lambda}^{-1}\}\), giving the
bound \(K^2/(N\lambda)\), and changes the inverse in
\(\operatorname{tr}\{\widehat V_{Y,A}M_{Y,B,\lambda}^{-1}\}\), giving the
bound \(K^4/(N\lambda^2)\).  Thus, whether \(i\in A\) or \(i\in B\),
\[
 |r_{Y,A\leftrightarrow B,\lambda}'-
 r_{Y,A\leftrightarrow B,\lambda}|
 \le\frac1N
 \left(\frac{K^2}{\lambda}+\frac{K^4}{\lambda^2}\right)
 =\frac2n
 \left(\frac{K^2}{\lambda}+\frac{K^4}{\lambda^2}\right).
\]
Also \(|F_u(z)|\le sK\), so the empirical center changes by at most
\(2sK/n\).  The quadratic part of the penalty changes by at most
\(\eta s^2K^2/n\).  The trace part changes by at most
\[
 \frac{2\kappa_q^2}{\eta n^2}
 \left(\frac{K^2}{\lambda}+\frac{K^4}{\lambda^2}\right).
\]
Adding the quadratic and trace moduli to the empirical-center
modulus \(2sK/n\) gives \(\beta_Y(s,\eta,\lambda;K,q)\) in
\eqref{eq:balanced-pullback-replacement-modulus}.  For balanced folds,
\eqref{eq:fixed-partition-pullback-certificate} is the mean-validity
inequality for the penalty \(\mathsf F_{Y,A}\), for each sign.
Proposition~\ref{prop:uniform-high-probability-guarantee} with
\(b_i^+=b_i^-=\beta_Y(s,\eta,\lambda;K,q)\) now gives
\eqref{eq:balanced-pullback-high-probability}.
\end{proof}

\subsection{Assembly and uniform width}
\label{supp:subsec:assembly-uniform-width}

\begin{proof}[Proof of
\MainResult{Theorem}{thm:componentwise-hybrid-local-confidence}]
If \(A\) is randomized, condition on the realized split \(A\).  Since
\(J_v^0=\langle X_0,a(v)\rangle\), substitute \(X=X_0\) and \(h=a(v)\) into
\MainResult{Theorem}{thm:balanced-one-split-linear-confidence-field} under the
covariance-matched sub-Gaussian Taylor condition, into
\MainResult{Corollary}{cor:gaussian-one-split-linear-confidence} under the
Gaussian Taylor condition, or into
\MainResult{Corollary}{cor:bounded-feature-one-split-linear-confidence} under
the bounded-feature Taylor condition.  Each named result gives, outside an
event of probability at most \(\alpha_J\),
\begin{equation}
\label{supp:eq:hybrid-taylor-component-bound}
 |(P-P_n)J_v^0|
 \le\widehat{\mathsf G}_{J,\alpha_J}^{\rm hp}(v),
 \qquad v\in\mathcal D_{\rho,s}.
\end{equation}
When the crossing component is included (\(\iota_C=1\)), apply
Proposition~\ref{prop:balanced-pullback-conditional-concentration} with
\begin{align*}
 \mathcal U&=\mathcal D_{\rho,s},&
 u_0&=0,&
 F_v&=C_v,&
 h(v)&=v,&
 Y&=Y_C,\\
 s&=\rho,&
 q&=q_C,&
 \eta&=\eta_C,&
 \lambda&=\lambda_C,&
 K&=K_C,&
 \alpha&=\alpha_C.
\end{align*}
The penalty \(\mathsf F_{Y_C,A}\) in
Proposition~\ref{prop:balanced-pullback-conditional-concentration} equals
\(\widehat{\mathsf F}_{C,\rho}\) by
\MainEquation{eq:finite-split-crossing-penalty}, and
the correction
\(\Gamma_{Y_C,\alpha_C}\) equals
\(\widehat\Gamma_{C,\rho,\alpha_C}\) by
\MainEquation{eq:hybrid-balanced-crossing-modulus} and
\MainEquation{eq:hybrid-crossing-confidence-correction}.  Thus, outside an
event of probability at most \(\alpha_C\),
\begin{equation}
\label{supp:eq:hybrid-crossing-component-bound}
 |(P-P_n)C_v|
 \le\widehat{\mathsf F}_{C,\rho}(v)
 +\widehat\Gamma_{C,\rho,\alpha_C}
 =\widehat{\mathsf G}_{C,\rho,\alpha_C}^{\rm hp}(v)
\end{equation}
simultaneously on \(\mathcal D_{\rho,s}\).  When the crossing component is
absent (\(\iota_C=0\)), the theorem assumes that \(C_v=0\) on a common
probability-one set for every \(v\in\mathcal D_{\rho,s}\), so
\eqref{supp:eq:hybrid-crossing-component-bound} reduces to \(0\le0\).

When the remainder component is included (\(\iota_R=1\)), apply
\eqref{eq:bounded-centered-exponential-moment} to each independent
\(\Lambda_\rho(Z_i)\in[0,K_R]\).  For every \(\theta>0\),
\[
 \mathbb E\exp\{\theta(P-P_n)\Lambda_\rho\}
 \le\exp\left\{\frac{\theta^2K_R^2}{8n}\right\}.
\]
If \(K_R=0\), the desired bound holds identically.  Suppose \(K_R>0\).
For \(x>0\) and \(\theta>0\),
\[
 \Pr\{(P-P_n)\Lambda_\rho>x\}
 \le\exp\left\{-\theta x+\frac{\theta^2K_R^2}{8n}\right\}.
\]
Taking \(\theta=4nx/K_R^2\) gives
\[
 \Pr\{(P-P_n)\Lambda_\rho>x\}
 \le\exp\left\{-\frac{2nx^2}{K_R^2}\right\}.
\]
Taking \(x=K_R\{\log(1/\alpha_R)/(2n)\}^{1/2}\) gives, with probability at
least \(1-\alpha_R\),
\begin{equation}
\label{supp:eq:hybrid-remainder-empirical-mean-bound}
 P\Lambda_\rho
 \le P_n\Lambda_\rho
 +K_R\sqrt{\frac{\log(1/\alpha_R)}{2n}}.
\end{equation}
The definition of \(h_v^{\rm fr}\) and
\MainEquation{eq:radial-remainder-smoothness-envelope} give
\(h_v^{\rm fr}\le\rho^{m+1}\Lambda_\rho\).  On the event in
\eqref{supp:eq:hybrid-remainder-empirical-mean-bound}, simultaneously in
\(v\),
\begin{equation}
\label{supp:eq:hybrid-remainder-component-bound}
 |PR_v^{\rm fr}|
 \le P|R_v^{\rm fr}|
 \le\rho^{m+1}P\Lambda_\rho
 \le\widehat{\mathsf R}_{\rho,\alpha_R}^{\rm hp}.
\end{equation}
When the remainder component is absent (\(\iota_R=0\)), the theorem assumes
that \(R_v^{\rm fr}=0\) on a common probability-one set for every
\(v\in\mathcal D_{\rho,s}\), so
\eqref{supp:eq:hybrid-remainder-component-bound} reduces to \(0\le0\).

A union bound over the Taylor event and the included crossing and remainder
events, together with \MainEquation{eq:hybrid-confidence-allocation}, shows
that the three component bounds hold jointly with probability at least
\(1-\alpha\).  On the intersection of the Taylor event and the included
crossing and remainder events, the identity
\(\delta_v=J_v^0+R_v^{\rm fr}+C_v\) and
\eqref{supp:eq:hybrid-taylor-component-bound},
\eqref{supp:eq:hybrid-crossing-component-bound}, and
\eqref{supp:eq:hybrid-remainder-component-bound} give the inequalities inside
\MainEquation{eq:hybrid-local-confidence-coverage}.
For a randomized split, the conditional failure bound holds for every realized
balanced split, so averaging over \(A\) gives
\MainEquation{eq:hybrid-local-confidence-coverage} jointly over \(A\) and the
observations.
\end{proof}

\medskip
\noindent\textit{Uniform-rate setup.}
Along even sample sizes \(n\to\infty\), put \(N:=n/2\ge2\), and fix
\(\iota_C,\iota_R\in\{0,1\}\) and deterministic sequences
\(0<\rho_n\le1\) and \(s_n>0\).  Let
\(\mathcal D_{\rho_n,s_n}\) be nonempty for all sufficiently large \(n\).
Let
\[
 \boldsymbol\alpha_n
 :=(\alpha_{J,n},\alpha_{C,n},\alpha_{R,n})\in(0,1)^3,
 \qquad \alpha_n\in(0,1),
\]
and suppose
\[
 \alpha_{J,n}+\iota_C\alpha_{C,n}
 +\iota_R\alpha_{R,n}\le\alpha_n
\]
for every \(n\).  Suppose the conditions of
\MainResult{Theorem}{thm:componentwise-hybrid-local-confidence} hold rowwise on
\(\mathcal D_{\rho_n,s_n}\), using the covariance-matched sub-Gaussian Taylor
condition in
\MainResult{Theorem}{thm:componentwise-hybrid-local-confidence} with a fixed
\(c_1\ge6\).  Define the reported half-width by
\[
 W_n(v):=\frac12\left\{
 \widehat{\mathsf U}_{\rho_n,s_n,\boldsymbol\alpha_n}^{\rm hyb}(v)
 -\widehat{\mathsf L}_{\rho_n,s_n,\boldsymbol\alpha_n}^{\rm hyb}(v)
 \right\}.
\]

Set
\[
 \Sigma_{J,n}:=\operatorname{Cov}(X_{0,n}),\qquad
 d_{J,n,\lambda_{J,n}}
 :=\operatorname{tr}\{\Sigma_{J,n}
 (\Sigma_{J,n}+\lambda_{J,n}I)^{-1}\},\qquad
 t_{J,n}:=\log\frac{c_1}{\alpha_{J,n}},
\]
and
\[
 \tau_n:=d_{J,n,\lambda_{J,n}}+t_{J,n}
 +\iota_C\log\frac2{\alpha_{C,n}}
 +\iota_R\log\frac1{\alpha_{R,n}}.
\]
Assume that \(d\) and \(m\) are fixed.
Suppose a common known constant \(\bar\kappa_J<\infty\) makes
\MainEquation{eq:main-covariance-matched-subgaussian} hold with \(X=X_{0,n}\)
for every \(n\), and suppose
\[
 \sup_n\|\Sigma_{J,n}\|_{\rm op}<\infty,\qquad
 \sup_n\lambda_{J,n}<\infty.
\]
The function \(\varepsilon_N\) uses \(C_{\rm tr}(\bar\kappa_J)\), and
\(\vartheta_{N,\lambda_{J,n}}\) uses
\(C_{\rm cov}(\bar\kappa_J)\).
Assume, eventually,
\[
 \varepsilon_N(t_{J,n})\le\frac12,
 \qquad
 \vartheta_{N,\lambda_{J,n}}(t_{J,n})\le\frac12.
\]
If \(\iota_C=1\), take
\(\eta_{C,n}=(\rho_n\sqrt n)^{-1}\) and assume
\[
 0<\inf_n\lambda_{C,n}\le\sup_n\lambda_{C,n}<\infty,
 \qquad \sup_nK_{C,n}<\infty.
\]
If \(\iota_R=1\), assume \(\sup_nK_{R,n}<\infty\).

\begin{proposition}[Uniform rate of the assembled half-width]
\label{supp:prop:assembled-confidence-width-rate}
The reported half-width satisfies
\begin{equation}
\label{supp:eq:assembled-local-simple-width-rate}
 \sup_{v\in\mathcal D_{\rho_n,s_n}}W_n(v)
 =O_p\left\{\rho_n\sqrt{\frac{\tau_n}{n}}
 +\iota_R\rho_n^{m+1}\right\}.
\end{equation}
For fixed splits, probability in the \(O_p\) statement is with respect to the
observations; for independently randomized splits, probability in the
\(O_p\) statement is with respect to the split and the observations jointly.
\end{proposition}

\begin{proof}
Subtracting
\MainEquation{eq:hybrid-lower-endpoint} from
\MainEquation{eq:hybrid-upper-endpoint} gives
\begin{equation}
\label{supp:eq:assembled-width-component-decomposition}
 W_n(v)=\widehat{\mathsf G}_{J,\alpha_{J,n}}^{\rm hp}(v)
 +\widehat{\mathsf G}_{C,\rho_n,\alpha_{C,n}}^{\rm hp}(v)
 +\widehat{\mathsf R}_{\rho_n,\alpha_{R,n}}^{\rm hp}.
\end{equation}
Put
\[
 Q_{J,n}:=\sup_{v\in\mathcal D_{\rho_n,s_n}}
 \langle a(v),(\Sigma_{J,n}+\lambda_{J,n}I)a(v)\rangle.
\]
Fix
\(\beta\in(0,1)\) and put \(t_\beta=\log(c_1/\beta)\).  If
\(\alpha_{J,n}\le\beta\), apply
\MainResult{Corollary}{cor:balanced-one-split-oracle-width} at the reported
level \(\alpha_{J,n}\).  Outside an event of probability at most \(\beta\),
\[
 \sup_{v\in\mathcal D_{\rho_n,s_n}}
 \widehat{\mathsf G}_{J,\alpha_{J,n}}^{\rm hp}(v)
 \le C
 \left\{\frac{Q_{J,n}
 (d_{J,n,\lambda_{J,n}}+t_{J,n})}{n}\right\}^{1/2},
\]
where \(C\) is uniform in \(n\) because \(\bar\kappa_J\) is common to all
rows.

Suppose instead that \(\alpha_{J,n}>\beta\).  Since the Taylor feature
dimension is fixed, \(d_{J,n,\lambda_{J,n}}\) is uniformly bounded by the
fixed Taylor feature dimension.  At the fixed tail parameter \(t_\beta\),
\MainEquation{eq:conditional-trace-relative-error} and
\MainEquation{eq:main-pilot-stability-rate} imply that
\(\varepsilon_N(t_\beta)\) and
\(\vartheta_{N,\lambda_{J,n}}(t_\beta)\) tend to zero.  Hence both stability
conditions in
\MainResult{Corollary}{cor:balanced-one-split-oracle-width} hold eventually
at the fixed tail parameter \(t_\beta\).  For fixed data and \(R=2\), the
numerator in \MainEquation{eq:conditional-evaluation-trace-ucb} is
nondecreasing in \(t\), whereas the positive denominator
\(1-\varepsilon_N(t)\) is nonincreasing.  Hence
\(\overline r_{B_A\mid A}(t;2)\), and then
\(\widehat\Delta_{X,A,\lambda}(t;2)\) in
\MainEquation{eq:balanced-quadratic-confidence-inflation}, are nondecreasing
in \(t\).  The definitions in
\MainEquation{eq:balanced-finite-split-inflated-intercept} and
\MainEquation{eq:hybrid-taylor-confidence-width} therefore imply that
\(t_{J,n}<t_\beta\) makes the reported Taylor half-width no larger than the
half-width constructed at level \(\beta\).  Applying
\MainResult{Corollary}{cor:balanced-one-split-oracle-width} at level
\(\beta\) gives, outside an event of probability at most \(\beta\),
\[
 \sup_{v\in\mathcal D_{\rho_n,s_n}}
 \widehat{\mathsf G}_{J,\alpha_{J,n}}^{\rm hp}(v)
 \le C
 \left\{\frac{Q_{J,n}
 (d_{J,n,\lambda_{J,n}}+t_\beta)}{n}\right\}^{1/2}.
\]
Since \(t_{J,n}\ge\log c_1\),
\[
 d_{J,n,\lambda_{J,n}}+t_\beta
 \le
 \max\left\{1,\frac{t_\beta}{\log c_1}\right\}
 \{d_{J,n,\lambda_{J,n}}+t_{J,n}\}.
\]
Since \(\beta\in(0,1)\) was arbitrary,
\[
 \sup_{v\in\mathcal D_{\rho_n,s_n}}
 \widehat{\mathsf G}_{J,\alpha_{J,n}}^{\rm hp}(v)
 =O_p\left\{\sqrt{\frac{Q_{J,n}
 (d_{J,n,\lambda_{J,n}}+t_{J,n})}{n}}\right\}.
\]

For \(\|v\|\le\rho_n\le1\),
\[
 \|a(v)\|^2=\sum_{k=1}^m\|v\|^{2k}\le m\rho_n^2.
\]
The uniform upper bounds on \(\|\Sigma_{J,n}\|_{\rm op}\) and
\(\lambda_{J,n}\) therefore give \(Q_{J,n}=O(\rho_n^2)\).
When the crossing component is included, write \(Y_{C,n}\) for the crossing
slope operator in row \(n\) and \(A_n,B_n=A_n^c\) for the balanced split.
The bound \(\|Y_{C,n}\|_{\rm HS}\le K_{C,n}\) gives
\[
 \sup_{\|v\|\le\rho_n}
 v^\top M_{Y_{C,n},A_n\leftrightarrow B_n,\lambda_{C,n}}v
 \le(K_{C,n}^2+\lambda_{C,n})\rho_n^2,
 \qquad
 r_{Y_{C,n},A_n\leftrightarrow B_n,\lambda_{C,n}}
 \le\frac{2K_{C,n}^2}{\lambda_{C,n}}.
\]
Since \(\kappa_{q_C}^2\le2\) and
\(\eta_{C,n}=(\rho_n\sqrt n)^{-1}\),
\MainEquation{eq:finite-split-crossing-penalty} gives
\[
 \sup_{v\in\mathcal D_{\rho_n,s_n}}
 \widehat{\mathsf F}_{C,\rho_n}(v)
 \le \eta_{C,n}(K_{C,n}^2+\lambda_{C,n})\rho_n^2
 +\frac{2\kappa_{q_C}^2K_{C,n}^2}
 {\eta_{C,n}\lambda_{C,n}n}
 =O(\rho_n/\sqrt n).
\]
Substituting the bounds on \(K_{C,n}\) and \(\lambda_{C,n}\),
\(\kappa_{q_C}^2\le2\), and
\(\eta_{C,n}=(\rho_n\sqrt n)^{-1}\) into
\MainEquation{eq:hybrid-balanced-crossing-modulus} gives
\[
 \beta_{C,\rho_n}
 =O\left(\frac{\rho_n}{n}
 +\frac{\rho_n}{n^{3/2}}\right)
 =O(\rho_n/n).
\]
Substituting \(\beta_{C,\rho_n}=O(\rho_n/n)\) into
\MainEquation{eq:hybrid-crossing-confidence-correction} gives
\[
 \widehat\Gamma_{C,\rho_n,\alpha_{C,n}}
 =O\left\{\rho_n
 \sqrt{\frac{\log(2/\alpha_{C,n})}{n}}\right\}.
\]
Since \(\log(2/\alpha_{C,n})\ge\log2\),
\MainEquation{eq:hybrid-crossing-confidence-width} gives
\[
 \sup_{v\in\mathcal D_{\rho_n,s_n}}
 \widehat{\mathsf G}_{C,\rho_n,\alpha_{C,n}}^{\rm hp}(v)
 \le
 \sup_{v\in\mathcal D_{\rho_n,s_n}}
 \widehat{\mathsf F}_{C,\rho_n}(v)
 +\widehat\Gamma_{C,\rho_n,\alpha_{C,n}}
 =O\left\{\rho_n
 \sqrt{\frac{\log(2/\alpha_{C,n})}{n}}\right\}.
\]
When the crossing component is absent,
\(\widehat{\mathsf G}_{C,\rho_n,\alpha_{C,n}}^{\rm hp}\equiv0\) by
definition.

When the remainder component is included,
\MainResult{Theorem}{thm:componentwise-hybrid-local-confidence} assumes
\(0\le\Lambda_{\rho_n}\le K_{R,n}\) almost surely.  Hence
\(P_n\Lambda_{\rho_n}\le K_{R,n}=O(1)\) under
\(\sup_nK_{R,n}<\infty\).  Since \(0<\rho_n\le1\),
\MainEquation{eq:hybrid-remainder-confidence-width} yields
\[
 \widehat{\mathsf R}_{\rho_n,\alpha_{R,n}}^{\rm hp}
 =O(\rho_n^{m+1})
 +O\left\{\rho_n\sqrt{\frac{\log(1/\alpha_{R,n})}{n}}\right\}.
\]
When the remainder component is absent,
\(\widehat{\mathsf R}_{\rho_n,\alpha_{R,n}}^{\rm hp}=0\).
Combining
\eqref{supp:eq:assembled-width-component-decomposition} with the definition
of \(\tau_n\) gives
\begin{align*}
 \sup_{v\in\mathcal D_{\rho_n,s_n}}W_n(v)
 &\le
 \sup_{v\in\mathcal D_{\rho_n,s_n}}
 \widehat{\mathsf G}_{J,\alpha_{J,n}}^{\rm hp}(v)\\
 &\quad +\iota_C
 \sup_{v\in\mathcal D_{\rho_n,s_n}}
 \widehat{\mathsf G}_{C,\rho_n,\alpha_{C,n}}^{\rm hp}(v)
 +\iota_R\widehat{\mathsf R}_{\rho_n,\alpha_{R,n}}^{\rm hp}\\
 &=O_p\Biggl\{
 \rho_n\sqrt{\frac{d_{J,n,\lambda_{J,n}}+t_{J,n}}{n}}
 +\iota_C\rho_n\sqrt{\frac{\log(2/\alpha_{C,n})}{n}}\\
 &\qquad
 +\iota_R\rho_n\sqrt{\frac{\log(1/\alpha_{R,n})}{n}}
 +\iota_R\rho_n^{m+1}\Biggr\}\\
 &=O_p\left\{\rho_n\sqrt{\frac{\tau_n}{n}}
 +\iota_R\rho_n^{m+1}\right\}.
\end{align*}
Hence \eqref{supp:eq:assembled-local-simple-width-rate} holds.
\end{proof}

\section{Crossing-slope calculations and LAD specialization}
\label{app:worked-examples}

\begin{proof}[Proof of
\MainResult{Proposition}{prop:finite-selection-crossing-slope}]
Fix \(z\in\mathcal Z_{\rm PS}\).  When \(b_0(z)=1\), the definition
\(C_v=\delta_v-\chi_0G_v\) makes \(C_v(z)-C_w(z)\) equal to
\(\ell(\theta_0+v,z)-\ell(\theta_0+w,z)\).  When \(\chi_0(z)=1\), the
identity \(C_v=\delta_v-G_v\) and
\(G_v=g_{0,z}(\theta_0+v)-g_{0,z}(\theta_0)\) make the difference equal to
\[
 \{\ell(\theta_0+v,z)-g_{0,z}(\theta_0+v)\}
 -\{\ell(\theta_0+w,z)-g_{0,z}(\theta_0+w)\}.
\]
On the convex set \(K_\rho\), the branch functions are
\(\mathsf B_\rho(z)\)-Lipschitz, and the functions
\(\bar\ell_{r,z}-g_{0,z}\) are
\(\mathsf D_\rho(z)\)-Lipschitz.  A continuous selection from finitely many
Lipschitz functions on a convex set inherits the largest branch constant
\cite[Proposition~4.1.2]{scholtes2012introduction}.  Applying
\cite[Proposition~4.1.2]{scholtes2012introduction} to \(\ell(\cdot,z)\) and
to the continuous selection
\(\ell(\cdot,z)-g_{0,z}\) proves the inequality in
\MainEquation{eq:pullback-lipschitz-increments} with the operator in
\MainEquation{eq:finite-selection-crossing-envelope}.

The branch gradients are jointly measurable as limits of difference
quotients along rational coordinate increments contained in the open set
\(U_{\mathcal D}\).  Each supremum over \(K_\rho\) equals the supremum over a fixed
countable dense subset because the gradients are continuous in \(\theta\).
The envelopes \(\mathsf B_\rho\) and \(\mathsf D_\rho\) are therefore
measurable, and compactness of \(K_\rho\) makes both envelopes finite.  Since
\(g_{0,z}\) is one of the branches,
\(\mathsf D_\rho\le2\mathsf B_\rho\), so
\(L_{C,\rho}^{\rm sel}\le2\mathsf B_\rho\).  If
\(\mathsf B_\rho\le\overline B_\rho\) almost surely, then
\(\|L I_d\|_{\rm HS}=\sqrt d\,L\) gives
\(\|Y_C\|_{\rm HS}\le2\sqrt d\,\overline B_\rho\) almost surely.
\end{proof}

The compact branch-gradient condition in
\MainResult{Proposition}{prop:finite-selection-crossing-slope} is additional
to Condition~(PS).  For example, take \(d=1\), \(\theta_0=0\),
\(\mathcal D_1=[0,1)\), and \(U_{\mathcal D}=(-1,1)\).  The loss
\[
 \ell(\theta)=\max\left\{0,\frac{\theta}{1-\theta}\right\}
\]
is a continuous selection from two \(C^\infty\) branches on
\(U_{\mathcal D}\) and has an interface at zero.  Since the base point lies
on the interface, \(C_v=\ell(v)-\ell(0)=v/(1-v)\) for \(v\in[0,1)\), which
is not Lipschitz on \(\mathcal D_1\).  The loss
\(\ell(\theta)=\max\{0,\theta/(1-\theta)\}\) satisfies Condition~(PS)
but admits no compact \(K_1\subset U_{\mathcal D}\) containing
\(\theta_0+\mathcal D_1\).

\subsection{Two gradient-envelope calculations}

For fixed \(K\), the observationwise \(K\)-means loss is a finite minimum of
quadratic center losses \cite{pollard1981strong}.  With centers in
\(\mathbb R^p\), write
\[
 \theta=(\mu_1,\ldots,\mu_K),\qquad
 \ell(\theta,x)=\min_{1\le k\le K}\|x-\mu_k\|^2.
\]
On the parameter ball
\(\|\theta-\theta_0\|\le\rho\), define
\(A_\rho(x)=\|x\|+\max_k\|\mu_{0k}\|+\rho\).  The gradient of branch
\(k\) has only one nonzero center block and norm at most \(2A_\rho(x)\).
For two distinct branches, the nonzero center blocks are disjoint, so the
gradient difference has norm at most \(2\sqrt2 A_\rho(x)\).  The two cases in
\MainResult{Proposition}{prop:finite-selection-crossing-slope} therefore give
the common choice
\begin{equation}
\label{eq:supp-kmeans-crossing-envelope}
 Y_C(x)=2\sqrt2 A_\rho(x)I_{Kp},
 \qquad
 \|Y_C(x)\|_{\rm HS}
 =2\sqrt{2Kp}\,A_\rho(x).
\end{equation}
If \(\|x\|\le R_x\) almost surely, a known admissible constant is
\[
 K_C=2\sqrt{2Kp}\{
 R_x+\max_k\|\mu_{0k}\|+\rho\}.
\]
At Taylor order \(m=2\), each frozen quadratic branch is represented exactly,
so \(R_v^{\rm fr}=0\).

For a fixed finite network with continuous activations having finitely many
polynomial pieces, the parameter-space refinement of
\cite[Section~4]{bartlett2019nearly} gives finitely many polynomial branch
maps at each fixed input.  With jointly measurable network inputs and a
smooth or continuous finite smooth-selection outer loss with jointly
measurable branch maps, finite composition makes the composite branch maps
jointly measurable.  The deterministic closure result in
\cite[Section~4.1]{scholtes2012introduction} gives a continuous finite branch
family.  On a compact convex local
parameter set \(K_\rho\), let \(\Pi\) be the finite set of network and
outer-loss branch choices, write \(g_{\pi,z}\) for the composite branch map
indexed by \(\pi\),
and define
\begin{equation}
\label{eq:supp-network-branch-gradient-envelope}
 \mathsf B_\rho^{\rm net}(z)
 :=\max_{\pi\in\Pi}\sup_{\theta\in K_\rho}
 \|\nabla_\theta g_{\pi,z}(\theta)\|.
\end{equation}
Compactness makes \(\mathsf B_\rho^{\rm net}(z)\) finite for each fixed
observation, and
\MainResult{Proposition}{prop:finite-selection-crossing-slope} permits
\(Y_C=2\mathsf B_\rho^{\rm net}I_{d_\theta}\).  Known bounds on the inputs and
responses, a prescribed compact parameter set, and a known uniform derivative
bound for the outer loss on the reachable output range keep every
preactivation, output, and branch derivative in one of finitely many known
compact ranges.  The input, response, parameter-set, and outer-derivative
bounds therefore give the known deterministic essential bound
\(\mathsf B_\rho^{\rm net}\le\overline B_\rho^{\rm net}\) and the valid
choice \(K_C=2\sqrt{d_\theta}\,\overline B_\rho^{\rm net}\).  Pointwise
finiteness alone is not enough: for squared outer loss, the derivative
\(2\{f_\theta(x)-y\}\) shows why an unbounded response can preclude a
deterministic almost-sure \(K_C\).

\subsection{LAD specialization}

\begin{proof}[Proof of
\MainResult{Corollary}{cor:finite-sample-lad-endpoints}]
Write \(e=y-x^\top\theta_0\).
On the Borel observation space \(\mathbb R^d\times\mathbb R\), take
\[
 \bar\ell_1(\theta;(x,y))=y-x^\top\theta,
 \qquad
 \bar\ell_2(\theta;(x,y))=x^\top\theta-y,
\]
so that \(\ell=\max(\bar\ell_1,\bar\ell_2)\), and take
\(\mathcal T=\{(\theta,(x,y)):y-x^\top\theta=0\}\).
The branches are jointly Borel measurable and affine in \(\theta\), and
\(\mathcal T\) is closed.  The measurable index
\[
 r_0(x,y)=
 \begin{cases}
  1,&e\ge0,\\
  2,&e<0
 \end{cases}
\]
selects the loss branch on the connected component containing \(\theta_0\)
off the interface and has the required default value on the interface.
Thus Condition~(PS) holds with \(\mathcal Z_{\rm PS}=\mathbb R^d\times
\mathbb R\).  Moreover,
\[
 b_v(x,y)=\mathbf1\{\min(e,e-x^\top v)\le0\le
 \max(e,e-x^\top v)\}
\]
is jointly Borel measurable in \((v,x,y)\).  Hence \(\chi_0,X_0,L_\rho\),
\(Y_\rho,J_v^0,C_v\) are jointly measurable, as are the sample matrices and
endpoints formed from \(\chi_0,X_0,L_\rho,Y_\rho,J_v^0,C_v\).

Condition on a realized balanced split \(A\).  For \(m=1\), \(a(v)=v\), so
\(\mathcal D_{\rho,\rho}=\mathcal D_\rho\) and
\(0\in\mathcal D_\rho\).  The LAD decomposition in
\MainEquation{eq:lad-crossing-term} gives
\[
 R_v^{\rm fr}=0,\qquad \delta_v=J_v^0+C_v.
\]
The cases \(e=0\), \(0<|e|<\rho\|x\|\), and
\(|e|\ge\rho\|x\|\) yield
\(|C_v-C_w|\le L_\rho|x^\top(v-w)|\).  The bound
\(|C_v-C_w|\le L_\rho|x^\top(v-w)|\) is the crossing-slope inequality in
\MainResult{Theorem}{thm:componentwise-hybrid-local-confidence} with
\(Y_C=Y_\rho\).
Since \(L_\rho\le2\) and \(\|x\|\le K\),
\[
 \|X_0\|\le K,\qquad \|Y_\rho\|_{\rm HS}\le2K
\]
almost surely.  Hence
\MainResult{Corollary}{cor:bounded-feature-one-split-linear-confidence}
with \(X=X_0\), norm bound \(K\), ridge \(\lambda_J\), and level
\(\alpha_J\) supplies
\(\widehat U_{J,A,\lambda_J,\alpha_J}
=\widehat U_{X_0,A,\lambda_J,\alpha_J}^{\rm bd}\).
The identities \(m=1\), \(s=\rho\), \(R_v^{\rm fr}=0\), and
\(\delta_v=J_v^0+C_v\), together with
\(Y_C=Y_\rho\), \(K_C=2K\), \(q_C=1\),
\(\iota_C=1\), and \(\iota_R=0\), verify the component assumptions of
\MainResult{Theorem}{thm:componentwise-hybrid-local-confidence}.
\MainResult{Theorem}{thm:componentwise-hybrid-local-confidence} gives the
simultaneous coverage assertion conditionally on \(A\).  If \(A\) is randomized
independently of the observations, averaging the conditional bound over
\(A\) gives the joint probability statement.
\end{proof}

\begin{proposition}[Direct full-ball LAD field and width comparison]
\label{prop:supp-lad-direct-full-ball}
Under the assumptions of
\MainResult{Corollary}{cor:finite-sample-lad-endpoints}, let
\[
 H_{\rm ctr}(n,\alpha)
 :=\frac{K\rho}{\sqrt n}
 \left\{2+\sqrt{2\log\frac2\alpha}\right\},
 \qquad \alpha\in(0,1).
\]
Then, with probability at least \(1-\alpha\), simultaneously for every
\(\|v\|\le\rho\),
\begin{equation}
\label{eq:supp-lad-direct-full-ball}
 P_n\delta_v-H_{\rm ctr}(n,\alpha)
 \le P\delta_v
 \le P_n\delta_v+H_{\rm ctr}(n,\alpha).
\end{equation}
If \(K\rho>0\) and the allocation in
\MainResult{Corollary}{cor:finite-sample-lad-endpoints} satisfies
\(\alpha_C\le\alpha<2e^{-2}\), then
\begin{equation}
\label{eq:supp-lad-direct-crossing-comparison}
 H_{\rm ctr}(n,\alpha)
 <\widehat\Gamma_{C,\rho,\alpha_C}.
\end{equation}
Consequently, at every \(\|v\|\le\rho\), the half-width in
\eqref{eq:supp-lad-direct-full-ball} is strictly smaller than the half-width
of the LAD endpoints in
\MainResult{Corollary}{cor:finite-sample-lad-endpoints}.
\end{proposition}

\begin{proof}
For \(z=(x,y)\), write \(e=y-x^\top\theta_0\).  The LAD increment
\[
 \delta_v(z)=|e-x^\top v|-|e|
\]
satisfies \(\delta_0=0\), \(|\delta_v|\le K\rho\), and
\[
 |\delta_v(z)-\delta_w(z)|
 \le |x^\top(v-w)|.
\]
Let \(\mathcal Z_0:=\{(x,y):\|x\|\le K\}\).  The assumption
\(\|x\|\le K\) almost surely gives \(P(\mathcal Z_0)=1\).  Fix a countable
dense subset \(\mathcal V_0\) of
\(\{v:\|v\|\le\rho\}\).  On \(\mathcal Z_0\),
\[
 |\delta_v(z)-\delta_w(z)|\le K\|v-w\|,
 \qquad
 |P\delta_v-P\delta_w|\le K\|v-w\|.
\]
Consequently, the empirical, ghost-sample, population, and Rademacher
processes indexed by \(v\) are continuous.  For each of the four processes,
replacing the closed-ball index set by \(\mathcal V_0\) leaves the supremum
unchanged, and countability of \(\mathcal V_0\) makes every restricted supremum
measurable.

For either \(\sigma\in\{-1,1\}\), ghost-sample symmetrization gives
\[
 \mathbb E\sup_{v\in\mathcal V_0}\sigma(P-P_n)\delta_v
 \le \frac2n\,
 \mathbb E_{Z,\varepsilon}\sup_{v\in\mathcal V_0}
 \sum_{i=1}^n\varepsilon_i\sigma\delta_v(Z_i).
\]
Apply Lemma~\ref{lem:rademacher-contraction-with-offsets} with
\(T=\mathcal V_0\), \(A(v)=0\),
\(\psi_i(v)=\sigma\delta_v(Z_i)\), and \(g_i(v)=x_i^\top v\).
Conditional on the observations, the symmetrized Rademacher supremum is at
most
\[
 \frac{2\rho}{n}\,
 \mathbb E_\varepsilon
 \left\|\sum_{i=1}^n\varepsilon_i x_i\right\|
 \le\frac{2\rho}{n}
 \left(\sum_{i=1}^n\|x_i\|^2\right)^{1/2}
 \le\frac{2K\rho}{\sqrt n}.
\]
The support-function identity gives the bound by
\(\rho\mathbb E_\varepsilon\|\sum_i\varepsilon_i x_i\|\).
Jensen's inequality and Rademacher orthogonality give
\[
 \mathbb E_\varepsilon\left\|\sum_i\varepsilon_i x_i\right\|
 \le \left(\sum_i\|x_i\|^2\right)^{1/2},
\]
and \(\|x_i\|\le K\) implies
\(\sum_i\|x_i\|^2\le nK^2\).
Continuity extends the expectation inequalities from \(\mathcal V_0\) to
the whole ball.  Hence
\[
 \widehat{\mathsf U}(v):=P_n\delta_v+\frac{2K\rho}{\sqrt n},
 \qquad
 \widehat{\mathsf L}(v):=P_n\delta_v-\frac{2K\rho}{\sqrt n}
\]
are finite-valued mean-valid upper and lower endpoints on
\(\{v:\|v\|\le\rho\}\).

For samples in \(\mathcal Z_0^n\), replacing one observation changes either
endpoint by at most \(2K\rho/n\).
Proposition~\ref{prop:uniform-high-probability-guarantee}, with
\(b_i^+=b_i^-=2K\rho/n\), adds
\[
 \left\{\frac12\log\frac2\alpha
 \sum_{i=1}^n\left(\frac{2K\rho}{n}\right)^2\right\}^{1/2}
 =\frac{K\rho}{\sqrt n}\sqrt{2\log\frac2\alpha}
\]
to the mean-valid endpoints.  The correction supplied by
Proposition~\ref{prop:uniform-high-probability-guarantee} proves
\eqref{eq:supp-lad-direct-full-ball}.

For the comparison, the LAD specialization uses \(K_C=2K\).  Dropping the
nonnegative terms in
\MainEquation{eq:hybrid-balanced-crossing-modulus} other than
\(2\rho K_C/n\), and then using
\MainEquation{eq:hybrid-crossing-confidence-correction}, gives
\[
 \widehat\Gamma_{C,\rho,\alpha_C}
 \ge \frac{2K\rho}{\sqrt n}
 \sqrt{2\log\frac2{\alpha_C}}
 \ge \frac{2K\rho}{\sqrt n}
 \sqrt{2\log\frac2\alpha}.
\]
The condition \(\alpha<2e^{-2}\) is equivalent to
\(\sqrt{2\log(2/\alpha)}>2\).  Since \(K\rho>0\), the lower bound
\(2K\rho\sqrt{2\log(2/\alpha)}/\sqrt n\) is strictly larger than
\(H_{\rm ctr}(n,\alpha)\), so the bound on
\(\widehat\Gamma_{C,\rho,\alpha_C}\) proves
\eqref{eq:supp-lad-direct-crossing-comparison}.  The Taylor width and the
remaining crossing penalty in the assembled LAD half-width are nonnegative,
so the assembled LAD half-width is strictly larger than
\(H_{\rm ctr}(n,\alpha)\) at every update in the ball.
\end{proof}

\subsection{Linear LAD matrices}

For the linear LAD model surrounding
\MainEquation{eq:lad-crossing-term}, write
\(e_i=y_i-x_i^\top\theta_0\) and, for each \(i\in[n]\), define the scalar
activity mask
\[
 L_{\rho,i}:=\mathbf1\{e_i=0\}
 +2\mathbf1\{0<|e_i|<\rho\|x_i\|\}.
\]
For a nonempty fold \(I\subseteq[n]\), define the foldwise
positive-semidefinite crossing matrix and, separately, the full-sample ridge
quadratic form by
\begin{equation}
\label{eq:lad-crossing-geometry}
 \widehat V_{C,\rho,I}
 :=\frac1{|I|}\sum_{i\in I}L_{\rho,i}^2x_ix_i^\top,
 \qquad
 \widehat q_{C,\rho,\lambda}(v)
 :=\frac1n\sum_{i=1}^nL_{\rho,i}^2(x_i^\top v)^2
 +\lambda\|v\|^2.
\end{equation}
The complete two-way trace is the scalar
\[
 \widehat r_{C,\rho,n_{\rm p},\lambda}^{\rm cf}
 :=\operatorname{Av}_A\left[
 \operatorname{tr}\{\widehat V_{C,\rho,B_A}
 (\widehat V_{C,\rho,A}+\lambda I)^{-1}\}
 +\operatorname{tr}\{\widehat V_{C,\rho,A}
 (\widehat V_{C,\rho,B_A}+\lambda I)^{-1}\}\right].
\]

The simultaneous endpoint in
\MainResult{Corollary}{cor:finite-sample-lad-endpoints} uses the global bound
\(\|Y_\rho\|_{\rm HS}\le2K\) in the crossing correction of
\MainResult{Corollary}{cor:finite-sample-lad-endpoints}.  On a finite,
prespecified grid, the observed active count yields an adaptive crossing
correction.

\Needspace{0.82\textheight}
\begin{proposition}[Finite-grid empirical-Bernstein LAD crossing event]
\label{prop:supp-lad-active-empirical-bernstein}
Assume \(n\ge2\), \(\|x\|\le K\) almost surely, and
\(\alpha_C\in(0,1)\).  Let
\(\mathcal G\subset\{v:\|v\|\le\rho\}\) be deterministic with
\(1\le|\mathcal G|<\infty\).  For \(v\in\mathcal G\), define the observed
crossing increments \(C_{i,v}\), the empirical mean \(\overline C_v\) and
sample variance \(S_v^2\) of \((C_{i,v})_{i=1}^n\),
the deterministic envelope \(B_v\), the active count \(k_v\), and the scalar
tail parameter \(t_C\) by
\begin{align*}
 C_{i,v}&:=|e_i-x_i^\top v|-|e_i|+
 \operatorname{sgn}(e_i)x_i^\top v,\\
 \overline C_v&:=P_nC_v,\qquad
 S_v^2:=\frac1{n-1}\sum_{i=1}^n(C_{i,v}-\overline C_v)^2,\\
 B_v&:=2K\|v\|,\qquad
 k_v:=\sum_{i=1}^n\mathbf1\{|e_i|\le|x_i^\top v|\},\qquad
 t_C:=\log\frac{4|\mathcal G|}{\alpha_C}.
\end{align*}
Then, with probability at least \(1-\alpha_C\), simultaneously for all
\(v\in\mathcal G\),
\begin{align}
 |(P-P_n)C_v|
 &\le \mathsf E_{C,v}
 :=\sqrt{\frac{2S_v^2t_C}{n}}
 +\frac{7B_vt_C}{3(n-1)},
 \label{eq:supp-lad-grid-empirical-bernstein}\\
 \mathsf E_{C,v}
 &\le B_v\left\{
 \sqrt{\frac{2k_vt_C}{n(n-1)}}+\frac{7t_C}{3(n-1)}
 \right\}.
 \label{eq:supp-lad-active-count-bound}
\end{align}
Consequently, for \(\alpha_J\in(0,1)\), let
\(\mathsf G_{J,\alpha_J}(v)\) be any simultaneous Taylor
half-width satisfying
\(\lvert(P-P_n)J_v^0\rvert\le\mathsf G_{J,\alpha_J}(v)\) on
\(\mathcal G\) with probability at least \(1-\alpha_J\).
Intersecting the two events gives LAD endpoints on \(\mathcal G\), centered
at \(P_n\delta_v\), with half-width
\(\mathsf G_{J,\alpha_J}(v)+\mathsf E_{C,v}\)
and joint coverage at least \(1-\alpha_J-\alpha_C\).
\end{proposition}

\begin{proof}
Put \(t=x_i^\top v\).  By
\MainEquation{eq:lad-crossing-term}, \(C_{i,v}=|t|\) when \(e_i=0\), while
for \(e_i\ne0\),
\[
 C_{i,v}=2\{\operatorname{sgn}(e_i)t-|e_i|\}_+.
\]
The two formulas for \(C_{i,v}\) show that \(C_{i,v}\ge0\).  When \(e_i=0\),
\(C_{i,v}=|t|\le K\|v\|\le B_v\).  When \(e_i\ne0\),
\[
 C_{i,v}\le2|t|\le2K\|v\|=B_v.
\]
If \(|e_i|>|t|\), then
\(\operatorname{sgn}(e_i)t-|e_i|\le|t|-|e_i|<0\), so \(C_{i,v}=0\).
Thus \(C_{i,v}\) can be nonzero only among the \(k_v\) observations counted
in the proposition.

For \(B_v>0\), apply
\cite[Theorem~4]{maurer2009empirical} to \(C_{i,v}/B_v\) and to
\(1-C_{i,v}/B_v\), assigning failure probability
\(\alpha_C/(2|\mathcal G|)\) to each one-sided bound.  The two one-sided
bounds have the common logarithmic factor
\(\log\{2/(\alpha_C/(2|\mathcal G|))\}=t_C\).  A union bound over the two
signs and the \(|\mathcal G|\) grid points gives
\eqref{eq:supp-lad-grid-empirical-bernstein}; when \(B_v=0\),
\(C_{i,v}=0\) for every \(i\) and \(\mathsf E_{C,v}=0\).  Finally,
\begin{equation}
\label{eq:supp-lad-active-count-variance-bound}
 S_v^2
 =\frac1{n-1}\left\{\sum_iC_{i,v}^2-n\overline C_v^2\right\}
 \le\frac1{n-1}\sum_iC_{i,v}^2
 \le \frac{B_v^2k_v}{n-1},
\end{equation}
and inserting \eqref{eq:supp-lad-active-count-variance-bound} into
\eqref{eq:supp-lad-grid-empirical-bernstein} gives
\eqref{eq:supp-lad-active-count-bound}.  Combining the simultaneous Taylor
bound \(\lvert(P-P_n)J_v^0\rvert\le\mathsf G_{J,\alpha_J}(v)\) with
\eqref{eq:supp-lad-grid-empirical-bernstein} gives
\[
 |(P-P_n)\delta_v|
 \le |(P-P_n)J_v^0|+|(P-P_n)C_v|
 \le\mathsf G_{J,\alpha_J}(v)+\mathsf E_{C,v}
\]
for every \(v\in\mathcal G\).  A union bound gives failure probability at
most \(\alpha_J+\alpha_C\), proving the final statement.
\end{proof}

\FloatBarrier
 
\bibliographystyle{imsart-number}

\end{document}